\documentclass[default, iicol]{sn-jnl}

\usepackage{graphicx,color,psfrag, colortbl}
\usepackage{amsmath,amssymb} 
\usepackage{amsthm}
\usepackage{mathtools}
\usepackage{multirow}
\usepackage{booktabs}
\usepackage{subcaption}
\usepackage{url}

\jyear{2022}%

\theoremstyle{thmstyleone}%

\newtheorem{definition}{Definition}
\newtheorem{theorem}{Theorem}
\newtheorem{lemma}{Lemma}
\newtheorem{remark}{Remark}

\newtheorem{proposition}{Proposition}
\newtheorem{summary}{Summary}

\newcommand*{\defeq}{\stackrel{\text{def}}{=}}
\DeclareMathOperator*{\argmin}{\arg\min}
\DeclareMathOperator*{\argmax}{\arg\max}

\definecolor{Gray}{gray}{0.9}
\definecolor{LightCyan}{rgb}{0.88,1,1}

\makeatletter

\newcommand{\Rmnum}[1]{\expandafter\@slowromancap\romannumeral #1@}
\makeatother

\begin{document}

\title[DefGPA]{Procrustes Analysis with Deformations: A Closed-Form Solution by Eigenvalue Decomposition}

\author*[1]{\fnm{Fang} \sur{Bai}}\email{Fang.Bai@yahoo.com}
\author[1,2]{\fnm{Adrien} \sur{Bartolli}}\email{Adrien.Bartoli@gmail.com}

\affil*[1]{ENCOV, TGI, Institut Pascal, UMR6602 CNRS, Universit\'e Clermont Auvergne}
\affil[2]{Department of Clinical Research and Innovation, CHU de Clermont-Ferrand}

\abstract{
Generalized Procrustes Analysis (GPA) is the problem of bringing multiple shapes into a common reference by estimating transformations. GPA has been extensively studied for the Euclidean and affine transformations. We introduce GPA with deformable transformations, which forms a much wider and difficult problem. We specifically study a class of transformations called the Linear Basis Warps (LBWs), which contains the affine transformation and most of the usual deformation models, such as the Thin-Plate Spline (TPS). GPA with deformations is a nonconvex underconstrained problem. We resolve the fundamental ambiguities of deformable GPA using two shape constraints requiring the eigenvalues of the shape covariance. These eigenvalues can be computed independently as a prior or posterior. We give a closed-form and optimal solution to deformable GPA based on an eigenvalue decomposition. This solution handles regularization, favoring smooth deformation fields. It requires the transformation model to satisfy a fundamental property of free-translations, which asserts that the model can implement any translation. We show that this property fortunately holds true for most common transformation models, including the affine and TPS models. For the other models, we give another closed-form solution to GPA, which agrees exactly with the first solution for models with free-translation. We give pseudo-code for computing our solution, leading to the proposed DefGPA method, which is fast, globally optimal and widely applicable. We validate our method and compare it to previous work on six diverse 2D and 3D datasets, with special care taken to choose the hyperparameters from cross-validation.
}

\begin{table*}[t]
\lstset{basicstyle=\footnotesize\ttfamily,breaklines=false}
\lstset{framextopmargin=50pt,frame=bottomline}
\begin{lstlisting}
@article{bai2022defgpa,
    title={Procrustes Analysis with Deformations: A Closed-Form Solution by Eigenvalue Decomposition}, 
    author={Bai, Fang and Bartoli, Adrien},
    journal={International Journal of Computer Vision}, 
    year={2022},
    volume={130},
    number={2},
    pages={567-593},
    doi={10.1007/s11263-021-01571-8}
}
\end{lstlisting}

\vspace{50pt}

Springer version:
\url{https://link.springer.com/article/10.1007/s11263-021-01571-8}

\vspace{5pt}

EnCoV version:
\url{http://igt.ip.uca.fr/encov/publications/pubfiles/2021_Bai_etal_IJCV_defgpa.pdf}

\vspace{5pt}

Code repository:
\url{https://bitbucket.org/clermontferrand/deformableprocrustes/src/master/}

\end{table*}

\maketitle

\tableofcontents

\section{Introduction}

The problem of Generalized Procrustes Analysis (GPA) is to \textit{register a set of shapes with known correspondences into a single unknown reference shape}, by estimating transformations.
The existing literature focuses on GPA with Euclidean (or similarity in case of scaling) and affine transformations, which do not cope with deformations.
When present in the datum shapes, the un-modeled deformations are typically considered as noise, which creates biases on the estimation thus causing misfitting to the data.
This is undesirable, particularly when the shape is nonrigid and the precision is critical,
for example in medical applications.

We consider GPA with deformation models.
In specific, we consider a generic class of transformation models, termed Linear Basis Warps (LBWs), which generalize over affine transformation models and most commonly used warp models,
like the Thin-Plate Spline (TPS) \citep{duchon1976interpolation, bookstein1989principal},
Radial Basis Functions (RBF) \citep{fornefett2001radial},
and
Free-Form Deformations (FFD)
\citep{rueckert1999nonrigid, szeliski1997spline}.
The LBW shares many similarities with the affine model, where the transformation parameters are linear as weights to a set of possibly nonlinear basis functions.
This gives additional modeling flexibility to the LBW which copes with nonlinear deformations, while retaining the computational simplicity of linear models.

We propose GPA with LBWs.
Our cost function is formulated in the coordinate frame of the reference shape (which is termed the reference-space cost).
This cost function is linear least-squares in optimizable parameters which maximizes the benefit of using the LBW.
An important point to bear in mind is that a sufficiently generic deformation model can match any arbitrary shape to another to a good degree of fitness.
This makes the estimation of the reference shape highly ambiguous. We resolve these ambiguities by shape constraints, and show that these extra degree of freedoms can be determined by the eigenvalues of the covariance matrix of the reference shape, termed \textit{reference covariance prior}. Importantly, the reference covariance prior is independent of the rest of the computation, thus can be enforced as a prior or posterior. We propose a method to estimate the reference covariance prior based on the similarity to the datum shapes while other options are also possible.

We propose a globally optimal solution in closed-form to the GPA formulation with LBWs.
Our solution is based on a specific case of the eigenvalue problem (Theorem \ref{theorem: general result: PCA with constraint XT u = 0}) and the characterization of LBWs (Theorem \ref{label: theorem: three equivalent conditions, p, q , cost} and Theorem \ref{label: theorem: three equivalent conditions, p, q , cost, in case of partial shapes}).
We present our ideas following a specific-to-general scheme:
Section \ref{section: Generalized Procrustes Analysis with Affine Models} for affine GPA,
Section \ref{Generalized Procrustes Analysis with Warp Models} for GPA with LBWs,
and Section \ref{section: Extensions to Partial Datum Shapes} for GPA with partial datum shapes.

	Section \ref{section: Generalized Procrustes Analysis with Affine Models}.
	We present in Section \ref{section: affine Standard Form with the Shape Constraint} our affine GPA formulation and the ideas behind the shape constraints
	$\boldsymbol{S} \boldsymbol{1} = \boldsymbol{0}$ and $\boldsymbol{S} \boldsymbol{S}^{\top} = \boldsymbol{\Lambda}$, and then in Section \ref{subsection. afine GAP, globally optimal solution} the theoretical results on how to solve the formulation globally in closed-form.
	We show that a closed-form solution exists if the matrix to decompose has an all-one eigenvector.
	In Section \ref{subsubsection: Reference Covariance Prior},
	we present a method to estimate the reference covariance prior $\boldsymbol{\Lambda}$.
	We show in
	Section \ref{subsection: affine GPA. coordinate transformation of datum shapes} that the proposed formulation is invariant to coordinate transformations of the datum shapes, where the optimal reference shape will remain the same.
	Lastly, in Section \ref{sec: eliminating translations - affine case},
	we draw the connection between our method and the classical affine GPA method by eliminating translations explicitly.

	Section \ref{Generalized Procrustes Analysis with Warp Models}.
	We first introduce LBWs (Section \ref{subsection: linear basis warps}) and GPA with LBWs (Section \ref{subsection: Generalized Procrustes Analysis with Linear Basis Warps}), and then relate the existence of an all-one eigenvector to the equivalent properties of LBWs (Section \ref{section: identifying translations in transformation models} and Theorem \ref{label: theorem: three equivalent conditions, p, q , cost} in specific).
	We introduce the concept of free-translations (Definition \ref{definition: LBW contains free-translations}),
	and show that to have the desired properties in Theorem \ref{label: theorem: three equivalent conditions, p, q , cost} 
	it suffices to have a LBW where its translation part is constraint-free (Proposition \ref{theorem: P1 = 0 if and only if the LBW contains free-translations}).
	Both the affine transformation and the TPS warp, which are the two most commonly used LBWs, fall into this category, thus the resulting GPA problems can be solved in closed-form by Theorem \ref{theorem: general result: PCA with constraint XT u = 0}
	and Proposition \ref{theorem: standard result: maximization in S, with S1=0}.
	Lastly, in Section \ref{Reformulation using soft constraint},
	we propose a soft-constrained version (using a penalty term) when the condition is not met, which can also be solved in closed-form.

	Section \ref{section: Extensions to Partial Datum Shapes}.
	We extend the results from full shapes to partial shapes,
	\textit{i.e.,} the estimation of the reference covariance prior via the shape completion
	(Section \ref{section: partial shapes Eliminating Translations and Estimating Lambda}),
	the GPA formulation with LBWs for partial shapes (Section \ref{subsection: Closed-Form Solution to Generalized Procrustes Analysis with Partial Shapes}),
	and the all-one eigenvector characterization
	(Section \ref{subsection: partial shapes. Eigenvector Characterization and Tuning Parameters}).
	Section \ref{subsection: LBWs GPA. coordinate transformation of datum shapes} shows that the proposed GPA formulation with LBWs is invariant to coordinate transformations of the datum shapes, where the optimal reference shape will remain the same.
	Section \ref{subsection: reflection} shows how to handle reflections in the solution.
	Section \ref{subsection: Pseudo code}
	gives pseudo-code to benefit future research.
	Implementation details are summarized in
	Algorithm \ref{algorithm: pseudo codo of calculating reference covariance prior}
	and
	Algorithm \ref{algorithm: pseudo codo of the deformable GPA}.

In a nutshell, we present a closed-form GPA method,
termed DefGPA (\textit{i.e.,} for GPA based on LBWs), which is fast, globally optimal and copes with the general problem with regularization and incomplete shapes.
The rest of this article is organized as follow.
Section \ref{preliminaries and backgrounds, notations} introduces our notation, the GPA problem, and the Brockett cost function on the Stiefel manifold.
Section \ref{section: related work} reviews the related work.
Section \ref{section: Generalized Procrustes Analysis with Affine Models},
\ref{Generalized Procrustes Analysis with Warp Models}
and \ref{section: Extensions to Partial Datum Shapes} describe our DefGPA method.
Section \ref{section: experimental results} provides experimental results.
Section \ref{section: conclusion} concludes the paper.

\section{Modeling Preliminaries}
\label{preliminaries and backgrounds, notations}

\subsection{Notation and Terminology}

\noindent\textbf{Notation.}
We use $\mathbb{R}$ to denote the set of real numbers, $\mathbb{R}^n$ the set of column vectors of dimension $n$, and $\mathbb{R}^{m \times n}$ the set of matrices of dimension $m \times n$.
All the vectors are column majored, denoted by lower-case characters in bold.
The matrices are upper-case characters in bold.
The scalars are in italics.
For matrices, we reserve $\boldsymbol{I}$ for identity matrices and $\boldsymbol{O}$ for all-zero matrices. For vectors we reserve $\boldsymbol{0}$ for all-zero vectors and $\boldsymbol{1}$ for all-one vectors. All these special matrices and vectors are assumed to have proper dimensions induced by the context. $\boldsymbol{A}^{\top}$ is the transpose of $\boldsymbol{A}$, $\boldsymbol{A}^{-1}$ the inverse of $\boldsymbol{A}$, and $\boldsymbol{A}^{\dagger}$ the Moore-Penrose pseudo-inverse of $\boldsymbol{A}$.
We use the notation $\mathbf{Range}\left(\boldsymbol{A}\right)$ and $\mathbf{Null}\left(\boldsymbol{A}\right)$ to denote the range space and null space of $\boldsymbol{A}$, respectively.
$\boldsymbol{A} \succeq \boldsymbol{O}$ means $\boldsymbol{A}$ is positive semidefinite.
$\boldsymbol{A} \succeq \boldsymbol{B}$ means $\boldsymbol{A} - \boldsymbol{B}$ is positive semidefinite.
The symbol $\lVert \cdot \rVert_F$ denotes the matrix Frobenius norm, and $\lVert \cdot \rVert_2$ the vector $\ell_2$ norm.
$\mathbf{tr}\left(\boldsymbol{A}\right)$ is the trace for a square matrix $\boldsymbol{A}$.
$\lVert \boldsymbol{A} \rVert_F^2 = 
\mathbf{tr}\left(\boldsymbol{A} \boldsymbol{A}^{\top}\right) = 
\mathbf{tr}\left(\boldsymbol{A}^{\top} \boldsymbol{A}\right)
$.
We use $\mathbf{nnz}(\cdot)$ as a shorthand for ``number of non-zeros''.

\vspace{5pt}
\noindent\textbf{The $d$ top eigenvectors and bottom eigenvectors.}
Let $\boldsymbol{A} = \boldsymbol{Q} \boldsymbol{\Lambda} \boldsymbol{Q}^{\top}$ be the eigenvalue decomposition of a symmetric matrix $\boldsymbol{A}$, where $\boldsymbol{Q}
=
\left[\boldsymbol{q}_1, \boldsymbol{q}_2, \dots, \boldsymbol{q}_n\right]$
is a square orthonormal matrix.
Let $\boldsymbol{\Lambda} = \mathrm{diag}\left(\lambda_1, \lambda_2, \dots, \lambda_n\right)$ with
$\lambda_1 \ge  \lambda_2 \ge \dots \ge \lambda_n$.
Then we call $\boldsymbol{q}_1, \boldsymbol{q}_2, \dots, \boldsymbol{q}_d$ in sequence the $d$ \textit{top eigenvectors}, and $\boldsymbol{q}_n, \boldsymbol{q}_{n-1}, \dots, \boldsymbol{q}_{n-d+1}$ in sequence the $d$ \textit{bottom eigenvectors} of $\boldsymbol{A}$.
The analogous concepts in the Singular Value Decomposition are the $d$ leftmost and rightmost singular vectors.

\vspace{5pt}
\noindent\textbf{Shape covariance matrix.}
We model shapes as point-clouds.
Given a $d$-dimensional shape of $m$ points $\boldsymbol{S} \in \mathbb{R}^{d \times m}$, the geometric center, also called \textit{centroid}, is the mean of its columns, defined as
$
\mathbf{mean}(\boldsymbol{S})= \frac{1}{m} \boldsymbol{S} \boldsymbol{1}
$.
The concept
\textit{covariance matrix} of the shape $\boldsymbol{S}$ is defined by:
\begin{equation*}
	\mathbb{C}\mathrm{ov} (\boldsymbol{S})
	= 
	\left(
	\boldsymbol{S}
	- \mathbf{mean}(\boldsymbol{S}) \boldsymbol{1}^{\top}
	\right)
	\left(
	\boldsymbol{S}
	- \mathbf{mean}(\boldsymbol{S}) \boldsymbol{1}^{\top}
	\right)^{\top}
	.
\end{equation*}
In particular, if $\boldsymbol{S}$ is at the origin of the coordinate frame, $\mathbf{mean}(\boldsymbol{S}) = \boldsymbol{0}$,
the shape covariance becomes
$
\mathbb{C}\mathrm{ov} (\boldsymbol{S})
=
\boldsymbol{S}
\boldsymbol{S}^{\top}
$.

\subsection{Generalized Procrustes Analysis}
\label{section: generalized procrustes analysis}

A datum shape is represented as a matrix $\boldsymbol{D}_i \in \mathbb{R}^{d \times m}$
$\left(i \in \left[ 1:n \right]\right)$.
Here $d \in \left\{2,3\right\}$ is the dimension of the points and $m$ the number of points in the shape.
We will index the $j$-th point in the shape $\boldsymbol{D}_i$ as $\boldsymbol{D}_i [j]$, which is the $j$-th column of $\boldsymbol{D}_i$.
The shape points are arranged correspondence-wise: the $j$-th point in shape $\boldsymbol{D}_{i_1}$ and the $j$-th point in shape $\boldsymbol{D}_{i_2}$ correspond to different observations of the same physical point.
Missing data are caused by unobserved points in some shapes. They are modeled by binary visibility variables $\gamma_{i,j}$
($i \in \left[ 1:n \right]$, $j \in \left[ 1:m \right]$), where
$\gamma_{i,j} = 1$ if and only if the $j$-th point occurs in the $i$-th shape, and $\gamma_{i,j} = 0$ otherwise.
The reference shape $\boldsymbol{S} \in \mathbb{R}^{d \times m}$ is defined accordingly, with its $j$-th column $\boldsymbol{S}[j]$ corresponding to the $j$-th physical point.
All the points occur in the reference shape $\boldsymbol{S}$.

We denote $\mathcal{T}_i$ $\left(i \in \left[ 1:n \right]\right)$ the transformation from the datum shape $\boldsymbol{D}_i$ to the reference shape $\boldsymbol{S}$.
Here $\mathcal{T}_i$ can be Euclidean, similarity, affine, or non-rigid transformations.
Different choices of $\mathcal{T}_i$ yield different types of GPA problems.
The general GPA problem can be written as:
\begin{equation}
\label{eq: GPA - model - visibility variables}
\begin{aligned}
	 \argmin_{\{\mathcal{T}_i\},\, \boldsymbol{S}}\quad
	& \sum_{i=1}^n \sum_{j=1}^m\,
	\gamma_{i,j}
	\lVert
	\mathcal{T}_i \left(\boldsymbol{D}_{i}[j]\right) - \boldsymbol{S}[j] \rVert_{2}^2
	\\[5pt] 
	\mathrm{s.t.}
	\quad &
	\mathcal{C} \left(\mathcal{T}_1, \mathcal{T}_2, \dots, \mathcal{T}_n, \boldsymbol{S}\right) = \boldsymbol{0}
	,
\end{aligned}
\end{equation}
where $\mathcal{C} \left(\mathcal{T}_1,
\mathcal{T}_2,
\dots,
\mathcal{T}_n, \boldsymbol{S}\right) = \boldsymbol{0}$ denotes the set of constraints used to avoid degeneracy.
A typical degenerate case is
$\mathcal{T}_i \left(\boldsymbol{p}\right) = \boldsymbol{0}$ for any point $\boldsymbol{p}$, and $\boldsymbol{S} = \boldsymbol{O}$.
The detailed choices of these constraints will be provided later on.

The GPA model in formulation (\ref{eq: GPA - model - visibility variables}) holds for many transformation models and handles missing datum points.
We give a more compact matrix form which will simplify the derivation of our methods.
To this end,
we define the diagonal visibility matrix $\boldsymbol{\Gamma}_i = \mathrm{diag}
\left(
\gamma_{i,1}, \gamma_{i,2}, \dots, \gamma_{i,m}
\right)$ $\left(i \in \left[ 1:n \right]\right)$, constructed by the visibility variables corresponding to the shape $\boldsymbol{D}_i$.
By noting that:
\begin{align*}
\chi^2_r & = 
 \sum_{i=1}^n \sum_{j=1}^m\,
\gamma_{i,j}
\lVert
\mathcal{T}_i \left(\boldsymbol{D}_{i}[j]\right) - \boldsymbol{S}[j] \rVert_{2}^2
\\ & = 
\sum_{i=1}^n \sum_{j=1}^m\,
\lVert
\gamma_{i,j}
\mathcal{T}_i \left(\boldsymbol{D}_{i}[j]\right) - \gamma_{i,j} \boldsymbol{S}[j] \rVert_{2}^2
\\ & = 
\sum_{i=1}^n\,
\lVert
\mathcal{T}_i \left(\boldsymbol{D}_{i}\right) \boldsymbol{\Gamma}_i
- \boldsymbol{S} \boldsymbol{\Gamma}_i
\rVert_{F}^2
,
\end{align*}
we can write the GPA model in a compact matrix form:
\begin{equation}
\label{eq: GPA - model - original visibility matrix}
\begin{aligned}
\argmin_{\{\mathcal{T}_i\},\, \boldsymbol{S}}\quad
& \sum_{i=1}^n\,
\lVert
\mathcal{T}_i \left(\boldsymbol{D}_{i}\right) \boldsymbol{\Gamma}_i
- \boldsymbol{S} \boldsymbol{\Gamma}_i
\rVert_{F}^2	
\\[5pt]
\mathrm{s.t.}
\quad
& \mathcal{C} \left(\mathcal{T}_1, \mathcal{T}_2, \dots, \mathcal{T}_n, \boldsymbol{S}\right) = \boldsymbol{0}
.
\end{aligned}
\end{equation}
In case of full shapes (shapes without missing datum points), we have $\boldsymbol{\Gamma}_i = \boldsymbol{I}$ $\left(i \in \left[ 1:n \right]\right)$.

\vspace{5pt}
\noindent\textbf{The reference-space cost and the datum-space cost.}
The cost function defined in formulations (\ref{eq: GPA - model - visibility variables}) and (\ref{eq: GPA - model - original visibility matrix})
is called the \textit{reference-space cost} since the residual error is evaluated in the coordinate frame of the reference shape \citep{bartoli2013stratified}.
In contrast, it is possible to derive a cost function in the \textit{datum-space}:
\begin{equation}
\label{eq: the datum space cost def}
\chi^2_d = 
\sum_{i=1}^n\,
\lVert
\boldsymbol{D}_{i} \boldsymbol{\Gamma}_i
- \mathcal{T}_i^{-1} \left(\boldsymbol{S}\right) \boldsymbol{\Gamma}_i
\rVert_{F}^2.
\end{equation}
The datum-space cost is a generative model since all the datum shapes $\boldsymbol{D}_i$ are generated by the single reference shape $\boldsymbol{S}$ with transformations $\mathcal{T}_i^{-1}$.
In contrast, the reference-space cost is discriminative,
as it seeks for the transformations $\mathcal{T}_i$ that can best match the datum shapes $\boldsymbol{D}_i$ with the reference shape $\boldsymbol{S}$.
The reference-space cost and the datum-space cost are identical if $\mathcal{T}_i$ represents Euclidean transformations \citep{bartoli2013stratified}.

\subsection{Brockett Cost Function on the Stiefel Manifold}
\label{section. Brockett Cost Function on the Stiefel Manifold}

The matrix Stiefel manifold is the set of matrices satisfying:
\begin{equation*}
St(d, m)
=
\left\{  \boldsymbol{X} \in \mathbb{R}^{m \times d}
\ \vert\ 
\boldsymbol{X}^{\top} \boldsymbol{X} = \boldsymbol{I} \right\}
.	
\end{equation*}
In particular, we shall use the classical result of the following Stiefel manifold optimization problem:
\begin{equation}
\label{eq: Stiefel manifold optimzation, Brockett cost fucntion}
\argmin_{ \boldsymbol{X} \in St(d, m)} \ 
\mathbf{tr}\left(  \boldsymbol{X}^{\top} 
\boldsymbol{\mathcal{P}}
\boldsymbol{X}  \boldsymbol{\Lambda} \right)
.
\end{equation}
The cost in problem (\ref{eq: Stiefel manifold optimzation, Brockett cost fucntion}) is termed the Brockett cost function in the Stiefel manifold optimization literature \citep{brockett1989least, birtea2019first, absil2009optimization}.
Importantly, this problem admits a closed-form solution.
The critical points of the Brockett cost function on the Stiefel manifold are the eigenvectors of $\boldsymbol{\mathcal{P}}$ \citep{brockett1989least, birtea2019first}.
The global minimum of problem (\ref{eq: Stiefel manifold optimzation, Brockett cost fucntion}) can thus be obtained by arranging the eigenvectors as follow.
Let $\left(\alpha_j, \boldsymbol{\xi}_j\right)$ $(j \in \left[ 1:m \right])$ be the set of eigenvalues and eigenvectors of $\boldsymbol{\mathcal{P}}$,
such that $\boldsymbol{\mathcal{P}} \boldsymbol{\xi}_j
= \alpha_j \boldsymbol{\xi}_j$,
with
$0 \le \alpha_1 \le \alpha_2 \le \dots \le \alpha_m$ arranged in the ascending order.
Let the diagonal elements of $\boldsymbol{\Lambda}$ being arranged in the descending order
$
\lambda_1 \ge \lambda_{2} \ge \dots \ge \lambda_{d}
$.
Then $\boldsymbol{X}^{\star} = [\boldsymbol{\xi}_1, \boldsymbol{\xi}_2, \dots, \boldsymbol{\xi}_d]$,
comprising the $d$ bottom eigenvectors of $\boldsymbol{\mathcal{P}}$,
is globally optimal to problem (\ref{eq: Stiefel manifold optimzation, Brockett cost fucntion}),
with a total cost
$
\alpha_1 \lambda_1 + \alpha_2 \lambda_{2}
+ \cdots + \alpha_{d} \lambda_{d}
$.
Any other combination yields a larger cost which proves the optimality of $\boldsymbol{X}^{\star}$, by a result of Hardy-Littlewood-Polya \citep{brockett1989least, hardy1952inequalities}.

\begin{lemma}
	\label{theorem: solution to standard form with P matrix in minimization}
	Let $\boldsymbol{\mathcal{P}}$ be a symmetric matrix with its $d$-bottom eigenvectors being $\boldsymbol{X} = [\boldsymbol{\xi}_1, \boldsymbol{\xi}_2, \dots, \boldsymbol{\xi}_d]$.
	Let $\boldsymbol{\Lambda}$ be a diagonal matrix $\boldsymbol{\Lambda} = \mathrm{diag}\left(\lambda_1, \lambda_2, \dots, \lambda_n\right)$ with
	$
	\lambda_1 \ge  \lambda_2 \ge \dots \ge \lambda_n \ge 0
	$.
	The globally optimal solution to the optimization problem:
	\begin{equation}
	\label{eq: standard form with P matrix in minimization}
	\begin{aligned}
	&
	\argmin_{\boldsymbol{S}}\  
	\mathbf{tr}\left( \boldsymbol{S} \boldsymbol{\mathcal{P}} \boldsymbol{S}^{\top} \right)
	\\ & 
	\mathrm{s.t.} \quad
	\boldsymbol{S} \boldsymbol{S}^{\top} = \boldsymbol{\Lambda}
	,	
	\end{aligned}
	\end{equation}
	is $\boldsymbol{S}^{\star} =  \sqrt{\boldsymbol{\Lambda}} \boldsymbol{X}^{\top}$,
	\textit{i.e.,} obtained by scaling the $d$ bottom eigenvectors of $\boldsymbol{\mathcal{P}}$ by $\sqrt{\boldsymbol{\Lambda}}$.
\end{lemma}
\begin{proof}

We introduce a matrix $\boldsymbol{X}$ so that
$\boldsymbol{S}^{\top} = \boldsymbol{X} \sqrt{\boldsymbol{\Lambda}}$.
The constraint $\boldsymbol{S} \boldsymbol{S}^{\top} = \boldsymbol{\Lambda}$ thus becomes $\boldsymbol{X}^{\top} \boldsymbol{X} = \boldsymbol{I}$.
The cost function $\mathbf{tr}\left( \boldsymbol{S} \boldsymbol{\mathcal{P}} \boldsymbol{S}^{\top} \right)$ can be rewritten as
$
\mathbf{tr}\left( \boldsymbol{S} \boldsymbol{\mathcal{P}} \boldsymbol{S}^{\top} \right)
=
\mathbf{tr}\left( \sqrt{\boldsymbol{\Lambda}} \boldsymbol{X}^{\top} \boldsymbol{\mathcal{P}} \boldsymbol{X} \sqrt{\boldsymbol{\Lambda}} \right)
=
\mathbf{tr}\left(  \boldsymbol{X}^{\top} \boldsymbol{\mathcal{P}} \boldsymbol{X}  \boldsymbol{\Lambda} \right)
.
$
The globally optimal solution to problem (\ref{eq: standard form with P matrix in minimization}) is thus
$
{\boldsymbol{S}^{\star}}^{\top} = \boldsymbol{X}^{\star} \sqrt{\boldsymbol{\Lambda}} 
$,
with $\boldsymbol{X}^{\star}$ being the solution of problem (\ref{eq: Stiefel manifold optimzation, Brockett cost fucntion}).
In other words, $\boldsymbol{S}^{\star}$
is obtained by scaling the $d$ bottom eigenvectors of $\boldsymbol{\mathcal{P}}$ by $\sqrt{\boldsymbol{\Lambda}}$.
\end{proof}

\begin{remark}
	Analogously, the globally optimal solution to the following maximization problem:
	\begin{equation}
	\label{eq: affine compact standard form with trace as maximization}
	\begin{aligned}
	&
	\argmax_{\boldsymbol{S}}\  
	\mathbf{tr}\left( \boldsymbol{S} \boldsymbol{\mathcal{Q}} \boldsymbol{S}^{\top} \right)
	\\ & 
	\mathrm{s.t.} \quad
	\boldsymbol{S} \boldsymbol{S}^{\top} = \boldsymbol{\Lambda}.	
	\end{aligned}
	\end{equation}
	is to scale
	the $d$ top eigenvectors of the matrix $\boldsymbol{\mathcal{Q}}$
	by $\sqrt{\boldsymbol{\Lambda}}$.
\end{remark}

\section{Related Work}
\label{section: related work}

GPA considers shape registration with known correspondences, which abounds in the literature. We classify the literature according to the used transformation model.

\subsection{Rigid Case}

\textbf{Two shapes.}
The classical approach to the Procrustes analysis problem starts with linear algebraic results of two shape registration using least-squares.
In particular, the closed-form solution of 
registering two point-clouds with orthonormal matrices (termed orthogonal Procrustes) was established half a century ago
\citep{green1952orthogonal, schonemann1966generalized}.
The result can be extended to tackle translations
\citep{arun1987least}
and scale factors
\citep{horn1988closed}.
The result in \citep{arun1987least, horn1988closed} uses orthonormal matrices to represent rotations.
Besides, it is possible to derive the result using
unit quaternions to represent rotations \citep{horn1987closed},
and furthermore, using dual quaternions to represent both rotations and translations
\citep{walker1991estimating}.
When confronting large noise, the orthonormal matrix based approaches \citep{arun1987least, horn1988closed} can give a reflection matrix (with determinant $-1$) as apposed to a valid rotation matrix \citep{umeyama1991least}.
A correction to the special orthogonal group ($\mathrm{SO}(3)$, which represents valid rotations) is given in \citep{umeyama1991least},
and a more concise derivation in \citep{kanatani1994analysis}.
A comparative study of these methods is provided in
\citep{eggert1997estimating}.
However from the statistical point of view, these methods assume that noise only occurs in the target point-cloud,
which is isotropic, identical and independently Gaussian distributed.
This assumption can deviate from real noise schemes.
More sophisticated formulations based on
anisotropic and inhomogeneous Gaussian noise models are discussed in \citep{ohta1998optimal, matei1999optimal},
where the renormalization technique based on the quaternion parameterization is used to solve the resulting optimization problem effectively.
Overall,
two shape rigid Procrustes analysis can be considered a solved problem.

\vspace{5pt}
\noindent\textbf{Multiple shapes.}
Multiple shape Procrustes analysis is a more difficult problem, which is termed generalized Procrustes analysis (GPA) \citep{kristof1971generalization, gower1975generalized}.
The work in \citep{kristof1971generalization, ten1977orthogonal} examined the optimality conditions of the orthogonal GPA (Euclidean GPA without translations).
The derived conditions are either necessary or sufficient, however not necessary-sufficient.
Actually, due to the existence of rotations, the Euclidean GPA, or similar problems like pose graph optimization (PGO) \citep{bai2021TRO} has recently been acknowledged as highly nonlinear and non-convex \citep{rosen2019se}.
Therefore up to now, exact techniques for the rigid GPA are all iterative.
In early days, using the pairwise rigid Procrustes result as a subroutine,
rigid GPA was solved by alternating the estimation of the reference shape and that of the transformations until convergence \citep{gower1975generalized, ten1977orthogonal, rohlf1990extensions, wen2006least}.
However, such methods are not guaranteed to attain the global optimum of the cost function.
The translation part of the rigid GPA problem is free of constraints, thus can be eliminated from the cost function by a variable projection scheme \citep{golub2003separable}.
As a result, the rigid GPA can be formulated as a rotation estimation problem on the manifold \citep{benjemaa1998solution, williams2001simultaneous, krishnan2005global}.
Sequential techniques were proposed based on the unit quaternion \citep{benjemaa1998solution}
and orthonormal matrices \citep{williams2001simultaneous}.
A complete approach on
the $\mathrm{SO(3)}$ Lie group has been provided in \citep{krishnan2005global}.
Building upon
efficient sparse linear algebra techniques
\citep{davis2006direct}
and Lie group theory \citep{iserles2000lie},
the Lie group approach can be considered as the current state-of-the-art.
An initialization technique to the iterative solver was provided in
\citep{bartoli2013stratified}.

\subsection{Affine Case}

There exists a vast body of research in
shape analysis \citep{kendall1984shape, kendall2009shape, dryden2016statistical} where the GPA problem arouse in parallel,
\textit{e.g.,}
the complex arithmetic approach for 2D shapes and its relation to Procrustes analysis \citep{kent1994complex}.
The shape statistics, \textit{i.e.,} mean and variability, are defined on aligned shapes (thus being invariant to rotation, translation and scaling).
A commonly used alignment technique is the Procrustes analysis \citep{kendall1984shape} or GPA \citep{goodall1991procrustes}.
The affine GPA has been studied
with a closed-from solution proposed \citep{rohlf1990extensions}.
In the present notation, following the result in \citep{rohlf1990extensions},
the principal axes of the reference shape are estimated as the $d$ top eigenvectors of the matrix $\sum_{i=1}^{n} \boldsymbol{D}_i^{\top}  \left(\boldsymbol{D}_i^{\top}\right)^{\dagger}$,
after centering each $\boldsymbol{D}_i$ to the origin of the coordinate frame and eliminating the translation parameters.

We derive similar results in
Section \ref{sec: eliminating translations - affine case} in the global optimization context instead of the alternation scheme used in \citep{rohlf1990extensions},
which means the scaling factor (\textit{i.e.,} the reference covariance prior used in the proposed shape constraint) is estimated in a different way (Section \ref{subsubsection: Reference Covariance Prior}).
We show in Section \ref{sec: eliminating translations - affine case} that this approach gives the same result as the proposed formulation (\ref{eq: affine-GPA, homogeneous formulation}),
while the proposed formulation without eliminating translations is more general and extensible to the LBWs.

The above affine GPA solution, in our context, is called the reference-space solution.
On the other hand, the full shape affine GPA in the datum-space admits a closed-form solution using the Singular Value Decomposition (SVD).
Let $\boldsymbol{D}_i$ be zero centered.
Then the datum-space cost is
$
\sum_{i=1}^n\,
\lVert
\boldsymbol{D}_{i} 
- \boldsymbol{A}_i \boldsymbol{S}
\rVert_{F}^2
$.
The optimal solution, $\left[
\boldsymbol{A}_1^{\top} 
, \cdots ,
\boldsymbol{A}_n^{\top}
\right]^{\top} \boldsymbol{S}$ as a whole, can be determined via the SVD of the matrix
$
[
\boldsymbol{D}_1^{\top} 
, \cdots ,
\boldsymbol{D}_n^{\top}
]^{\top}
$.
This idea is in the same spirit of the
Tomasi-Kanade factorization \citep{tomasi1992shape} in computer vision,
which studied the affine transformation of the orthographic camera model.
We term this method as $\ast$AFF\_d and use it as a benchmark algorithm in the experiments.

\subsection{Deformable Case}

In shape analysis, a shape is considered as an element on the so-called \textit{shape manifold} (the set of all possible shapes) \citep{kendall1984shape, kilian2007geometric}.
The shape deformation is thus studied as the shortest path on the shape manifold, \textit{i.e.,} the geodesic path under a chosen Riemannian metric \citep{kilian2007geometric}.
In the rigid case, after discarding translation and scaling, the shape manifold is a quotient manifold invariant to rotations and the Riemannian metric can be chosen as the Procrustes distance \citep{kendall1984shape}.
For deformable cases,
the Riemannian metric can be otherwise formulated based on the rigidity or isometry \citep{kilian2007geometric}.
The work \citep{kendall1984shape, kilian2007geometric}, as well as this paper, use landmarks (or meshes with fixed connections) to model shapes and define transformations.
In addition, recent research has studied shape analysis based on curves \citep{joshi2007novel, younes2008metric} or surfaces, \textit{e.g.,}
level sets \citep{osher2003level},
medial surfaces
\citep{bouix2005hippocampal},
Q-maps \citep{kurtek2010novel, kurtek2011elastic}, Square Root Normal Fields (SRNF) \citep{jermyn2012elastic, laga2017numerical}, etc.
See the review papers \citep{younes2012spaces, laga2018survey}.
There is also a line of work including skeletal structures, \textit{e.g.,} the medial axis representations (M-rep) \citep{fletcher2004principal},
and
SCAPE (mesh with an articulated skeleton) \citep{anguelov2005scape}.

Beyond different shape modeling techniques,
deformations based on landmarks/meshes have a rich literature.
Typically,
the deformation of a mesh is defined piece-wisely, with a local transformation associated to each vertex or triangle \citep{freifeld2012lie},
or with additional geometric or smoothing constraints
\citep{allen2003space, anguelov2005scape, sumner2007embedded, song2020efficient}.
It is worth mentioning that such piece-wise models are commonly used in describing template based deformations, where a reference shape is given a priori.
In this paper, we use landmarks to model shapes and warps to model deformations.
The landmarks can be extracted from the key-points or the samplings of the shape boundary \citep{cootes1995active}.
While the proposed method applies to meshes as well, the connection of vertices (or edges of the meshes) is never required.
Our analysis is based on a class of warps, called LBWs,
which is a linear combination of the nonlinear basis functions \citep{rueckert1999nonrigid, szeliski1997spline, bookstein1989principal,fornefett2001radial, bartoli2010generalized}.
A typical example of the LBWs is the well-known TPS warp \citep{bookstein1989principal} which we will use to demonstrate our results.

There has been work using warps to refine the shapes after solving the rigid registration \citep{brown2007global, kim20083d}.
Such methods are not template free in essence as when it comes to the estimation of the warp parameters, a reference shape has been known already.
Little attention has been paid to estimate the reference shape and the transformation parameters all together in a unified manner.
In this work, we provide a closed-form solution to the unified estimation problem.

Different from classical affine methods
\citep{kendall1984shape,goodall1991procrustes, rohlf1990extensions},
our method allows for the possibility to keep the translation parameters during the estimation (Theorem \ref{theorem: general result: PCA with constraint XT u = 0}
and Proposition \ref{theorem: standard result: maximization in S, with S1=0}).
This is crucial to GPA with LBWs, as in certain cases, the translation parameters cannot be explicitly identified.
In specific, for LBWs,
we relate the constraint-free translation (Definition \ref{definition: LBW contains free-translations}) to the existence of an eigenvector of all-ones and the equivalent properties of LBWs
(Theorem \ref{label: theorem: three equivalent conditions, p, q , cost}
and Theorem \ref{label: theorem: three equivalent conditions, p, q , cost, in case of partial shapes}),
which constitutes the foundation of our closed-form solution.

\section{Generalized Procrustes Analysis with the Affine Model}
\label{section: Generalized Procrustes Analysis with Affine Models}

\subsection{Standard Form with the Shape Constraint}
\label{section: affine Standard Form with the Shape Constraint}

We start with the case without missing datum points, \textit{i.e.,}
$\boldsymbol{\Gamma}_i = \boldsymbol{I}$ in formulation (\ref{eq: GPA - model - original visibility matrix}).
We study the affine-GPA with
$\mathcal{T}_i \left(\boldsymbol{D}_i\right) = \boldsymbol{A}_i \boldsymbol{D}_i
+
\boldsymbol{t}_i \boldsymbol{1}^{\top}$, with $\boldsymbol{A}_i \in \mathbb{R}^{d \times d}$ being linear and $\boldsymbol{t}_i \in \mathbb{R}^{d}$ a translation.
We write the affine transformation in homogeneous form as:
\begin{equation}
\label{eq: homogeneous representation of affine transformation}
\boldsymbol{A}_i \boldsymbol{D}_i
+
\boldsymbol{t}_i \boldsymbol{1}^{\top}
=
\begin{bmatrix}
\boldsymbol{A}_i & \boldsymbol{t}_i
\end{bmatrix}
\begin{bmatrix}
\boldsymbol{D}_i \\
\boldsymbol{1}^{\top}
\end{bmatrix}
\defeq
\boldsymbol{\tilde{A}}_i \boldsymbol{\tilde{D}}_i
,
\end{equation}
where we term $\boldsymbol{\tilde{D}}_i$ the \textit{homogeneous representation} of $\boldsymbol{D}_i$ by completing $\boldsymbol{D}_i$ with a row of all ones.
In this representation, matrix $\boldsymbol{\tilde{A}}_i = [\boldsymbol{A}_i,\, \boldsymbol{t}_i]$ contains all parameters to be estimated.

Following the reference-space cost (\ref{eq: GPA - model - original visibility matrix}),
our proposed formulation is given as:
\begin{equation}
\label{eq: affine-GPA, homogeneous formulation}
\mathrm{\Rmnum{1}:}
\begin{cases}
\argmin\limits_{\{\boldsymbol{\tilde{A}}_i\},\, \boldsymbol{S}} \quad 
& 
\sum_{i=1}^n\,
\lVert \boldsymbol{\tilde{A}}_i \boldsymbol{\tilde{D}}_i - \boldsymbol{S} \rVert_F^2
\\[10pt]
\mathrm{s.t.} \quad &
\boldsymbol{S} \boldsymbol{S}^{\top} = \boldsymbol{\Lambda},
\quad
\boldsymbol{S}  \boldsymbol{1} = \boldsymbol{0}
.
\end{cases}
\end{equation}
In formulation (\ref{eq: affine-GPA, homogeneous formulation}),
$\boldsymbol{S} \boldsymbol{S}^{\top} = \boldsymbol{\Lambda}$
and
$\boldsymbol{S}  \boldsymbol{1} = \boldsymbol{0}$
are the shape constraints, which concretize the constraint $\mathcal{C} \left(\cdot\right) = \boldsymbol{0}$ in formulation (\ref{eq: GPA - model - original visibility matrix}).
The matrix $\boldsymbol{\Lambda}$,
termed \textit{reference covariance prior}, is a diagonal matrix $\boldsymbol{\Lambda} = \mathrm{diag}\left(\lambda_1, \lambda_2, \dots, \lambda_d\right)$
driven by $d$ parameters $\lambda_1, \lambda_2, \dots, \lambda_d$.
We shall shortly present in Section \ref{subsubsection: Reference Covariance Prior} a method to estimate $\boldsymbol{\Lambda}$ based on rigidity.
The ideas behind these two shape constraints are motivated as follow.

\subsubsection{The Shape Constraint $\boldsymbol{S}  \boldsymbol{1} = \boldsymbol{0}$}

This shape constraint is used to center the reference shape to the origin of the coordinate frame. The role of this constraint is twofold:
1) it provides $d$ constraints to remove the $d$ gauge freedoms caused by translations;
2) it reduces the shape covariance matrix of the reference shape to the form
$\mathbb{C}\mathrm{ov} (\boldsymbol{S})
=
\boldsymbol{S} \boldsymbol{S}^{\top} 
$
as in this case
$\mathbf{mean}(\boldsymbol{S}) = \frac{1}{m} \boldsymbol{S}  \boldsymbol{1} = \boldsymbol{0}$.

\subsubsection{The Shape Constraint $\boldsymbol{S} \boldsymbol{S}^{\top} = \boldsymbol{\Lambda}$}

This shape constraint is used to: 1) capture the eigenvalues of the shape covariance matrix $\mathbb{C}\mathrm{ov} (\boldsymbol{S})
=
\boldsymbol{S} \boldsymbol{S}^{\top} 
$, and 2) fix the gauge freedom caused by rotations.
The main insight here is that the eigenvalues of the shape covariance matrix remain unchanged if the shape undergoes rigid transformations.
Moreover, by letting $\mathbb{C}\mathrm{ov} (\boldsymbol{S}) = \boldsymbol{S} \boldsymbol{S}^{\top} = \boldsymbol{\Lambda}$ we fix the gauge freedom  caused by rotations.

\begin{lemma}
\label{lemma: the relation of shape covariance matrices after rigid transformation}
$\mathbb{C}\mathrm{ov} (\boldsymbol{R} \boldsymbol{S} +
\boldsymbol{t} \boldsymbol{1}^{\top})
=
\boldsymbol{R}\,
\mathbb{C}\mathrm{ov} (\boldsymbol{S})
\boldsymbol{R}^{\top}
$
for any arbitrary rotation $\boldsymbol{R}$ and translation $\boldsymbol{t}$.
\end{lemma}
\begin{proof}
By the fact that
$
\mathbf{mean}(\boldsymbol{R}\boldsymbol{S} +
\boldsymbol{t} \boldsymbol{1}^{\top})
=
\frac{1}{m}\left(\boldsymbol{R}\boldsymbol{S} +
\boldsymbol{t} \boldsymbol{1}^{\top}\right)
\boldsymbol{1}
=
\boldsymbol{R}\,
\mathbf{mean}(\boldsymbol{S}) + \boldsymbol{t}
$,
we thus have
$
\left(\boldsymbol{R}\boldsymbol{S} +
\boldsymbol{t} \boldsymbol{1}^{\top} \right)
-
\mathbf{mean}(\boldsymbol{R}\boldsymbol{S} +
\boldsymbol{t} \boldsymbol{1}^{\top})
\boldsymbol{1}^{\top}
= 
\boldsymbol{R} \left( \boldsymbol{S} - \mathbf{mean}(\boldsymbol{S})\boldsymbol{1}^{\top} \right)
$.
\end{proof}

Let $\mathbb{C}\mathrm{ov} (\boldsymbol{S}) = \boldsymbol{U} \boldsymbol{\Lambda} \boldsymbol{U}^{\top}$ be the eigenvalue decomposition of $\mathbb{C}\mathrm{ov} (\boldsymbol{S})$.
Then
$
\boldsymbol{R}\,
\mathbb{C}\mathrm{ov} (\boldsymbol{S})
\boldsymbol{R}^{\top}
=
\boldsymbol{R}
\boldsymbol{U} \boldsymbol{\Lambda} \boldsymbol{U}^{\top}
\boldsymbol{R}^{\top}
$ admits the eigenvalue decomposition of
$\mathbb{C}\mathrm{ov} (\boldsymbol{R} \boldsymbol{S} +
\boldsymbol{t} \boldsymbol{1}^{\top})$ by Lemma \ref{lemma: the relation of shape covariance matrices after rigid transformation}.
This means
matrix $\mathbb{C}\mathrm{ov} (\boldsymbol{R} \boldsymbol{S} +
\boldsymbol{t} \boldsymbol{1}^{\top})
$
has the same eigenvalues as matrix $\mathbb{C}\mathrm{ov} (\boldsymbol{S})$, which are given as the diagonal elements of $\boldsymbol{\Lambda}$.
Moreover, $\mathbb{C}\mathrm{ov} (\boldsymbol{R} \boldsymbol{S} +
\boldsymbol{t} \boldsymbol{1}^{\top}) = \boldsymbol{\Lambda}$ if and only if $\boldsymbol{R} = \boldsymbol{U}^{\top}$ which fixes the rotation.

In general, we want the reference shape to look similar to the datum shapes, thus we choose the eigenvalues of $\boldsymbol{S} \boldsymbol{S}^{\top}$ to be close to those of the datum shape covariance matrices.
Based on this idea, we present a method to estimate $\boldsymbol{\Lambda}$
in Section \ref{subsubsection: Reference Covariance Prior}.
However, as we shall see in what follows, the matrix $\boldsymbol{\Lambda}$ is never used in the intermediate calculation thus can be determined separately as a prior or posterior.

\subsection{Globally Optimal Solution}
\label{subsection. afine GAP, globally optimal solution}

Problem (\ref{eq: affine-GPA, homogeneous formulation}) is separable. Given $\boldsymbol{S}$, we obtain $\boldsymbol{\tilde{A}}_i
= \boldsymbol{S} (\boldsymbol{\tilde{D}}_i)^{\dagger}  
$,
where $(\boldsymbol{\tilde{D}}_i)^{\dagger} = \boldsymbol{\tilde{D}}_i^{\top}
(\boldsymbol{\tilde{D}}_i \boldsymbol{\tilde{D}}_i^{\top})^{-1}$ is the Moore-Penrose pseudo-inverse of $\boldsymbol{\tilde{D}}_i$.
Substituting $\boldsymbol{\tilde{A}}_i
= \boldsymbol{S} (\boldsymbol{\tilde{D}}_i)^{\dagger}  
$ into the cost function of problem (\ref{eq: affine-GPA, homogeneous formulation}), we obtain:
\begin{align*}
&
\sum_{i=1}^n\,
\lVert \boldsymbol{\tilde{A}}_i \boldsymbol{\tilde{D}}_i - \boldsymbol{S} \rVert_F^2
=
\sum_{i=1}^n\,
\lVert \boldsymbol{S} (\boldsymbol{\tilde{D}}_i)^{\dagger} \boldsymbol{\tilde{D}}_i - \boldsymbol{S} \rVert_F^2
\\ = &
\sum_{i=1}^n\,
\mathbf{tr}\left( \boldsymbol{S} 
\left(
\boldsymbol{I} - 
(\boldsymbol{\tilde{D}}_i)^{\dagger} \boldsymbol{\tilde{D}}_i
\right)
\boldsymbol{S}^{\top} \right)
,
\end{align*}
where we have used the fact that matrix
$
\boldsymbol{I} - 
(\boldsymbol{\tilde{D}}_i)^{\dagger} \boldsymbol{\tilde{D}}_i
=
\boldsymbol{I} - 
\boldsymbol{\tilde{D}}_i^{\top}
(\boldsymbol{\tilde{D}}_i \boldsymbol{\tilde{D}}_i^{\top})^{-1}  \boldsymbol{\tilde{D}}_i
$
is symmetric and idempotent (because it is the orthogonal projection matrix to the null space of $\boldsymbol{\tilde{D}}_i$).
Let us denote
$
\boldsymbol{\mathcal{Q}}_{\mathrm{\Rmnum{1}}} = 
\sum_{i=1}^{n} (\boldsymbol{\tilde{D}}_i)^{\dagger} \boldsymbol{\tilde{D}}_i
=
\sum_{i=1}^{n} \boldsymbol{\tilde{D}}_i^{\top}
(\boldsymbol{\tilde{D}}_i \boldsymbol{\tilde{D}}_i^{\top})^{-1}  \boldsymbol{\tilde{D}}_i
$.
By using $\boldsymbol{S} \boldsymbol{S}^{\top} = \boldsymbol{\Lambda}$,
the cost function can be simplified as:
\begin{equation*}
\lVert \boldsymbol{\tilde{A}}_i \boldsymbol{\tilde{D}}_i - \boldsymbol{S} \rVert_F^2
=
n \mathbf{tr} (\boldsymbol{\Lambda})
-
\mathbf{tr}\left( \boldsymbol{S} 
\boldsymbol{\mathcal{Q}}_{\mathrm{\Rmnum{1}}}
\boldsymbol{S}^{\top} \right)
.
\end{equation*}
Therefore after eliminating $\boldsymbol{\tilde{A}}_i$ from the optimization, problem (\ref{eq: affine-GPA, homogeneous formulation}) reduces to:
\begin{equation}
\label{eq: argmax, PCA in S, with S1=0}
\begin{aligned}
&
\argmax_{\boldsymbol{S}}\  
\mathbf{tr}\left( \boldsymbol{S} 
\boldsymbol{\mathcal{Q}}_{\mathrm{\Rmnum{1}}}
\boldsymbol{S}^{\top} \right)
\\
&
\mathrm{s.t.} \quad 
\boldsymbol{S} \boldsymbol{S}^{\top} = \boldsymbol{\Lambda},
\quad
\boldsymbol{S}  \boldsymbol{1} = \boldsymbol{0}
.
\end{aligned}
\end{equation}

\begin{lemma}
$\boldsymbol{\mathcal{Q}}_{\mathrm{\Rmnum{1}}} \boldsymbol{1} =  n \boldsymbol{1}$.
Moreover if we let $
\boldsymbol{\mathcal{P}}_{\mathrm{\Rmnum{1}}} = 
\sum_{i=1}^{n}
\left(\boldsymbol{I} -
(\boldsymbol{\tilde{D}}_i)^{\dagger} \boldsymbol{\tilde{D}}_i
\right)
$
then $\boldsymbol{\mathcal{P}}_{\mathrm{\Rmnum{1}}} \boldsymbol{1} = \boldsymbol{0}$. 
\end{lemma}
\begin{proof}
The matrix $(\boldsymbol{\tilde{D}}_i)^{\dagger} \boldsymbol{\tilde{D}}_i = \boldsymbol{\tilde{D}}_i^{\top}
(\boldsymbol{\tilde{D}}_i \boldsymbol{\tilde{D}}_i^{\top})^{-1}  \boldsymbol{\tilde{D}}_i$
is the orthogonal projection matrix to the range space of $\boldsymbol{\tilde{D}}_i^{\top}$.
Because $\boldsymbol{1} \in \mathbf{Range}(\boldsymbol{\tilde{D}}_i^{\top})$ as $\boldsymbol{1}$ is the last column of $\boldsymbol{\tilde{D}}_i^{\top}$,
we have
$\boldsymbol{\tilde{D}}_i^{\top}
(\boldsymbol{\tilde{D}}_i \boldsymbol{\tilde{D}}_i^{\top})^{-1}  \boldsymbol{\tilde{D}}_i
\boldsymbol{1} = \boldsymbol{1}
$,
thus
$\boldsymbol{\mathcal{Q}}_{\mathrm{\Rmnum{1}}} \boldsymbol{1} = 
\sum_{i=1}^{n} \boldsymbol{\tilde{D}}_i^{\top}
(\boldsymbol{\tilde{D}}_i \boldsymbol{\tilde{D}}_i^{\top})^{-1}  \boldsymbol{\tilde{D}}_i \boldsymbol{1}
= n \boldsymbol{1}
$.
By noticing $\boldsymbol{\mathcal{P}}_{\mathrm{\Rmnum{1}}} + \boldsymbol{\mathcal{Q}}_{\mathrm{\Rmnum{1}}} = n \boldsymbol{I}$,
we have
$\boldsymbol{\mathcal{P}}_{\mathrm{\Rmnum{1}}} \boldsymbol{1} = n \boldsymbol{1} - \boldsymbol{\mathcal{Q}}_{\mathrm{\Rmnum{1}}} \boldsymbol{1} = \boldsymbol{0}$.
\end{proof}

Therefore problem (\ref{eq: argmax, PCA in S, with S1=0}) admits a special property that $\boldsymbol{\mathcal{Q}}_{\mathrm{\Rmnum{1}}}$ 
has an eigenvector $\boldsymbol{1}$ because
$\boldsymbol{\mathcal{Q}}_{\mathrm{\Rmnum{1}}} \boldsymbol{1} =  n \boldsymbol{1}$.
As a result, this problem admits a closed-form solution by the following proposition whose detailed proof requires Theorem \ref{theorem: general result: PCA with constraint XT u = 0} which we will present shortly.

\begin{proposition}
	\label{theorem: standard result: maximization in S, with S1=0}
	If $\boldsymbol{1}$ is an eigenvector of $\boldsymbol{\mathcal{Q}}$,
	then the globally optimal solution to problem
	\begin{equation}
	\label{eq: standard result: maximization in S, with S1=0}
	\begin{aligned}	
	&
	\argmax_{\boldsymbol{S}}\  
	\mathbf{tr}\left( \boldsymbol{S} 
	\boldsymbol{\mathcal{Q}}
	\boldsymbol{S}^{\top} \right)
	\\
	&
	\mathrm{s.t.} \quad
	\boldsymbol{S} \boldsymbol{S}^{\top} = \boldsymbol{\Lambda},
	\quad
	\boldsymbol{S}  \boldsymbol{1} = \boldsymbol{0},	
	\end{aligned}	
	\end{equation}
	is to scale by $\sqrt{\boldsymbol{\Lambda}}$ the $d$ top eigenvectors of $\boldsymbol{\mathcal{Q}}$ excluding the vector $\boldsymbol{1}$.
\end{proposition}
\begin{proof}
	Using
	$
	\boldsymbol{S}^{\top} = \boldsymbol{X} \sqrt{\boldsymbol{\Lambda}} 
	$,
	we can rewrite problem (\ref{eq: standard result: maximization in S, with S1=0}) into problem (\ref{eq: general results: argmax, PCA in S, with S1=0}) with $\boldsymbol{u} = \boldsymbol{1}$.
	By Theorem \ref{theorem: general result: PCA with constraint XT u = 0}, the optimal $\boldsymbol{X}$ is the $d$ top eigenvectors of $\boldsymbol{\mathcal{Q}}$ excluding the vector $\boldsymbol{1}$.
\end{proof}

We summarize the globally optimal solution to the affine GPA formulation (\ref{eq: affine-GPA, homogeneous formulation}) as follow.
\begin{summary}
	Problem (\ref{eq: affine-GPA, homogeneous formulation}) is solved in closed-form.
	The optimal reference shape $\boldsymbol{S}^{\star}$ is obtained by scaling by $\sqrt{\boldsymbol{\Lambda}}$ the $d$ top eigenvectors of $\boldsymbol{\mathcal{Q}}_{\mathrm{\Rmnum{1}}}$ excluding the vector $\boldsymbol{1}$.
	The optimal affine transformations are given by $\boldsymbol{\tilde{A}}_i^{\star}
	=  [\boldsymbol{A}_i^{\star},\, \boldsymbol{t}_i^{\star}]
	=
	\boldsymbol{S}^{\star} \boldsymbol{\tilde{D}}_i^{\top}
	(\boldsymbol{\tilde{D}}_i \boldsymbol{\tilde{D}}_i^{\top})^{-1}  
	$
	$\left(i \in \left[ 1:n \right]\right)$.
\end{summary}

Now we present the supporting results to Proposition \ref{theorem: standard result: maximization in S, with S1=0},
\textit{i.e.,}
Theorem \ref{theorem: general result: PCA with constraint XT u = 0} which gives the solution to problem (\ref{eq: general results: argmax, PCA in S, with S1=0}).
We start with the following Lemma.

\begin{lemma}
	\label{lemma: a result on eigenvalue decomposition of symmetric matrix}
	Let $\boldsymbol{M} \in \mathbb{R}^{m \times m}$ be a symmetric matrix,
	and
	$(\alpha_j, \boldsymbol{\xi}_j)$
	$(j \in \left[ 1:m \right])$ the set of eigenvalues and unit eigenvectors of $\boldsymbol{M}$ such that
	$
	\boldsymbol{M} \boldsymbol{\xi}_j = \alpha_j \boldsymbol{\xi}_j
	$,
	$
	\lVert \boldsymbol{\xi}_j \rVert_2 = 1
	$.
	Then $\boldsymbol{M}$ can be written as:
	$
	\boldsymbol{M} = \sum_{j=1}^{m} \alpha_j \boldsymbol{\xi}_j \boldsymbol{\xi}_j^{\top}
	$.	
\end{lemma}

By Lemma \ref{lemma: a result on eigenvalue decomposition of symmetric matrix},
given an arbitrary eigenvector $\boldsymbol{\xi}_k$ of $\boldsymbol{M}$, by adding multiples of $\boldsymbol{\xi}_k
\boldsymbol{\xi}_k^{\top}$, saying $a \boldsymbol{\xi}_k
\boldsymbol{\xi}_k^{\top}$ ($a$ is an arbitrary real number) to $\boldsymbol{M}$,
we have
$\boldsymbol{M} + a \boldsymbol{\xi}_k
\boldsymbol{\xi}_k^{\top}
= 
\sum_{j=1, j\neq k}^{m} \alpha_j \boldsymbol{\xi}_j \boldsymbol{\xi}_j^{\top}
+
\left( \alpha_k + a \right)
\boldsymbol{\xi}_k
\boldsymbol{\xi}_k^{\top}
$.
Therefore the matrix $\boldsymbol{M} + a \boldsymbol{\xi}_k
\boldsymbol{\xi}_k^{\top}$ has exactly the same set of eigenvectors as $\boldsymbol{M}$.
In $\boldsymbol{M} + a \boldsymbol{\xi}_k
\boldsymbol{\xi}_k^{\top}$,
the eigenvalue of $\boldsymbol{\xi}_k$ becomes $\alpha_k + a$, while the rest remains unchanged as that in $\boldsymbol{M}$.

\begin{theorem}
	\label{theorem: general result: PCA with constraint XT u = 0}
	If $\boldsymbol{u}$ is an eigenvector of $\boldsymbol{\mathcal{Q}}$,
	then the solution to the optimization problem	
	\begin{equation}
	\label{eq: general results: argmax, PCA in S, with S1=0}
	\begin{aligned}
	&
	\argmax_{\boldsymbol{X}}\  
	\mathbf{tr}\left( \boldsymbol{X}^{\top} 
	\boldsymbol{\mathcal{Q}}
	\boldsymbol{X} \boldsymbol{\Lambda} \right)
	\\
	&
	\mathrm{s.t.} \quad
	\boldsymbol{X}^{\top} \boldsymbol{X} = \boldsymbol{I},
	\quad
	\boldsymbol{X}^{\top}  \boldsymbol{u} = \boldsymbol{0}	
	\end{aligned}	
	\end{equation}
	is the $d$ top eigenvectors of $\boldsymbol{\mathcal{Q}}$ excluding the vector $\boldsymbol{u}$.
\end{theorem}
\begin{proof}
Since $\boldsymbol{X}^{\top}  \boldsymbol{u} = \boldsymbol{0}$, problem (\ref{eq: general results: argmax, PCA in S, with S1=0}) is equivalent to the following:
\begin{equation}
\label{eq: property: proof - equivalent formulation in X}
\begin{aligned}
&
\argmax_{\boldsymbol{X}}\  
\mathbf{tr}\left( \boldsymbol{X}^{\top} 
\left( \boldsymbol{\mathcal{Q}} -  c \boldsymbol{u} \boldsymbol{u}^{\top} \right)
\boldsymbol{X} \boldsymbol{\Lambda} \right)
\\
&
 \mathrm{s.t.} \quad
\boldsymbol{X}^{\top}  \boldsymbol{X} = \boldsymbol{I},
\quad
\boldsymbol{X}^{\top}  \boldsymbol{u} = \boldsymbol{0}
,
\end{aligned}
\end{equation}
where $c$ is an arbitrary scalar.

Now we consider the following relaxation of problem (\ref{eq: property: proof - equivalent formulation in X}) without constraint $\boldsymbol{X}^{\top}  \boldsymbol{u} = \boldsymbol{0}$, and assume $c$ is sufficiently big:
\begin{equation}
\label{eq: property: proof - final result in X}
\begin{aligned}
&
\argmax_{\boldsymbol{X}}\  
\mathbf{tr}\left( \boldsymbol{X}^{\top} 
\left( \boldsymbol{\mathcal{Q}} -  c \boldsymbol{u} \boldsymbol{u}^{\top} \right)
\boldsymbol{X} \boldsymbol{\Lambda} \right)
\\
&
\mathrm{s.t.} \quad
\boldsymbol{X}^{\top}  \boldsymbol{X} = \boldsymbol{I}
,
\end{aligned}
\end{equation}
which admits a standard Brockett cost function on the Stiefel manifold (see Section \ref{section. Brockett Cost Function on the Stiefel Manifold}).
The optimal solution of problem (\ref{eq: property: proof - final result in X}), saying $\boldsymbol{X}_*$, comprises the $d$ top eigenvectors of $\boldsymbol{\mathcal{Q}} - c \boldsymbol{u} \boldsymbol{u}^{\top}$.

Since $\boldsymbol{u}$ is an eigenvector of $\boldsymbol{\mathcal{Q}}$, matrix $\boldsymbol{\mathcal{Q}} -  c \boldsymbol{u} \boldsymbol{u}^{\top}$ has the same set of eigenvectors as $\boldsymbol{\mathcal{Q}}$.
By assuming $c$ to be sufficiently big,  it is always possible to shift $\boldsymbol{u}$ to the bottom-eigenvector of $\boldsymbol{\mathcal{Q}} -  c \boldsymbol{u} \boldsymbol{u}^{\top}$, thus the eigenvector $\boldsymbol{u}$ is always excluded from $\boldsymbol{X}_*$.
Thus the $d$ top eigenvectors of $\boldsymbol{\mathcal{Q}} - c \boldsymbol{u} \boldsymbol{u}^{\top}$ are the $d$ top eigenvectors of $\boldsymbol{\mathcal{Q}}$ excluding the eigenvector $\boldsymbol{u}$.

Lastly, the eigenvectors of $\boldsymbol{\mathcal{Q}} - c \boldsymbol{u} \boldsymbol{u}^{\top}$ with respect to different eigenvalues are orthogonal thus $ {\boldsymbol{X}_*}^{\top} \boldsymbol{u} = \boldsymbol{0}$.
Therefore problem (\ref{eq: property: proof - final result in X}) is a tight relaxation of problem (\ref{eq: property: proof - equivalent formulation in X}) if $c$ is sufficiently big.
\end{proof}

\subsection{Estimation of the Reference Covariance Prior}
\label{subsubsection: Reference Covariance Prior}

We estimate the reference covariance prior $\boldsymbol{\Lambda}$ using the eigenvalues of the datum shape covariance matrices $\mathbb{C}\mathrm{ov}(\boldsymbol{D}_i \boldsymbol{D}_i^{\top})
$ $\left(i \in \left[ 1:n \right]\right)$.
By zero-centering each datum shape $\boldsymbol{D}_i$ as
$\boldsymbol{\bar{D}}_i = \boldsymbol{D}_i -
\frac{1}{m}
\boldsymbol{D}_i  \boldsymbol{1} \boldsymbol{1}^{\top}
$, we have
$\mathbb{C}\mathrm{ov}(\boldsymbol{D}_i \boldsymbol{D}_i^{\top})
=
\boldsymbol{\bar{D}}_i \boldsymbol{\bar{D}}_i^{\top}
$.
Since $\boldsymbol{\Lambda}$ is diagonal,
we denote $\boldsymbol{\Lambda} = \mathrm{diag}\left(\lambda_1, \lambda_2, \dots, \lambda_d\right)$
and define the vector
$\boldsymbol{ \lambda } = 
\left[ \lambda_1, \lambda_2, \dots, \lambda_d \right]^{\top}$.
Abusing notations, we collect the eigenvalues of each datum shape covariance $\boldsymbol{D}_i \boldsymbol{D}_i^{\top}$ by a diagonal matrix $\boldsymbol{\Lambda}_i$
and define the vector
$$\boldsymbol{ \lambda }_i = 
\left[ \lambda_1^{(i)}, \lambda_2^{(i)}, \dots, \lambda_d^{(i)} \right]^{\top}.$$
Without loss of generality,
we assume that the elements in $\boldsymbol{ \lambda }_i$ have been sorted in the descending (or non-ascending) order such that $\lambda_1^{(i)} \ge \lambda_2^{(i)} \ge \dots \ge \lambda_d^{(i)} \ge 0$.
The task is now to estimate $\boldsymbol{ \lambda }$ from $\boldsymbol{ \lambda }_i$
$\left( i \in \left[ 1:n \right] \right)$.

To proceed, we consider the geometric implication of:
$$
\sqrt{\boldsymbol{ \lambda }_i} = \begin{bmatrix} \sqrt{\lambda_1^{(i)}}, & \sqrt{\lambda_2^{(i)}}, & \dots, & \sqrt{\lambda_d^{(i)}} \end{bmatrix}^{\top},
$$
\textit{i.e.,} the vector comprising the $d$ leftmost singular values of shape $\boldsymbol{\bar{D}}_i$.
It has been known in \citep{horn1987closed} that the scale of the shape $\boldsymbol{D}_i$ can be represented by $\frac{1}{\sqrt{m}}\lVert \boldsymbol{\bar{D}}_i \rVert_F$. Dropping the common constant $\frac{1}{\sqrt{m}}$, we consider:
\begin{align*}
&
\lVert \boldsymbol{\bar{D}}_i \rVert_F
=
\sqrt{ \mathbf{tr}\left( \boldsymbol{\bar{D}}_i \boldsymbol{\bar{D}}_i^{\top} \right) }
=
\sqrt{ \mathbf{tr}\left( \boldsymbol{\Lambda}_i \right) }
\\ = & 
\sqrt{
	\lambda_1^{(i)} + \lambda_2^{(i)} + \dots + \lambda_d^{(i)}
}
=
\lVert 
\sqrt{\boldsymbol{ \lambda }_i}
\rVert_2
.
\end{align*}
This suggests that $\sqrt{\boldsymbol{ \lambda }_i}$ is a vector with each of its component representing the scale along the corresponding Euclidean axis, and its length the overall shape scale.
Given $\sqrt{\boldsymbol{ \lambda }_1}, \sqrt{\boldsymbol{ \lambda }_2}, \dots, \sqrt{\boldsymbol{ \lambda }_n}$ from $n$ datum shapes,
we are interested in finding $\sqrt{\boldsymbol{\lambda}} = 
\left[ \sqrt{\lambda_1}, \sqrt{\lambda_2}, \dots, \sqrt{\lambda_d} \right]^{\top}$.

The reference shape is defined up to scale (\textit{i.e.,} up to a similarity transformation), thus in $\sqrt{\boldsymbol{\lambda}}$, all that matters is the proportion of its components, namely the direction of the vector $\sqrt{\boldsymbol{\lambda}}$.
Therefore we propose to estimate the direction of $\sqrt{\boldsymbol{\lambda}}$ on the unit ball, denoted by $\boldsymbol{\theta}$, by minimizing the angles between $\boldsymbol{\theta}$ and each $\sqrt{\boldsymbol{\lambda}_i}$ via maximizing their inner products as:
\begin{equation}
\label{eq: optimization problem in reference covariance prior}
\begin{aligned}
\boldsymbol{\theta}^{\star}
= & 
\argmax_{\boldsymbol{\theta} \in \mathbb{R}^d}\,
\sum_{i=1}^n
\left(
\boldsymbol{\theta}^{\top}
\frac{\sqrt{\boldsymbol{ \lambda }_i}}{\lVert 
\sqrt{\boldsymbol{ \lambda }_i} \rVert_2}
\right)^2
\\
&
\mathrm{s.t.} \quad
\lVert \boldsymbol{\theta} \rVert_2 = 1
.
\end{aligned}
\end{equation}
Upon defining
$
\boldsymbol{\Pi}
=
\begin{bmatrix}
\frac{\sqrt{\boldsymbol{ \lambda }_1}}{\lVert 
	\sqrt{\boldsymbol{ \lambda }_1} \rVert_2}
&
\frac{\sqrt{\boldsymbol{ \lambda }_2}}{\lVert 
	\sqrt{\boldsymbol{ \lambda }_2} \rVert_2}
&
\cdots
&
\frac{\sqrt{\boldsymbol{ \lambda }_n}}{\lVert 
	\sqrt{\boldsymbol{ \lambda }_n} \rVert_2}
\end{bmatrix}
$,
problem (\ref{eq: optimization problem in reference covariance prior}) can be rewritten as the Rayleigh quotient optimization:
\begin{align*}
\boldsymbol{\theta}^{\star}
= &
\argmax_{\boldsymbol{\theta} \in \mathbb{R}^d}\,
\boldsymbol{\theta}^{\top} \boldsymbol{\Pi}
\boldsymbol{\Pi}^{\top} \boldsymbol{\theta}
\\ &
\mathrm{s.t.} \quad
\lVert \boldsymbol{\theta} \rVert_2 = 1
,
\end{align*}
whose solution is given in closed-form,
which is the leftmost left singular vector of the matrix $\boldsymbol{\Pi}$,
or equivalently the top eigenvector of the matrix $\boldsymbol{\Pi}
\boldsymbol{\Pi}^{\top}$.
Since $\boldsymbol{\Pi}
\boldsymbol{\Pi}^{\top}$ has non-negative elements, by the Perron–Frobenius theorem,
the elements of $\boldsymbol{\theta}^{\star}$ can be chosen all non-negative.

Given $\boldsymbol{\theta}^{\star}$,
we can choose $\sqrt{\boldsymbol{\lambda}} = s \boldsymbol{\theta}^{\star}$ up to a scale factor $s > 0$.
In practice, it is desirable to choose $\sqrt{\boldsymbol{\lambda}}$ to be at the similar scale of $\sqrt{\boldsymbol{\lambda}_i}$.
Therefore, we choose the average scale
$s  = \frac{1}{n} \sum_{i=1}^{n}  \lVert 
\sqrt{\boldsymbol{ \lambda }_i}
\rVert_2$.

\subsection{The Coordinate Transformation of Datum Shapes}
\label{subsection: affine GPA. coordinate transformation of datum shapes}

We show that by applying (possibly distinct) arbitrary rigid transformations to each datum shape $\boldsymbol{D}_i$,
the optimal reference shape of problem (\ref{eq: affine-GPA, homogeneous formulation}) remains unchanged.
Thus formulation (\ref{eq: affine-GPA, homogeneous formulation}) is unbiased when facing coordinate transformations.

\begin{lemma}
	\label{lemma: applying transformation to a data shape does not change the orthogonal projection matrix}
	Let $\boldsymbol{D}' = \boldsymbol{R}\boldsymbol{D} + \boldsymbol{t} \boldsymbol{1}^{\top}$.
	We further denote:
	\begin{equation*}
	\boldsymbol{\tilde{D}} = \begin{bmatrix}
	\boldsymbol{D} \\ \boldsymbol{1}^{\top}
	\end{bmatrix}
	, \quad
	\boldsymbol{\tilde{D}}' = \begin{bmatrix}
	\boldsymbol{D}' \\ \boldsymbol{1}^{\top}
	\end{bmatrix}
	.
	\end{equation*}	
	Then we have
	$\mathbf{Range} (\boldsymbol{\tilde{D}}'^{\top})
	=
	\mathbf{Range} (\boldsymbol{\tilde{D}}^{\top})$
	and the orthogonal projection matrices thus satisfy:
	\begin{equation*}
	\boldsymbol{\tilde{D}}'^{\top}
	(\boldsymbol{\tilde{D}}'
	\boldsymbol{\tilde{D}}'^{\top})^{-1}
	\boldsymbol{\tilde{D}}'
	=
	\boldsymbol{\tilde{D}}^{\top}
	(\boldsymbol{\tilde{D}}
	\boldsymbol{\tilde{D}}^{\top})^{-1}
	\boldsymbol{\tilde{D}}
	.
	\end{equation*}
\end{lemma}
\begin{proof}
	The matrices $\boldsymbol{\tilde{D}}'^{\top}$ and $\boldsymbol{\tilde{D}}^{\top}$ have the same range space,
	\textit{i.e.,}
	$\mathbf{Range} (\boldsymbol{\tilde{D}}'^{\top})
	=
	\mathbf{Range} (\boldsymbol{\tilde{D}}^{\top})$,
	because:
	\begin{equation*}
	\boldsymbol{\tilde{D}}'^{\top} = \begin{bmatrix}
	\boldsymbol{D}^{\top} \boldsymbol{R}^{\top} + \boldsymbol{1} \boldsymbol{t}^{\top} ,  & \boldsymbol{1}
	\end{bmatrix}
	=
	\boldsymbol{\tilde{D}}^{\top}
	\begin{bmatrix}
	\boldsymbol{R}^{\top} & \boldsymbol{0} \\
	\boldsymbol{t}^{\top} & 1
	\end{bmatrix}
	.
	\end{equation*}
	The orthogonal projection matrices (also called orthogonal projectors) to $\mathbf{Range} (\boldsymbol{\tilde{D}}'^{\top})$ and $\mathbf{Range} (\boldsymbol{\tilde{D}}^{\top})$ are the same by uniqueness \citep{meyer2000matrix}.
\end{proof}

\begin{proposition}
	\label{proposition: affine GPA, the coordinate transformation of datum shapes does not change optimal reference shape}
	In problem (\ref{eq: affine-GPA, homogeneous formulation}),
	when we apply arbitrary rigid transformations to each datum shape $\boldsymbol{D}_i$,
	the matrix $\boldsymbol{\mathcal{Q}}_{\mathrm{\Rmnum{1}}}$ remains the same. So does the optimal reference shape.
\end{proposition}
\begin{proof}
	By Lemma \ref{lemma: applying transformation to a data shape does not change the orthogonal projection matrix},
	matrix
	$\boldsymbol{\tilde{D}}_i^{\top}
	(\boldsymbol{\tilde{D}}_i \boldsymbol{\tilde{D}}_i^{\top})^{-1}  \boldsymbol{\tilde{D}}_i$
	remains the same when we apply arbitrary rigid transformations to the datum shape $\boldsymbol{D}_i$.
	Thus $\boldsymbol{\mathcal{Q}}_{\mathrm{\Rmnum{1}}} = 
	\sum_{i=1}^{n}
	\boldsymbol{\tilde{D}}_i^{\top}
	(\boldsymbol{\tilde{D}}_i \boldsymbol{\tilde{D}}_i^{\top})^{-1}  \boldsymbol{\tilde{D}}_i$
	remains the same as well.
\end{proof}

\subsection{Connection to Classical Results by Eliminating Translations}
\label{sec: eliminating translations - affine case}

We recapitulate the key idea of affine GPA in \citep{rohlf1990extensions} as follows in formulation (\ref{eq: procrustes formulation of zero-translation affine, eigen constraints}),
and show that this approach attains the same result as formulation (\ref{eq: affine-GPA, homogeneous formulation}) while the latter is more general.
In formulation (\ref{eq: affine-GPA, homogeneous formulation}),
the optimal $\boldsymbol{t}_i$ given $\boldsymbol{A}_i$ and $\boldsymbol{S}$ is
$
\boldsymbol{t}_i = - \frac{1}{m} (\boldsymbol{A}_i \boldsymbol{D}_i - \boldsymbol{S}) \boldsymbol{1}
$.
Moreover if $\boldsymbol{S} \boldsymbol{1} = \boldsymbol{0}$, we have $\boldsymbol{t}_i =
- \frac{1}{m} \boldsymbol{A}_i \boldsymbol{D}_i \boldsymbol{1}$.
Substituting the estimate $\boldsymbol{t}_i =
- \frac{1}{m} \boldsymbol{A}_i \boldsymbol{D}_i \boldsymbol{1}$ back to formulation (\ref{eq: procrustes formulation of zero-translation affine, eigen constraints}),
we have:
\begin{equation}
\label{eq: procrustes formulation of zero-translation affine, eigen constraints}
\begin{cases}
\argmin\limits_{\{\boldsymbol{A}_i\},\, \boldsymbol{S}} \quad 
& \sum_{i=1}^n\,
\lVert \boldsymbol{A}_i \boldsymbol{\bar{D}}_i - \boldsymbol{S} \rVert_F^2
\\[10pt]
\mathrm{s.t.} \quad
& \boldsymbol{S} \boldsymbol{S}^{\top} = \boldsymbol{\Lambda}
,
\end{cases}
\end{equation}
where $\boldsymbol{\bar{D}}_i = \boldsymbol{D}_i -
\frac{1}{m}
\boldsymbol{D}_i  \boldsymbol{1} \boldsymbol{1}^{\top}$
$\left(i \in \left[ 1:n \right]\right)$ are zero-centered datum shapes.
Let $\boldsymbol{\mathfrak{Q}}_{\circ} = \sum_{i=1}^{n}
\boldsymbol{\bar{D}}_i^{\top}  (\boldsymbol{\bar{D}}_i\boldsymbol{\bar{D}}_i^{\top})^{-1}  \boldsymbol{\bar{D}}_i$.
Following a similar derivation to Section \ref{subsection. afine GAP, globally optimal solution}, we can reduce problem (\ref{eq: procrustes formulation of zero-translation affine, eigen constraints}) to formulation (\ref{eq: affine compact standard form with trace as maximization}) with $\boldsymbol{\mathcal{Q}} = \boldsymbol{\mathfrak{Q}}_{\circ}$, thus the optimal reference shape $\boldsymbol{S}$ of problem (\ref{eq: procrustes formulation of zero-translation affine, eigen constraints}) is to scale the
$d$ top eigenvectors of
$\boldsymbol{\mathfrak{Q}}_{\circ}$
by $\sqrt{\boldsymbol{\Lambda}}$.
We drop the constraint $\boldsymbol{S} \boldsymbol{1} = \boldsymbol{0}$ in formulation (\ref{eq: procrustes formulation of zero-translation affine, eigen constraints}) because this constraint is automatically satisfied for the optimal $\boldsymbol{S}$ of formulation (\ref{eq: procrustes formulation of zero-translation affine, eigen constraints}), as
$\boldsymbol{\mathfrak{Q}}_{\circ} \boldsymbol{1} = \boldsymbol{0}$ thus $\boldsymbol{1}$ is orthogonal to the $d$ top eigenvectors corresponding to nonzero eigenvalues.

\begin{proposition}
$\boldsymbol{\mathcal{Q}}_{\mathrm{\Rmnum{1}}}
=
\frac{n}{m} \boldsymbol{1} \boldsymbol{1}^{\top} + 
\boldsymbol{\mathfrak{Q}}_{\circ}
$.
The $d$ top eigenvectors of $\boldsymbol{\mathcal{Q}}_{\mathrm{\Rmnum{1}}}$ excluding the vector $\boldsymbol{1}$ are the $d$ top eigenvectors of
$\boldsymbol{\mathfrak{Q}}_{\circ}$,
thus problems (\ref{eq: affine-GPA, homogeneous formulation}) and (\ref{eq: procrustes formulation of zero-translation affine, eigen constraints}) give the same optimal reference shape $\boldsymbol{S}$.
\end{proposition}
\begin{proof}
Let us denote the homogeneous form of the zero-centered datum shape as:
\begin{equation*}
\boldsymbol{\tilde{\bar{D}}}_i = \begin{bmatrix}
\boldsymbol{\bar{D}}_i \\ \boldsymbol{1}^{\top}
\end{bmatrix}
,\quad\mathrm{thus}\quad
\boldsymbol{\tilde{\bar{D}}}_i
\boldsymbol{\tilde{\bar{D}}}_i^{\top}
= 
\begin{bmatrix}
\boldsymbol{\bar{D}}_i \boldsymbol{\bar{D}}_i^{\top} & \boldsymbol{O} \\
\boldsymbol{O} & m
\end{bmatrix}
,
\end{equation*}
where we have used the fact that $\boldsymbol{\bar{D}}_i \boldsymbol{1} = \boldsymbol{0}$.
From Proposition \ref{proposition: affine GPA, the coordinate transformation of datum shapes does not change optimal reference shape}, we have
$
\boldsymbol{\mathcal{Q}}_{\mathrm{\Rmnum{1}}}
=
\sum_{i=1}^{n}
\boldsymbol{\tilde{D}}_i^{\top}
(\boldsymbol{\tilde{D}}_i
\boldsymbol{\tilde{D}}_i^{\top})^{-1}
\boldsymbol{\tilde{D}}_i
=
\sum_{i=1}^{n}
\boldsymbol{\tilde{\bar{D}}}_i^{\top}
\left( \boldsymbol{\tilde{\bar{D}}}_i
\boldsymbol{\tilde{\bar{D}}}_i^{\top}
\right)^{-1}  \boldsymbol{\tilde{\bar{D}}}_i
$.
By a straightforward calculation, it can be verified that:
\begin{equation*}
\boldsymbol{\tilde{\bar{D}}}_i^{\top}
\left( \boldsymbol{\tilde{\bar{D}}}_i
\boldsymbol{\tilde{\bar{D}}}_i^{\top}
\right)^{-1}  \boldsymbol{\tilde{\bar{D}}}_i
=
\frac{1}{m}
\boldsymbol{1}\boldsymbol{1}^{\top} + 
\boldsymbol{\bar{D}}_i^{\top}  (\boldsymbol{\bar{D}}_i\boldsymbol{\bar{D}}_i^{\top})^{-1}  \boldsymbol{\bar{D}}_i
,
\end{equation*}
thus we obtain $\boldsymbol{\mathcal{Q}}_{\mathrm{\Rmnum{1}}}
=
\frac{n}{m} \boldsymbol{1} \boldsymbol{1}^{\top} + 
\boldsymbol{\mathfrak{Q}}_{\circ}
$.
We notice $\boldsymbol{\mathfrak{Q}}_{\circ} \boldsymbol{1} = \boldsymbol{0}$ because of
$\boldsymbol{\bar{D}}_i \boldsymbol{1} = \boldsymbol{0}$, which means $\boldsymbol{1}$ is an eigenvector of $\boldsymbol{\mathfrak{Q}}_{\circ}$ with eigenvalue $0$.
Therefore $\boldsymbol{\mathcal{Q}}_{\mathrm{\Rmnum{1}}}$
and $\boldsymbol{\mathfrak{Q}}_{\circ}$ have the same eigenvectors by Lemma \ref{lemma: a result on eigenvalue decomposition of symmetric matrix},
while $\boldsymbol{1}$ corresponds to the largest eigenvalue $n$ in $\boldsymbol{\mathcal{Q}}_{\mathrm{\Rmnum{1}}}$
and in contrast to the smallest eigenvalue $0$ in $\boldsymbol{\mathfrak{Q}}_{\circ}$.
\end{proof}

\section{Generalized Procrustes Analysis with the Deformation Model}
\label{Generalized Procrustes Analysis with Warp Models}

We consider GPA with the deformation model as nonlinear warps.
In particular, we consider a class of generalized warps, termed LBWs,
whose transformation parameters are linear with respect to the nonlinear basis functions that lift source points to the higher-dimensional feature space.
In this section, we consider full shape registration and postpone the discussion of partial shape registration to Section \ref{section: Extensions to Partial Datum Shapes}.

\subsection{Linear Basis Warps}
\label{subsection: linear basis warps}

A warp is a generalized transformation, that maps a source point $\boldsymbol{p} \in \mathbb{R}^d$ $\left(d\in \left\{2,3\right\}\right)$ to its target point $\boldsymbol{p}' \in \mathbb{R}^d$.
The LBW is expressed as a linear combination of a set of basis functions.

\subsubsection{Formulation}

Let $\boldsymbol{\beta}(\boldsymbol{p})$ be a vector of basis functions:
\begin{equation*}
\boldsymbol{\beta}(\boldsymbol{p}) = 
\left[
\beta_1 (\boldsymbol{p}),\,
\beta_2 (\boldsymbol{p}),
\dots,
\beta_l (\boldsymbol{p})
\right]^{\top}
,
\end{equation*}
with each element $\beta_k (\boldsymbol{p}): \mathbb{R}^d \rightarrowtail \mathbb{R}$
$\left(k \in \left[ 1:l \right]\right)$ being a scalar basis function.
The vectorized basis function
$\boldsymbol{\beta}(\boldsymbol{p}): \mathbb{R}^d \rightarrowtail \mathbb{R}^l$
brings a $d$-dimensional point $\boldsymbol{p}$ to the $l$-dimensional feature space,
thus is also termed a feature mapping in the context of linear regression models \citep{bishop2006pattern}.
Given the basis function $\boldsymbol{\beta}(\boldsymbol{p})$,
we write a warp model,
$\mathcal{W} (\boldsymbol{p}, \boldsymbol{W}): \mathbb{R}^d \rightarrowtail \mathbb{R}^d$,
as:
\begin{equation}
\label{eq - warp model - operating on point}
\mathcal{W} (\boldsymbol{p}, \boldsymbol{W})
= \boldsymbol{W}^{\top} \boldsymbol{\beta}(\boldsymbol{p})
,
\end{equation}
with $\boldsymbol{W} \in \mathbb{R}^{l \times d}$ being the unknown weight matrix.
An example is provided in Appendix \ref{appendix: TPS Standard Form as a Linear Regression Model}, showing how to write the TPS warp in this form.

\subsubsection{Operating on Point Clouds}
\label{subsection: generalized warp models, operating on point cloud}

For a point-cloud of $m$ points in a matrix $\boldsymbol{D} = \left[
\boldsymbol{p}_1,\,
\boldsymbol{p}_2,
\dots,
\boldsymbol{p}_m
\right] \in \mathbb{R}^{d \times m}$.
We apply the warp $\mathcal{W} (\boldsymbol{p}, \boldsymbol{W})$ to each point in $\boldsymbol{D}$ to obtain its warped version.
Abusing notations, we write the result as:
\begin{equation}
\label{eq - warp model - operating on data}
\mathcal{W} (\boldsymbol{D}, \boldsymbol{W}) = \boldsymbol{W}^{\top} \boldsymbol{\mathcal{B}}(\boldsymbol{D})
,
\end{equation}
with $\boldsymbol{\mathcal{B}}(\boldsymbol{D}) \in \mathbb{R}^{l \times m}$ collecting the feature of each point in $\boldsymbol{D}$ as its columns:
\begin{equation*}
\boldsymbol{\mathcal{B}}(\boldsymbol{D}) =
\left[\boldsymbol{\beta}(\boldsymbol{p}_1),\,
\boldsymbol{\beta}(\boldsymbol{p}_2),
\dots,
\boldsymbol{\beta}(\boldsymbol{p}_m)
\right].
\end{equation*}

\subsubsection{Regularization}

The warp is often used with a regularization term to avoid over-fitting, for instance, if the dimension of the feature space is greater or equivalent to the number of points in the point-cloud.
In the context of deformations, such a term is formed from the partial derivatives of the warp, with different physical implications.
In particular, the second-order derivatives are used in the TPS warp:
\begin{equation*}
\mathcal{R} (\boldsymbol{W})
=
\int_{\mathbb{R}^d}\,
\left\|
\frac{\partial^2}{\partial \boldsymbol{p}^2}
\mathcal{W} (\boldsymbol{p}, \boldsymbol{W})
\right\|_F^2\, d \boldsymbol{p}
.
\end{equation*}
This is directly proportional to the bending energy.
The bending energy term is exactly zero if and only if the warp is affine \citep{bookstein1989principal}.
Other possibilities of regularizations include
the spring term used in elastic registration \citep{christensen2001consistent},
and the
viscosity term used in fluid registration
\citep{bro1996fast}.

For the TPS warp, the integral can be solved in closed-form, as a quadratic form of the transformation parameters:
\begin{equation}
\label{eq: standard form regularization term}
\mathcal{R} (\boldsymbol{W}) =
\lVert
\boldsymbol{Z} \boldsymbol{W}
\rVert_F^2.
\end{equation}
Here $\boldsymbol{Z}$ is given as the square root of the bending energy matrix, see Appendix \ref{appendix, TPS Regularization by Bending Energy Matrix}.
For other warps, we assume the regularization term can be fairly approximated by the quadratic form as well.

The regularization term is however not compulsory.
It is always possible to avoid over-fitting by limiting the dimension of the feature space, by choosing a smaller $l \ll m$
\citep{rueckert1999nonrigid}.
For completeness of the discussion, we consider the case with regularization.

\subsubsection{Examples of Linear Basis Warps}

The LBW generalizes over many deformation models,
\textit{e.g.}~the Free-Form Deformations (FFD) \citep{rueckert1999nonrigid, szeliski1997spline},
and
the Radial Basis Functions (RBF) \citep{bookstein1989principal,fornefett2001radial}.
Concretely we will use the Thin-Plate Spline (TPS) \citep{duchon1976interpolation, bookstein1989principal},
a theoretically principled RBF that minimizes the overall bending energy,
in the practical implementation of our theory.
An introduction of the TPS warp as an LBW is provided in Appendix \ref{Appendix: Thin-Plate Spline}.

The affine transformation is a special case of the LBW without regularization.
This can be shown from the homogeneous form in equation (\ref{eq: homogeneous representation of affine transformation}),
by setting:
\begin{equation*}
\boldsymbol{W}^{\top} = \begin{bmatrix}
\boldsymbol{A}, & \boldsymbol{t}
\end{bmatrix}
, \quad
\boldsymbol{\mathcal{B}} (\boldsymbol{D}) = 
\begin{bmatrix}
\boldsymbol{D} \\
\boldsymbol{1}^{\top}
\end{bmatrix}
.
\end{equation*}

\subsection{Generalized Procrustes Analysis with Linear Basis Warps}
\label{subsection: Generalized Procrustes Analysis with Linear Basis Warps}

We consider the case without missing datum points, and use the LBW in formulation (\ref{eq: GPA - model - original visibility matrix}).
We constrain the reference shape to be zero centered by $\boldsymbol{S}  \boldsymbol{1} = \boldsymbol{0}$.
The shape constraint $\boldsymbol{S} \boldsymbol{S}^{\top} = \boldsymbol{\Lambda}$ is used to enforce the rigidity of the solution,
and the reference covariance prior $\boldsymbol{\Lambda}$ is estimated as in
Section \ref{subsubsection: Reference Covariance Prior}.
We propose the following formulation for deformable GPA:
\begin{equation}
\label{eq: Procrustes with warp models - S 1 = 0, SS =L}
\mathrm{\Rmnum{2}:}
\begin{cases}
\argmin\limits_{\{\boldsymbol{W}_i\},\, \boldsymbol{S}}\quad
& \sum_{i=1}^n\,
\lVert
\boldsymbol{W}_i^{\top} \boldsymbol{\mathcal{B}}_i(\boldsymbol{D}_i) 
- 
\boldsymbol{S}
\rVert_{F}^2	
\\[5pt] &  +
\sum_{i=1}^n\,
\mu_i
\lVert
\boldsymbol{Z}_i \boldsymbol{W}_i
\rVert_F^2
\\[10pt]
\mathrm{s.t.}\quad
& 	\boldsymbol{S} \boldsymbol{S}^{\top} = \boldsymbol{\Lambda},
\quad
\boldsymbol{S}  \boldsymbol{1} = \boldsymbol{0}.
\end{cases}
\end{equation}
Although being termed deformable GPA, formulation (\ref{eq: Procrustes with warp models - S 1 = 0, SS =L}) includes the affine-GPA formulation (\ref{eq: affine-GPA, homogeneous formulation}) in homogeneous form as a special case.
The translation part cannot be identified directly from the LBW, thus needs to be estimated jointly inside the transformation parameters.

To proceed, we define the shorthand
$\boldsymbol{\mathcal{B}}_i
\defeq
\boldsymbol{\mathcal{B}}_i(\boldsymbol{D}_i)
$.
The problem is separable.
Given a reference shape $\boldsymbol{S}$, the estimate of $\boldsymbol{W}_i$ is given in closed-form by solving the linear least-squares problem whose solution is:
\begin{equation}
\label{eq: deformable GPA, weight paramters with respect to S}
	\boldsymbol{W}_i
	=
	\left(\boldsymbol{\mathcal{B}}_i
	\boldsymbol{\mathcal{B}}_i^{\top}
	+
	\mu_i \boldsymbol{Z}_i^{\top}  \boldsymbol{Z}_i\right)^{-1}
	\boldsymbol{\mathcal{B}}_i
	{\boldsymbol{S}}^{\top}
	.
\end{equation}
Substituting equation (\ref{eq: deformable GPA, weight paramters with respect to S}) into problem (\ref{eq: Procrustes with warp models - S 1 = 0, SS =L}),
we obtain an optimization problem with respect to $\boldsymbol{S}$ only:
\begin{equation}
\label{eq: deformable-GPA minization, in S only, S1 = 0}
\begin{aligned}
&
\argmin_{\boldsymbol{S}}\  
\mathbf{tr}\left(\boldsymbol{S} \boldsymbol{\mathcal{P}}_{\mathrm{\Rmnum{2}}} \boldsymbol{S}^{\top}\right)
\\
&
\mathrm{s.t.} \quad
\boldsymbol{S} \boldsymbol{S}^{\top} = \boldsymbol{\Lambda},
\quad
\boldsymbol{S}  \boldsymbol{1} = \boldsymbol{0}
,
\end{aligned}
\end{equation}
where matrix $\boldsymbol{\mathcal{P}}_{\mathrm{\Rmnum{2}}}$ is defined as:
\begin{equation*}
\boldsymbol{\mathcal{P}}_{\mathrm{\Rmnum{2}}}
=
\sum_{i=1}^{n} \left(\boldsymbol{I} - \boldsymbol{\mathcal{B}}_i^{\top}
\left(\boldsymbol{\mathcal{B}}_i
\boldsymbol{\mathcal{B}}_i^{\top}
+
\mu_i \boldsymbol{Z}_i^{\top}  \boldsymbol{Z}_i\right)^{-1}
\boldsymbol{\mathcal{B}}_i \right)
.
\end{equation*}
We define $\boldsymbol{Q}_i
=
\boldsymbol{\mathcal{B}}_i^{\top}
\left(\boldsymbol{\mathcal{B}}_i
\boldsymbol{\mathcal{B}}_i^{\top}
+
\mu_i \boldsymbol{Z}_i^{\top}  \boldsymbol{Z}_i\right)^{-1}
\boldsymbol{\mathcal{B}}_i$,
and:
\begin{equation*}
\boldsymbol{\mathcal{Q}}_{\mathrm{\Rmnum{2}}}
= \sum_{i=1}^{n} \boldsymbol{\mathcal{B}}_i^{\top}
\left(\boldsymbol{\mathcal{B}}_i
\boldsymbol{\mathcal{B}}_i^{\top}
+
\mu_i \boldsymbol{Z}_i^{\top}  \boldsymbol{Z}_i\right)^{-1}
\boldsymbol{\mathcal{B}}_i
= \sum_{i=1}^{n} \boldsymbol{Q}_i
.
\end{equation*}
Then
$\boldsymbol{\mathcal{P}}_{\mathrm{\Rmnum{2}}}
=
\sum_{i=1}^{n} \left(\boldsymbol{I} - \boldsymbol{Q}_i\right)
=
n \boldsymbol{I}
-
\sum_{i=1}^{n}
\boldsymbol{Q}_i
$,
and
$
\mathbf{tr}\left(\boldsymbol{S} \boldsymbol{\mathcal{P}}_{\mathrm{\Rmnum{2}}} \boldsymbol{S}^{\top}\right)
=
n \mathbf{tr}\left( \boldsymbol{\Lambda} \right)
-
\mathbf{tr}\left(\boldsymbol{S} \boldsymbol{\mathcal{Q}}_{\mathrm{\Rmnum{2}}} \boldsymbol{S}^{\top}\right)
$.
Problem (\ref{eq: deformable-GPA minization, in S only, S1 = 0})
can thus be equivalently written as a maximization problem:
\begin{equation}
\label{eq: standard form - maximization - deformable GPA - in S only with S1 = 0}
\begin{aligned}
&
\argmax_{\boldsymbol{S}}\  
\mathbf{tr}\left(\boldsymbol{S} \boldsymbol{\mathcal{Q}}_{\mathrm{\Rmnum{2}}} \boldsymbol{S}^{\top}\right)
\\
&
\mathrm{s.t.} \quad
\boldsymbol{S} \boldsymbol{S}^{\top} = \boldsymbol{\Lambda},
\quad
\boldsymbol{S}  \boldsymbol{1} = \boldsymbol{0}
.
\end{aligned}
\end{equation}

If $\boldsymbol{1}$ is an eigenvector of $\boldsymbol{\mathcal{Q}}_{\mathrm{\Rmnum{2}}}$,
or equivalently if $\boldsymbol{1}$ is an eigenvector of $\boldsymbol{\mathcal{P}}_{\mathrm{\Rmnum{2}}}$,
then Problem (\ref{eq: deformable-GPA minization, in S only, S1 = 0}) and Problem (\ref{eq: standard form - maximization - deformable GPA - in S only with S1 = 0}) can be solved globally by Proposition \ref{theorem: standard result: maximization in S, with S1=0}.
Once we obtain the optimal reference shape $\boldsymbol{S}^{\star}$, the optimal transformation parameters can be calculated by equation (\ref{eq: deformable GPA, weight paramters with respect to S}).

\subsection{Eigenvector Characterization and Globally Optimal Solution}
\label{section: identifying translations in transformation models}

Now we characterize a class of LBWs for which the resulting $\boldsymbol{\mathcal{P}}_{\mathrm{\Rmnum{2}}}$
and
$\boldsymbol{\mathcal{Q}}_{\mathrm{\Rmnum{2}}}$
have an eigenvector $\boldsymbol{1}$.
To this end, we prove that the following statements are equivalent.
The proofs have been moved to Appendix \ref{appendix: proof of Theorem 1 - four equivalent statements} to benefit easy reading.
\begin{theorem}
	\label{label: theorem: three equivalent conditions, p, q , cost}
	The following statements are equivalent:
	\begin{enumerate}
		\item[$(a)$] $\boldsymbol{\mathcal{P}}_{\mathrm{\Rmnum{2}}} \boldsymbol{1} = \boldsymbol{0}$.
		\item[$(b)$] $\boldsymbol{\mathcal{Q}}_{\mathrm{\Rmnum{2}}} \boldsymbol{1} = n \boldsymbol{1}$.
		\item[$(c)$]
		$\boldsymbol{Q}_i \boldsymbol{1} = \boldsymbol{1}$.
		\item[$(d)$]
		The cost function $\mathbf{tr}\left(\boldsymbol{S} \boldsymbol{\mathcal{P}}_{\mathrm{\Rmnum{2}}} \boldsymbol{S}^{\top}\right)$ is invariant to translations.
		\item[$(e)$]
		There exists $\boldsymbol{x}$ such that $\boldsymbol{\mathcal{B}}_i(\boldsymbol{D}_i)^{\top} \boldsymbol{x}
		= 
		\boldsymbol{1}$;
		moreover if $\mu_i > 0$,
		$\boldsymbol{x}$ must satisfy
		$\boldsymbol{Z}_i \boldsymbol{x} = \boldsymbol{0}$.
	\end{enumerate}	
\end{theorem}

Theorem \ref{label: theorem: three equivalent conditions, p, q , cost} relates various aspects of the LBW based GPA to the existence of an eigenvector $\boldsymbol{1}$ in $\boldsymbol{\mathcal{P}}_{\mathrm{\Rmnum{2}}}$ and $\boldsymbol{\mathcal{Q}}_{\mathrm{\Rmnum{2}}}$.
In particular,
case $(d)$ states that the reduced problem is invariant to translations,
and case $(e)$ stipulates the rules that the LBW as a mapping must follow.
While case $(e)$ in Theorem \ref{label: theorem: three equivalent conditions, p, q , cost}
is related to the datum shape $\boldsymbol{D}_i$,
we show in what follows that it can be satisfied if the LBW satisfies certain property, making the statement independent of the input datum shapes.

Now we show a sufficient condition to case $(e)$ in
Theorem \ref{label: theorem: three equivalent conditions, p, q , cost} which is
that the LBW must contain \textit{free-translations}.
We drop the subscript $i$ and use the notation $\boldsymbol{\mathcal{B}}\left(\cdot\right)$ to indicate that this is a property of the LBW thus is independent of the warp input.

\begin{definition}
	\label{definition: LBW contains free-translations}
	The LBW, given by $\{\boldsymbol{W}^{\top} \boldsymbol{\mathcal{B}}\left(\cdot\right),\, \boldsymbol{Z} \boldsymbol{W},\, \mu \}$ is said to contain \textit{free-translations} if:
	there exists $\boldsymbol{x}$ such that $\boldsymbol{\mathcal{B}}\left(\cdot\right)^{\top} \boldsymbol{x}
	= 
	\boldsymbol{1}$;
	moreover if $\mu > 0$,
	$\boldsymbol{x}$ must satisfy
	$\boldsymbol{Z} \boldsymbol{x} = \boldsymbol{0}$.
\end{definition}

We now explain the idea behind Definition \ref{definition: LBW contains free-translations}.
The LBW contains free-translations,
if we can recover the translation vector $\boldsymbol{t}$ explicitly by an invertible matrix $\boldsymbol{G}$ such that:
\begin{align*}
\boldsymbol{W}^{\top} \boldsymbol{\mathcal{B}}\left(\cdot\right)
= &
\boldsymbol{W}^{\top}
\boldsymbol{G}^{-1}
\left(\boldsymbol{\mathcal{B}}\left(\cdot\right)^{\top} \boldsymbol{G}^{\top}\right)^{\top}
\\[5pt] = &
\underbracket{
\left[
\, \cdots \, \boldsymbol{t} \, \cdots \,
\right]}_{\boldsymbol{W}^{\top}\boldsymbol{G}^{-1}}\,
{\underbracket{
\left[
\, \cdots \, \boldsymbol{1} \, \cdots \,
\right]
}_{\boldsymbol{\mathcal{B}}\left(\cdot\right)^{\top} \boldsymbol{G}^{\top}} }^{\top},
\end{align*}
where $\boldsymbol{W}^{\top}
\boldsymbol{G}^{-1}$ contains a column vector $\boldsymbol{t}$
and $\boldsymbol{\mathcal{B}}\left(\cdot\right)^{\top} \boldsymbol{G}^{\top}$ contains a column vector $\boldsymbol{1}$.
Therefore
$\boldsymbol{1} \in
\mathbf{Range}\left( \boldsymbol{\mathcal{B}}\left(\cdot\right)^{\top} \boldsymbol{G}^{\top} \right)
=
\mathbf{Range}\left( \boldsymbol{\mathcal{B}}\left(\cdot\right)^{\top}
\right)
$.
Using $\boldsymbol{G}$, the term $\boldsymbol{Z} \boldsymbol{W}$ can be decomposed as:
\begin{align*}
\boldsymbol{Z} \boldsymbol{W}
& =
\boldsymbol{Z} \boldsymbol{G}^{\top}
\left(\boldsymbol{W}^{\top}  \boldsymbol{G}^{-1}\right)^{\top}
\\[5pt] & =
\boldsymbol{Z} \boldsymbol{G}^{\top}\,
{\underbracket{\left[\, \cdots \, \boldsymbol{t} \, \cdots\, \right]}_{\boldsymbol{W}^{\top}\boldsymbol{G}^{-1}}}^{\top}.
\end{align*}
Without loss of generality, we assume that in ${\boldsymbol{\mathcal{B}}\left(\cdot\right)^{\top} \boldsymbol{G}^{\top}}$ the $k$-th column vector is $\boldsymbol{1}$,
identified by the standard basis vector $\boldsymbol{e}_k$, such that
$\boldsymbol{\mathcal{B}}\left(\cdot\right)^{\top} \boldsymbol{G}^{\top}
\boldsymbol{e}_k = \boldsymbol{1}$.
Then it can be easily verified that the translation $\boldsymbol{t}$ is constraint-free in $\boldsymbol{Z} \boldsymbol{W}$
if and only if
the $k$-th column of $\boldsymbol{Z} \boldsymbol{G}^{\top}$ is $\boldsymbol{0}$,
\textit{i.e.},
$\boldsymbol{Z} \boldsymbol{G}^{\top} \boldsymbol{e}_k = \boldsymbol{0}$.
Since $\boldsymbol{\mathcal{B}}\left(\cdot\right)^{\top}$ has full column rank, so $\boldsymbol{G}^{\top}\boldsymbol{e}_k$ is the unique $\boldsymbol{x}$ such that $\boldsymbol{\mathcal{B}}\left(\cdot\right)^{\top} \boldsymbol{x}
= 
\boldsymbol{1}$,
$\boldsymbol{Z} \boldsymbol{x} = \boldsymbol{0}$.

\begin{proposition}
	\label{theorem: P1 = 0 if and only if the LBW contains free-translations}	
	The statements in Theorem \ref{label: theorem: three equivalent conditions, p, q , cost} are satisfied if the LBW contains free-translations.
\end{proposition}

We state that the affine transformation and the TPS warp satisfy Definition \ref{definition: LBW contains free-translations}.
The proof for the affine case is straightforward.
For the TPS warp, the algebraic proof is given in Appendix \ref{appendix: theorem equivalent conditions are satisified for the TPS warp}.
Intuitively, this can be explained because
for the TPS warp, the bending energy term only affects the nonlinear part of the warp, which means that the linear part of the TPS warp is constraint-free and so is its translational component.

\begin{proposition}
	\label{theorem: affine, and TPS satisfy the desirabed property to be solved globally}
	The affine transformation and the TPS warp satisfy the statement of Theorem \ref{label: theorem: three equivalent conditions, p, q , cost}.
\end{proposition}

The above results are summarized as follow:
\begin{summary}
	Formulation (\ref{eq: Procrustes with warp models - S 1 = 0, SS =L}) can be solved globally
	if the the LBW contains free-translations.
	The optimal reference shape $\boldsymbol{S}^{\star}$ is to scale by $\sqrt{\boldsymbol{\Lambda}}$ the $d$ top eigenvectors of $\boldsymbol{\mathcal{Q}}_{\mathrm{\Rmnum{2}}}$ (or equivalently the $d$ bottom eigenvectors of $\boldsymbol{\mathcal{P}}_{\mathrm{\Rmnum{2}}}$) excluding the eigenvector $\boldsymbol{1}$.
	The optimal transformation parameters are
	$
	\boldsymbol{W}_i
	=
	\left(\boldsymbol{\mathcal{B}}_i
	\boldsymbol{\mathcal{B}}_i^{\top}
	+
	\mu_i \boldsymbol{Z}_i^{\top}  \boldsymbol{Z}_i\right)^{-1}
	\boldsymbol{\mathcal{B}}_i
	{\boldsymbol{S}^{\star}}^{\top}
	$
	$\left(i \in \left[ 1:n \right]\right)$.
\end{summary}

\subsection{Reformulation using the Soft Constraint}
\label{Reformulation using soft constraint}

In Theorem \ref{theorem: general result: PCA with constraint XT u = 0},
we have proved the equivalence of formulations (\ref{eq: general results: argmax, PCA in S, with S1=0}) and (\ref{eq: property: proof - final result in X}).
Therefore, formulation (\ref{eq: standard form - maximization - deformable GPA - in S only with S1 = 0}) can be equivalently written as:
\begin{equation}
\label{eq: problem: maximization, Q - n 11, PCA form}
\begin{aligned}
&
\argmax_{\boldsymbol{S}}\  
\mathbf{tr}\left(\boldsymbol{S}
\left(
\boldsymbol{\mathcal{Q}}_{\mathrm{\Rmnum{2}}}
- n \boldsymbol{\bar{1}}\boldsymbol{\bar{1}}^{\top}
\right)
\boldsymbol{S}^{\top}\right)
\\
&
\mathrm{s.t.} \quad
\boldsymbol{S} \boldsymbol{S}^{\top} = \boldsymbol{\Lambda},
\end{aligned}
\end{equation}
where we have used
$\boldsymbol{\mathcal{Q}}_{\mathrm{\Rmnum{2}}} \boldsymbol{\bar{1}} = n \boldsymbol{\bar{1}}$,
with $\boldsymbol{\bar{1}}$ being normalized.

In problem (\ref{eq: problem: maximization, Q - n 11, PCA form}),
we can replace $n$ with any $n' \ge n$.
This is because $\left(
\boldsymbol{\mathcal{Q}}_{\mathrm{\Rmnum{2}}}
-
n'
\boldsymbol{\bar{1}}
\boldsymbol{\bar{1}}^{\top}
\right)
\boldsymbol{\bar{1}}
=
(n - n') \boldsymbol{\bar{1}}
$,
where the eigenvalue corresponding to the eigenvector $\boldsymbol{\bar{1}}$ becomes $(n - n') \le 0$,
while the rest of the eigenvalue of $\left(
\boldsymbol{\mathcal{Q}}_{\mathrm{\Rmnum{2}}}
- \nu \boldsymbol{\bar{1}}\boldsymbol{\bar{1}}^{\top}
\right)$ is nonnegative.
Thus the eigenvector $\boldsymbol{\bar{1}}$ is excluded in the solution for any $n' \ge n$.
The rest of the eigenvectors remain unchanged.
Expanding the cost of problem (\ref{eq: problem: maximization, Q - n 11, PCA form}),
we obtain:
\begin{equation}
\begin{aligned}
&
\argmax_{\boldsymbol{S}}\  
\mathbf{tr}\left(\boldsymbol{S}
\left(
\boldsymbol{\mathcal{Q}}_{\mathrm{\Rmnum{2}}}
\right)
\boldsymbol{S}^{\top}\right)
- n'
\lVert
\boldsymbol{S} \boldsymbol{\bar{1}}
\rVert_2^2
\\
&
\mathrm{s.t.} \quad
\boldsymbol{S} \boldsymbol{S}^{\top} = \boldsymbol{\Lambda},
\end{aligned}
\end{equation}
where $n' \ge n$.
Finally we substitute $\boldsymbol{\mathcal{P}}_{\mathrm{\Rmnum{2}}}
=
n \boldsymbol{I}
-
\boldsymbol{\mathcal{Q}}_{\mathrm{\Rmnum{2}}}
$
and
$\boldsymbol{\bar{1}} = \frac{1}{m}\boldsymbol{1}$,
then obtain a minimization problem:
\begin{equation}
\label{eq: deformable GPA in S only, minimization, with soft constraint}
\begin{aligned}
&
\argmin_{\boldsymbol{S}}\  
\mathbf{tr}\left(\boldsymbol{S}
\left(
\boldsymbol{\mathcal{P}}_{\mathrm{\Rmnum{2}}}
\right)
\boldsymbol{S}^{\top}\right)
+ \nu
\lVert
\boldsymbol{S} \boldsymbol{1}
\rVert_2^2
\\
&
\mathrm{s.t.} \quad
\boldsymbol{S} \boldsymbol{S}^{\top} = \boldsymbol{\Lambda},
\end{aligned}
\end{equation}
where $\nu \ge n/m$.
Note that problem (\ref{eq: deformable GPA in S only, minimization, with soft constraint})
is equivalent to problem (\ref{eq: deformable-GPA minization, in S only, S1 = 0}),
while the hard constraint $\boldsymbol{S}  \boldsymbol{1} = \boldsymbol{0}$ has been reformulated as a soft constraint in the form of a penalty term.

At last, formulation (\ref{eq: Procrustes with warp models - S 1 = 0, SS =L}) is equivalent to the following one with the soft constraint:
\begin{equation}
\label{eq: Procrustes with warp models - soft constraint}
\begin{aligned}
\argmin_{\{\boldsymbol{W}_i\},\, \boldsymbol{S}}\quad
& \sum_{i=1}^n\,
\lVert
\boldsymbol{W}_i^{\top} \boldsymbol{\mathcal{B}}_i(\boldsymbol{D}_i) 
- 
\boldsymbol{S}
\rVert_{F}^2	
\\ & +
\sum_{i=1}^n\,
\mu_i
\lVert
\boldsymbol{Z}_i \boldsymbol{W}_i
\rVert_F^2
+
\nu
\lVert
\boldsymbol{S} \boldsymbol{1}
\rVert_2^2
\\[5pt]
\mathrm{s.t.}
\quad
& 	\boldsymbol{S} \boldsymbol{S}^{\top} = \boldsymbol{\Lambda}.
\end{aligned}
\end{equation}
The equivalence can be easily shown
by noting that given $\boldsymbol{S}$, the term $\nu
\lVert
\boldsymbol{S} \boldsymbol{1}
\rVert_2^2$ becomes a constant,
thus the relation between the estimates of $\boldsymbol{W}_i$ and $\boldsymbol{S}$
remains the same.
Then with an analogous derivation to Section \ref{subsection: Generalized Procrustes Analysis with Linear Basis Warps},
after eliminating $\boldsymbol{W}_i$,
problem (\ref{eq: Procrustes with warp models - soft constraint}) is reduced to
problem (\ref{eq: deformable GPA in S only, minimization, with soft constraint}).

It is worth mentioning that
formulation (\ref{eq: Procrustes with warp models - soft constraint}) is equivalent to formulation (\ref{eq: Procrustes with warp models - S 1 = 0, SS =L})
for any LBW that satisfies Theorem \ref{label: theorem: three equivalent conditions, p, q , cost} (for example, the affine transformation and the TPS warp).
However formulation (\ref{eq: Procrustes with warp models - soft constraint})
is more tractable in terms of solution methods, since there is no need to take care of the hard constraint $\boldsymbol{S}  \boldsymbol{1} = \boldsymbol{0}$,
which can be difficult if Theorem \ref{label: theorem: three equivalent conditions, p, q , cost} is not satisfied (\textit{e.g.,} the translation part is constrained).
In any case, formulation (\ref{eq: Procrustes with warp models - soft constraint}) can be reduced to a Brockett cost function, which can be easily solved globally.

\section{Generalized Procrustes Analysis with Partial Shapes}
\label{section: Extensions to Partial Datum Shapes}

Now we have equipped with enough insights to discuss GPA with partial shapes.
We will follow the soft regularized method discussed in Section
\ref{Reformulation using soft constraint},
which generalizes over the cases where Theorem \ref{label: theorem: three equivalent conditions, p, q , cost} is not satisfied.
As the affine transformation is a special case of the LBW, we consider GPA with the LBW only.

\subsection{Estimating the Reference Covariance Prior}
\label{section: partial shapes Eliminating Translations and Estimating Lambda}

The reference shape $\boldsymbol{S}$ is a full shape.
In order to estimate the reference covariance prior $\boldsymbol{\Lambda}$,
we recover a full shape representation for each $\boldsymbol{D}_i$
$\left(i \in \left[ 1:n \right]\right)$.
Such a process is meant to be cheap,
thus we use the classical pairwise similarity GPA to compute the similarity transformation between datum shapes, and then complete the missing points by averaging their occurrence in other shapes.

For each partial shape $\boldsymbol{D}_i$,
we compute the pairwise similarity transformation between $\boldsymbol{D}_i$ and $\boldsymbol{D}_k$ $\left(k \in \left[ 1: n\right]\right)$ by:
\begin{multline*}
\hat{s}_{ik},\, \boldsymbol{\hat{R}}_{ik},\, \boldsymbol{\hat{t}}_{ik}
= \\
\argmin_{s_{ik},\,\boldsymbol{R}_{ik},\,\boldsymbol{t}_{ik}}\  \lVert
\left(
s_{ik} \boldsymbol{R}_{ik} \boldsymbol{D}_k
+ \boldsymbol{t}_{ik} \boldsymbol{1}^{\top} - 
\boldsymbol{D}_i
\right)
\boldsymbol{\Gamma}_i \boldsymbol{\Gamma}_k
\rVert_F^2
,
\end{multline*}
where $\boldsymbol{R}_{ik}$ is an orthonormal matrix, $s_{ik}$ a non-negative scalar, and $\boldsymbol{t}_{ik}$ a $d$-dimensional translation vector.
This classical Procrustes problem can be solved in closed-form
(see Appendix \ref{appendix: Procrustes problem on sod}).

Subsequently we complete the missing points in $\boldsymbol{D}_i$
using their occurrence in other shapes by:
\begin{equation}
\label{eq: full shape completion - mathematical modle}
\boldsymbol{\mathfrak{D}}_i
=\boldsymbol{D}_i
\boldsymbol{\Gamma}_i
+
\boldsymbol{\hat{D}}_i
\boldsymbol{\Gamma}_{+}^{-1}
\left(\boldsymbol{I} - \boldsymbol{\Gamma}_i\right),
\end{equation}
to obtain a full shape $\boldsymbol{\mathfrak{D}}_i$.
Here
$\boldsymbol{\hat{D}}_i = 
\sum_{k=1}^{n} \left( \hat{s}_{ik} \boldsymbol{\hat{R}}_{ik} \boldsymbol{D}_k
+ \boldsymbol{\hat{t}}_{ik} \boldsymbol{1}^{\top} \right) \boldsymbol{\Gamma}_k$
and
$\boldsymbol{\Gamma}_{+} = \sum_{k=1}^{n} 
\boldsymbol{\Gamma}_k$.
Then we estimate the reference covariance prior $\boldsymbol{ \Lambda }$ by the result in Section \ref{subsubsection: Reference Covariance Prior} based on the zero-centered full shapes $\boldsymbol{\bar{\mathfrak{D}}}_i
=
\boldsymbol{\mathfrak{D}}_i
-
\frac{1}{m} \boldsymbol{\mathfrak{D}}_i \boldsymbol{1} \boldsymbol{1}^{\top}
$,
which is detailed in Algorithm \ref{algorithm: pseudo codo of calculating reference covariance prior}.

\begin{algorithm*}[t]
	
	\DontPrintSemicolon
	
	\SetKwFunction{getdLeftMostSingularValueFunc}{$d$-LeftmostSingularValues}
	
	\SetKwFunction{getLeftMostSingularVectorFunc}{LeftmostSingularVector}
	
	\SetKwFunction{EstimateReferenceCovariancePriorFunc}{EstimateReferenceCovariancePrior}

\SetKwProg{FEstimateReferenceCovariancePrior}{function}{($\boldsymbol{\mathfrak{D}}_1,\boldsymbol{\mathfrak{D}}_2, \dots, \boldsymbol{\mathfrak{D}}_n$)}{end}

\FEstimateReferenceCovariancePrior{\EstimateReferenceCovariancePriorFunc}{
	\For{$i \in \left[ 1:n \right]$}{		
	$\boldsymbol{\bar{\mathfrak{D}}}_i
	=
	\boldsymbol{\mathfrak{D}}_i
	-
	\frac{1}{m} \boldsymbol{\mathfrak{D}}_i \boldsymbol{1} \boldsymbol{1}^{\top}
	$
	\tcc*{Zero-centered shapes}
	
	$\sqrt{\boldsymbol{\lambda}_i} \gets$
	\getdLeftMostSingularValueFunc{$\boldsymbol{\bar{\mathfrak{D}}}_i$}
	\tcc*[f]{Elements in non-ascending order}
	}

	$\boldsymbol{\Pi}
	\gets
	\begin{bmatrix}
	\frac{\sqrt{\boldsymbol{ \lambda }_1}}{\lVert 
		\sqrt{\boldsymbol{ \lambda }_1} \rVert_2}
	&
	\frac{\sqrt{\boldsymbol{ \lambda }_2}}{\lVert 
		\sqrt{\boldsymbol{ \lambda }_2} \rVert_2}
	&
	\cdots
	&
	\frac{\sqrt{\boldsymbol{ \lambda }_n}}{\lVert 
		\sqrt{\boldsymbol{ \lambda }_n} \rVert_2}
	\end{bmatrix}
	$\;

	$\boldsymbol{\theta}^{\star} \gets$
	\getLeftMostSingularVectorFunc{$\boldsymbol{\Pi}$}\;
	
	$s \gets \frac{1}{n} \sum_{i=1}^{n}  \lVert 
	\sqrt{\boldsymbol{ \lambda }_i}
	\rVert_2$
	\tcc*{Use average scale}
	
	$\sqrt{\boldsymbol{\lambda}} \gets s \boldsymbol{\theta}^{\star}$\;
	
	$\sqrt{\boldsymbol{\Lambda}} \gets \mathrm{diag} 
	( \sqrt{\boldsymbol{\lambda}} ) $
	\tcc*{Construct diagonal matrix}
	
	$\boldsymbol{\Lambda} \gets
	\sqrt{\boldsymbol{\Lambda}} \sqrt{\boldsymbol{\Lambda}}$\;
	
	\KwRet{$\boldsymbol{\Lambda}$}
}

	\caption{Estimate Reference Covariance Prior \label{algorithm: pseudo codo of calculating reference covariance prior}}
\end{algorithm*}

\subsection{Closed-Form Solution}
\label{subsection: Closed-Form Solution to Generalized Procrustes Analysis with Partial Shapes}

We extend the soft-regularized formulation (\ref{eq: Procrustes with warp models - soft constraint}) to partial shape GPA as:
\begin{equation}
\label{eq: prtial shape Procrustes with warp models}
\mathrm{\Rmnum{3}:}
\begin{cases}
\argmin\limits_{\{\boldsymbol{W}_i\},\, \boldsymbol{S}} \quad
& \sum_{i=1}^n\,
\Vert
\boldsymbol{W}_i^{\top} \boldsymbol{\mathcal{B}}_i(\boldsymbol{D}_i) \boldsymbol{\Gamma}_{i}
- 
\boldsymbol{S} \boldsymbol{\Gamma}_{i}
\Vert_{F}^2	
\\[5pt] & +
\sum_{i=1}^n\,
\mu_i
\lVert
\boldsymbol{Z}_i \boldsymbol{W}_i
\rVert_F^2
+
\nu
\lVert \boldsymbol{S} \boldsymbol{1} \rVert_2^2
\\[10pt]
\mathrm{s.t.}
\quad
& 	\boldsymbol{S} \boldsymbol{S}^{\top} = \boldsymbol{\Lambda}
.
\end{cases}
\end{equation}
This formulation includes formulation (\ref{eq: Procrustes with warp models - soft constraint}) as a special case,
thus is the ultimate form we will implement.

Formulation (\ref{eq: prtial shape Procrustes with warp models}) can be solved by firstly eliminating the transformation parameters,
and then solving an optimization problem in $\boldsymbol{S}$ by Lemma \ref{theorem: solution to standard form with P matrix in minimization}.
The key steps are sketched as follow.
We define the shorthand
$
\boldsymbol{\mathcal{B}}_i
\defeq
\boldsymbol{\mathcal{B}}_i(\boldsymbol{D}_i)
$.
Given $\boldsymbol{S}$, problem (\ref{eq: prtial shape Procrustes with warp models}) becomes linear least-squares in $\boldsymbol{W}_i$ whose solution is:
\begin{equation}
\label{eq: the optimal w_i in partial DefGPA formulation given S}
\boldsymbol{W}_i
=
\left(
\boldsymbol{\mathcal{B}}_i
\boldsymbol{\Gamma}_i
\boldsymbol{\mathcal{B}}_i^{\top}
+
\mu_i \boldsymbol{Z}_i^{\top}  \boldsymbol{Z}_i
\right)^{-1}
\boldsymbol{\mathcal{B}}_i \boldsymbol{\Gamma}_i
{\boldsymbol{S}}^{\top}
.
\end{equation}
By substituting equation (\ref{eq: the optimal w_i in partial DefGPA formulation given S}) into problem (\ref{eq: prtial shape Procrustes with warp models}),
we obtain:
\begin{equation}
\label{eq: partial shape GPA, in S only, minimization}
\begin{aligned}
&
\argmin_{\boldsymbol{S}}\  
\mathbf{tr}\left(\boldsymbol{S}
\left(
\boldsymbol{\mathcal{P}}_{\mathrm{\Rmnum{3}}}
+
\nu \boldsymbol{1} \boldsymbol{1}^{\top}
\right)
\boldsymbol{S}^{\top}\right)
\\ & 
\mathrm{s.t.} \quad
\boldsymbol{S} \boldsymbol{S}^{\top} = \boldsymbol{\Lambda},
\end{aligned}
\end{equation}
where
$\boldsymbol{\mathcal{P}}_{\mathrm{\Rmnum{3}}}$
is defined as:
\begin{multline}
\label{eq: Matrix P used for all cases}
\boldsymbol{\mathcal{P}}_{\mathrm{\Rmnum{3}}}
= 
\sum_{i=1}^{n} 
\bigg(
\boldsymbol{\Gamma}_i - \\ 
\boldsymbol{\Gamma}_i
\boldsymbol{\mathcal{B}}_i^{\top}
\left(
\boldsymbol{\mathcal{B}}_i
\boldsymbol{\Gamma}_i
\boldsymbol{\mathcal{B}}_i^{\top}
+
\mu_i \boldsymbol{Z}_i^{\top}  \boldsymbol{Z}_i
\right)^{-1}
\boldsymbol{\mathcal{B}}_i \boldsymbol{\Gamma}_i
\bigg).
\end{multline}

Problem (\ref{eq: partial shape GPA, in S only, minimization}) can be solved in closed-form by Lemma \ref{theorem: solution to standard form with P matrix in minimization}.
We summarize the above results as:
\begin{summary}
	Formulation (\ref{eq: prtial shape Procrustes with warp models}) can be solved globally.
	The optimal reference shape $\boldsymbol{S}^{\star}$ is obtained as scaling the $d$ bottom eigenvectors of $\left(
	\boldsymbol{\mathcal{P}}_{\mathrm{\Rmnum{3}}}
	+
	\nu \boldsymbol{1} \boldsymbol{1}^{\top} \right)$ by $\sqrt{\boldsymbol{\Lambda}}$.
	The optimal transformation parameters are given by
	$\boldsymbol{W}_i^{\star}
	=
	\left(
	\boldsymbol{\mathcal{B}}_i
	\boldsymbol{\Gamma}_i
	\boldsymbol{\mathcal{B}}_i^{\top}
	+
	\mu_i \boldsymbol{Z}_i^{\top}  \boldsymbol{Z}_i
	\right)^{-1}
	\boldsymbol{\mathcal{B}}_i \boldsymbol{\Gamma}_i
	{\boldsymbol{S}^{\star}}^{\top}
	$
	$\left(i \in \left[ 1:n \right]\right)$.
\end{summary}

\subsection{Eigenvector Characterization and Tuning Parameters}
\label{subsection: partial shapes. Eigenvector Characterization and Tuning Parameters}

We rewrite $\boldsymbol{\mathcal{P}}_{\mathrm{\Rmnum{3}}}$ as
$
\boldsymbol{\mathcal{P}}_{\mathrm{\Rmnum{3}}}
=
\sum_{i=1}^{n}
\boldsymbol{P}_i
$, with:
\begin{equation*}
\boldsymbol{P}_i
=
\boldsymbol{\Gamma}_i
-
\boldsymbol{\Gamma}_i
\boldsymbol{\mathcal{B}}_i^{\top}
\left(
\boldsymbol{\mathcal{B}}_i
\boldsymbol{\Gamma}_i
\boldsymbol{\Gamma}_i
\boldsymbol{\mathcal{B}}_i^{\top}
+
\mu_i \boldsymbol{Z}_i^{\top}  \boldsymbol{Z}_i
\right)^{-1}
\boldsymbol{\mathcal{B}}_i
\boldsymbol{\Gamma}_i
.
\end{equation*}
Then $\boldsymbol{P}_i$ satisfies: $\boldsymbol{I} \succeq \boldsymbol{\Gamma}_i \succeq \boldsymbol{P}_i \succeq \boldsymbol{O}$.
This can be shown by writing $\boldsymbol{P}_i$ as
$\boldsymbol{P}_i
=
\boldsymbol{\Gamma}_i \boldsymbol{\Phi}_i \boldsymbol{\Gamma}_i$
with
$\boldsymbol{\Phi}_i = \boldsymbol{I}
-
\boldsymbol{\Gamma}_i
\boldsymbol{\mathcal{B}}_i^{\top}
\left(
\boldsymbol{\mathcal{B}}_i
\boldsymbol{\Gamma}_i
\boldsymbol{\Gamma}_i
\boldsymbol{\mathcal{B}}_i^{\top}
+
\mu_i \boldsymbol{Z}_i^{\top}  \boldsymbol{Z}_i
\right)^{-1}
\boldsymbol{\mathcal{B}}_i
\boldsymbol{\Gamma}_i$
where
$ \boldsymbol{I} \succeq \boldsymbol{\Phi}_i \succeq \boldsymbol{O}$.
As a summation, $\boldsymbol{\mathcal{P}}_{\mathrm{\Rmnum{3}}}$ satisfies:
$
n \boldsymbol{I}
\succeq
\boldsymbol{\mathcal{P}}_{\mathrm{\Rmnum{3}}}
\succeq
\boldsymbol{O}
$.

We extend Theorem \ref{label: theorem: three equivalent conditions, p, q , cost} from full shapes to partial shapes as follow.
\begin{theorem}
\label{label: theorem: three equivalent conditions, p, q , cost, in case of partial shapes}
The following statements are equivalent:
\begin{itemize}
\item[$(a)$] $\boldsymbol{\mathcal{P}}_{\mathrm{\Rmnum{3}}} \boldsymbol{1} = \boldsymbol{0}$.
\item[$(b)$] $\boldsymbol{P}_i \boldsymbol{1} = \boldsymbol{0}$.
\item[$(c)$] There exists $\boldsymbol{x}$ such that $\boldsymbol{\Gamma}_i \boldsymbol{\mathcal{B}}_i(\boldsymbol{D}_i)^{\top} \boldsymbol{x}
= 
\boldsymbol{\Gamma}_i \boldsymbol{1}$;
moreover if $\mu_i > 0$,
$\boldsymbol{x}$ must satisfy
$\boldsymbol{Z}_i \boldsymbol{x} = \boldsymbol{0}$.
\end{itemize}
\end{theorem}
\begin{proof}
The proof is analogous to that for Theorem \ref{label: theorem: three equivalent conditions, p, q , cost}.
See Appendix \ref{appendix: proofs of theorem 3 for parital shapes}.
\end{proof}

By Definition \ref{definition: LBW contains free-translations},
if the LBW contains free-translations, then
the LBW satisfies:
``there exists $\boldsymbol{x}$ such that $\boldsymbol{\mathcal{B}}_i\left(\cdot\right)^{\top} \boldsymbol{x}
= 
\boldsymbol{1}$;
moreover if $\mu_i > 0$,
$\boldsymbol{x}$ must satisfy
$\boldsymbol{Z}_i \boldsymbol{x} = \boldsymbol{0}$''.
Thus by left-multiplying $\boldsymbol{\mathcal{B}}_i(\cdot)^{\top}$ with $\boldsymbol{\Gamma}_i$, the following statement is also true:
``there exists $\boldsymbol{x}$ such that $\boldsymbol{\Gamma}_i \boldsymbol{\mathcal{B}}_i(\cdot)^{\top} \boldsymbol{x}
= 
\boldsymbol{\Gamma}_i \boldsymbol{1}$;
moreover if $\mu_i > 0$,
$\boldsymbol{x}$ must satisfy
$\boldsymbol{Z}_i \boldsymbol{x} = \boldsymbol{0}$''.
Therefore the property of free-translations in the LBW is sufficient for case $(c)$
in Theorem \ref{label: theorem: three equivalent conditions, p, q , cost, in case of partial shapes}.
\begin{proposition}
	\label{theorem: affine, and TPS satisfy the desirabed property, P1 = 0 for partial shape GPA}
	The statements in Theorem \ref{label: theorem: three equivalent conditions, p, q , cost, in case of partial shapes} are satisfied if the LBW contains free-translations.
	The affine transformation and the TPS satisfy the statements in Theorem \ref{label: theorem: three equivalent conditions, p, q , cost, in case of partial shapes}.
\end{proposition}

If the statements in Theorem \ref{label: theorem: three equivalent conditions, p, q , cost, in case of partial shapes} are satisfied, then $\boldsymbol{1}$ is an eigenvector of $\left(
\boldsymbol{\mathcal{P}}_{\mathrm{\Rmnum{3}}}
+
\nu \boldsymbol{1} \boldsymbol{1}^{\top} \right)$ with
$
\left(
\boldsymbol{\mathcal{P}}_{\mathrm{\Rmnum{3}}}
+
\nu \boldsymbol{1} \boldsymbol{1}^{\top} \right)
\boldsymbol{1} = m \nu \boldsymbol{1}
$.
Thus if we choose $\nu \ge n/m$, the eigenvector $\boldsymbol{1}$ is excluded from the solution.
\begin{proposition}
	\label{Proposition: the criterion to choose the tuning parameter of LBWs}
	The tuning parameter $\nu$ can be safely set to any $\nu \ge n/m$.
\end{proposition}

The rows of $\boldsymbol{S}^{\star}$ are (un-normalized) eigenvectors of $\left(
\boldsymbol{\mathcal{P}}_{\mathrm{\Rmnum{3}}}
+
\nu \boldsymbol{1} \boldsymbol{1}^{\top} \right)$
obtained by excluding the eigenvector $\boldsymbol{1}$.
By the orthogonality of eigenvectors, we know 
$\boldsymbol{S}^{\star} \boldsymbol{1} = \boldsymbol{0}$.
Thus we conclude:
\begin{proposition}
	If the statements in Theorem \ref{label: theorem: three equivalent conditions, p, q , cost, in case of partial shapes} are satisfied and $\nu \ge n/m$, the optimal reference shape $\boldsymbol{S}^{\star}$ is zero-centered.
\end{proposition}

Thus if the statements in Theorem \ref{label: theorem: three equivalent conditions, p, q , cost, in case of partial shapes} are satisfied, solving formulation (\ref{eq: prtial shape Procrustes with warp models}) with the soft constraint is equivalent to solving the one with the hard constraint:
\begin{equation}
\label{eq: prtial shape Procrustes with warp models, hard constraint S1 = 0}
\begin{aligned}
\argmin_{\{\boldsymbol{W}_i\},\, \boldsymbol{S}}\ 
& \sum_{i=1}^n\,
\lVert
\boldsymbol{W}_i^{\top} \boldsymbol{\mathcal{B}}_i(\boldsymbol{D}_i) \boldsymbol{\Gamma}_{i}
- 
\boldsymbol{S} \boldsymbol{\Gamma}_{i}
\rVert_{F}^2	
\\ & +
\sum_{i=1}^n\,
\mu_i
\lVert
\boldsymbol{Z}_i \boldsymbol{W}_i
\rVert_F^2
\\[5pt]
\mathrm{s.t.}
\quad
& 	\boldsymbol{S} \boldsymbol{S}^{\top} = \boldsymbol{\Lambda},
\quad
\boldsymbol{S}  \boldsymbol{1} = \boldsymbol{0}
.
\end{aligned}
\end{equation}
However, we would always recommend users to solve formulation (\ref{eq: prtial shape Procrustes with warp models}) since it is always solvable in closed-form,
and generalizes to LBWs where the statements in Theorem \ref{label: theorem: three equivalent conditions, p, q , cost, in case of partial shapes} are not satisfied.

\subsection{The Coordinate Transformation of Datum Shapes}
\label{subsection: LBWs GPA. coordinate transformation of datum shapes}

We apply (possibly distinct) rigid transformations $(\boldsymbol{R}_i,\, \boldsymbol{t}_i)$ to each datum shape $\boldsymbol{D}_i$, and denote the transformed datum shapes as $\boldsymbol{D}_i' = \boldsymbol{R}_i \boldsymbol{D}_i + \boldsymbol{t}_i$
$\left(i \in \left[ 1:n \right]\right)$.
We denote $\boldsymbol{Z}_i'$ the new regularization matrix in replacement of $\boldsymbol{Z}_i$ under transformed datum shapes.

\begin{lemma}
	In formulation (\ref{eq: prtial shape Procrustes with warp models}),
	if there exists an invertible matrix $\boldsymbol{H}_i$ such that
	$\boldsymbol{\mathcal{B}}_i\left(\boldsymbol{D}_i'\right)
	=
	\boldsymbol{H}_i
	\boldsymbol{\mathcal{B}}_i(\boldsymbol{D}_i)$
	and $
	\boldsymbol{Z}_i' = \boldsymbol{Z}_i \boldsymbol{H}_i^{\top} 
	$
	for each $i \in \left[ 1:n \right]$,
	the matrix $\boldsymbol{\mathcal{P}}_{\mathrm{\Rmnum{3}}}$ remains the same.
	So does the optimal reference shape $\boldsymbol{S}$.
\end{lemma}
\begin{proof}
	The proof is obvious as
	$\boldsymbol{\mathcal{P}}_{\mathrm{\Rmnum{3}}}$
	is invariant to such transformations.
\end{proof}

\begin{proposition}
	For the TPS warp,
	if we apply the same rigid transformation $(\boldsymbol{R}_i,\, \boldsymbol{t}_i)$ to both the datum shape $\boldsymbol{D}_i$ and the control points in $\boldsymbol{D}_i$,
	then
	$\boldsymbol{\mathcal{B}}_i(\boldsymbol{D}_i')
	=
	\boldsymbol{\mathcal{B}}_i(\boldsymbol{D}_i)$ and $\boldsymbol{Z}_i' = \boldsymbol{Z}_i$.
\end{proposition}
\begin{proof}
	See Appendix \ref{appendix. apply coordinate transformations to the TPS warp}.
\end{proof}

Together with the discussion on the affine case in Section \ref{subsection: affine GPA. coordinate transformation of datum shapes},
we have the following conclusion:
\begin{proposition}
	In formulation (\ref{eq: prtial shape Procrustes with warp models}),
	if the LBW is chosen as the affine transformation or the TSP warp,
	then the optimal reference shape $\boldsymbol{S}$ remains the same when we apply rigid coordinate transformations to the datum shapes ahead.
\end{proposition}
The above result indicates that we can parameterize the datum shapes in any coordinate frame, while the solution of formulation (\ref{eq: prtial shape Procrustes with warp models}) will give exactly the same optimal reference shape.
The optimal transformations of formulation (\ref{eq: prtial shape Procrustes with warp models}) will automatically accommodate the coordinate transformations.

\subsection{Reflection}
\label{subsection: reflection}

The reflection in the computed reference shape $\boldsymbol{S}^{\star}$ can be easily coped with by simply flipping the sign of one row in $\boldsymbol{S}^{\star}$.
Let
$\boldsymbol{S}^{\star} = 
\left[
\boldsymbol{s}_1,
\boldsymbol{s}_2,
\dots,
\boldsymbol{s}_d
\right]^{\top}$.
It is easy to verify that $\boldsymbol{S}^{\star}$ is still globally optimal if we
flip the signs of any $\boldsymbol{s}_k$ $\left( k \in \left[ 1:d \right] \right)$.
Assume there are no reflections between the datum shapes.
The reflection in $\boldsymbol{S}^{\star}$ can be detected by computing an orthogonal Procrustes between
$\boldsymbol{S}^{\star}$
and any one of $\boldsymbol{D}_i$ $\left(i \in \left[ 1:n \right]\right)$ by:
\begin{equation*}
\boldsymbol{\hat{R}}, \boldsymbol{\hat{t}}
=
\argmin_{\boldsymbol{R} \in \mathrm{O}\left(d\right),\, \boldsymbol{t}}\ 	\lVert 
\left(
\boldsymbol{R} \boldsymbol{D}_i + \boldsymbol{t} \boldsymbol{1}^{\top} - \boldsymbol{S}^{\star} 
\right) \boldsymbol{\Gamma}_{i}
\rVert_F^2
.
\end{equation*}
The optimal $\boldsymbol{t}$ is $\boldsymbol{\hat{t}} = - \frac{1}{\mathbf{nnz}(\boldsymbol{\Gamma}_{i})}
(
\boldsymbol{\hat{R}} \boldsymbol{D}_i - \boldsymbol{S}^{\star} 
) \boldsymbol{\Gamma}_{i} \boldsymbol{1}$.
Denote
$\boldsymbol{K}_i = 
\boldsymbol{\Gamma}_{i} - \frac{1}{\mathbf{nnz}(\boldsymbol{\Gamma}_{i})}
\boldsymbol{\Gamma}_{i}\boldsymbol{1}
\boldsymbol{1}^{\top}\boldsymbol{\Gamma}_{i}$.
Denote
$\boldsymbol{E} = \boldsymbol{D}_i \boldsymbol{K}_i 
\boldsymbol{K}_i^{\top}
{\boldsymbol{S}^{\star}}^{\top}$
and its SVD as
$
\boldsymbol{E} =
\boldsymbol{U} \boldsymbol{\Sigma} \boldsymbol{V}^{\top}
$.
Then the optimal $\boldsymbol{R}$ is
$\boldsymbol{\hat{R}} = \boldsymbol{V} \boldsymbol{U}^{\top}$.

If $\det (\boldsymbol{\hat{R}}) = 1$, there is no reflection.
If $\det (\boldsymbol{\hat{R}}) = -1$,
we let
$\boldsymbol{S}^{\star} = 
\left[
-\boldsymbol{s}_1,
\boldsymbol{s}_2,
\dots,
\boldsymbol{s}_d
\right]^{\top}$.
The correctness of such an approach is shown as follow.
The determinants satisfy
$\det \left(\boldsymbol{E}\right) = 
\det \left(\boldsymbol{U}\right)
\det \left( \boldsymbol{\Sigma} \right)
\det \left(\boldsymbol{V}^{\top}\right)
$
and
$\det (\boldsymbol{\hat{R}})
=
\det \left(\boldsymbol{U}\right) \det (\boldsymbol{V}^{\top})
$.
Because $\det \left( \boldsymbol{\Sigma} \right) > 0$,
we know that
$\det \left(\boldsymbol{E}\right)$ and $\det (\boldsymbol{\hat{R}})$ have the same sign.
By flipping the sign of one row in $\boldsymbol{S}^{\star}$,
we flip the sign of one column in $\boldsymbol{E}$,
which causes the flip of the sign of $\det \left(\boldsymbol{E}\right)$,
thus $\det (\boldsymbol{\hat{R}})$ as well.

\subsection{Pseudo-Code}
\label{subsection: Pseudo code}

Our algorithm is rather easy to implement.
We term the proposed deformable GPA framework as DefGPA.
The overall procedure is given as pseudo-code in Algorithm \ref{algorithm: pseudo codo of the deformable GPA}.
We release our Matlab implementation of DefGPA to foster future research in this direction.

\noindent\textbf{Code repository:}
\url{https://bitbucket.org/clermontferrand/deformableprocrustes/src/master/}

\begin{algorithm*}[th]
\DontPrintSemicolon

\SetKwBlock{KwInit}{Initialization}{end}

\SetKwFunction{chooseSmoothParamFun}{ChooseSmoothingParameter}

\SetKwFunction{initializeWarpModelFun}{InitializeWarpModel}

\SetKwFunction{PairwiseEuclideanProcrustesFun}{PairwiseSimilarityProcrustes}

\SetKwFunction{getdBottomEigenvectorsFun}{$d$-BottomEigenvectors}

\SetKwFunction{EstimateReferenceCovariancePriorFunc}{EstimateReferenceCovariancePrior}

\BlankLine	

	\KwData{$\{\boldsymbol{D}_i \in \mathbb{R}^{d \times m}, \boldsymbol{\Gamma}_i \in \mathbb{R}^{m \times m}\}$, $i \in \left[1:n\right]$
	\tcc*{Datum shapes and visibility indicators}}

	\KwOut{$\boldsymbol{S}^{\star}$ and $\{\boldsymbol{W}_i^{\star}\}$, $i \in \left[ 1:n \right]$
	\tcc*{Reference shape and transformation parameters}}

\BlankLine
	
\Begin{
	\For(\tcc*[f]{Users' choice}){$i \in \left[ 1:n \right]$}{
		$\boldsymbol{\mathcal{B}}_i(\cdot),\, \boldsymbol{Z}_i \gets $
		\initializeWarpModelFun
		\tcc*{Choose and initialize LBWs}
		
		$\mu_i \gets$ \chooseSmoothParamFun
		\tcc*{Choose LBWs' smoothing parameters}
		
		$\boldsymbol{\mathcal{B}}_i
		\gets
		\boldsymbol{\mathcal{B}}_i(\boldsymbol{D}_i)$  \tcc*[f]{Lift to feature space}
	}

	$\nu \gets m/n$
	\tcc*{By Proposition \ref{Proposition: the criterion to choose the tuning parameter of LBWs}}
	
	\For(\tcc*[f]{Shape completion from pairwise Procrustes}){$i \in \left[ 1:n \right]$}{
		\eIf(\tcc*[f]{A full shape}){$\mathbf{nnz}(\boldsymbol{\Gamma}_i) = n$}{$\boldsymbol{\mathfrak{D}}_i \gets \boldsymbol{D}_i$}
		(\tcc*[f]{A partial shape})
		{
		\For{$k \in \left[ 1:n \right]$}{
		$\hat{s}_{ik},\, \boldsymbol{\hat{R}}_{ik},\, \boldsymbol{\hat{t}}_{ik} \gets$
		\PairwiseEuclideanProcrustesFun{$\boldsymbol{D}_i$, $\boldsymbol{D}_k$}\;
		}
		$\boldsymbol{\mathfrak{D}}_i \gets$
		Equation (\ref{eq: full shape completion - mathematical modle})
		\tcc*[f]{Predict missing points and complete shape}
		}	
	}
	
	$\boldsymbol{\Lambda} \gets$
	\EstimateReferenceCovariancePriorFunc{$\boldsymbol{\mathfrak{D}}_1,\boldsymbol{\mathfrak{D}}_2, \dots, \boldsymbol{\mathfrak{D}}_n$}
	\tcc*[f]{Algorithm \ref{algorithm: pseudo codo of calculating reference covariance prior}}
	
	$\boldsymbol{\mathcal{P}} \gets$
	Equation (\ref{eq: Matrix P used for all cases})\;

	$\boldsymbol{X}^{\star} \gets$
	\getdBottomEigenvectorsFun{$\boldsymbol{\mathcal{P}} + \nu \boldsymbol{1} \boldsymbol{1}^{\top}$}
	\tcc*[f]{Solve for reference shape}

	$\boldsymbol{S}^{\star} \gets \sqrt{\boldsymbol{\Lambda}} {\boldsymbol{X}^{\star}}^{\top}$\;

	\For(\tcc*[f]{Solve for individual transformations}){$i \in \left[ 1:n \right]$}{
	$\boldsymbol{W}_i^{\star}
	\gets
	\left(
	\boldsymbol{\mathcal{B}}_i
	\boldsymbol{\Gamma}_i
	\boldsymbol{\mathcal{B}}_i^{\top}
	+
	\mu_i \boldsymbol{Z}_i^{\top}  \boldsymbol{Z}_i
	\right)^{-1}
	\boldsymbol{\mathcal{B}}_i \boldsymbol{\Gamma}_i
	{\boldsymbol{S}^{\star}}^{\top}
	$\;
}		
}
	\caption{DefGPA -- GPA with LBWs \label{algorithm: pseudo codo of the deformable GPA}}
\end{algorithm*}

\section{Experimental Results}
\label{section: experimental results}

We provide experimental results with respect to various deformable scenarios.
While our method adapts to general LBWs, we use the affine transformation and the TPS warp to show the results.
A brief introduction of the TPS warp is provided in Appendix \ref{Appendix: Thin-Plate Spline}.

\subsection{Experimental Setups}

The datasets used for evaluation are listed in Table \ref{table: List of datasets used for evaluation},
while samples are given in Figure~\ref{fig. samples of each dataset, with corresponding landmarks}.
In Table \ref{table: List of datasets used for evaluation},
``F'' standards for full shapes (without missing datum points), and ``P'' for partial shapes (with missing datum points).
These are public datasets coming from different papers and designed for different problems.
These datasets cover the case of structural deformations like facial expressions \citep{bartoli2013stratified}, deformable objects \citep{gallardo2017dense, bartoli2006towards}, and tissue deformations \citep{bilic2019liver}.
Both 2D and 3D cases are considered.
The Liver dataset is a mesh with 4004 corresponding vertices, while the others are 2D/3D images with corresponding landmarks.

\begin{table*}[t]
	\centering
	\begin{tabular*}{0.99\textwidth}{@{} @{\extracolsep{\fill}} lcccc r @{}}
		\toprule\midrule
		Dataset & Dim. & F/P & Landmarks & Shapes & Description \\
		\midrule
		Face & 2D & P & 43 - 68 & 10 & different facial expressions and camera perspectives \\
		Bag & 2D & F & 155 & 8 & deforming handbag \\
		Pillow & 2D & F & 69 & 10 & deforming pillow cover  \\
		LiTS & 3D & F & 54 & 8 & CT scans with fiducial landmarks \\
		Liver & 3D & F & 4004 & 10 & simulated deformations of a human 3D liver model \\
		ToyRug & 3D & F & 30 & 200 & 2D features and 3D points from a stereo rig \\
		\midrule\bottomrule
	\end{tabular*}
	\caption{List of datasets used for evaluation.}
	\label{table: List of datasets used for evaluation}
\end{table*}

\begin{figure*}[t]
	\centering
	\begin{subfigure}{0.18\textwidth}
		\includegraphics[width=1\textwidth]{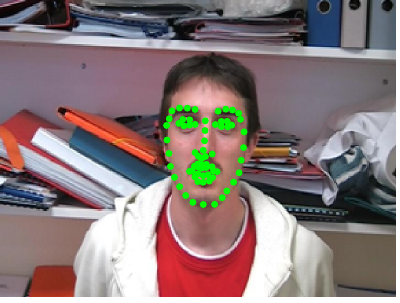}\caption{Face}
	\end{subfigure}
	\begin{subfigure}{0.18\textwidth}
		\includegraphics[width=1\textwidth]{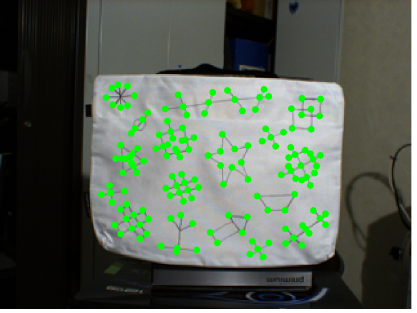}\caption{Bag}
	\end{subfigure}
	\begin{subfigure}{0.18\textwidth}
		\includegraphics[width=1\textwidth]{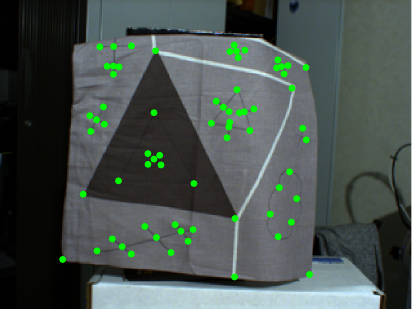}\caption{Pillow}
	\end{subfigure}
	\begin{subfigure}{0.19\textwidth}
		\includegraphics[width=1\textwidth]{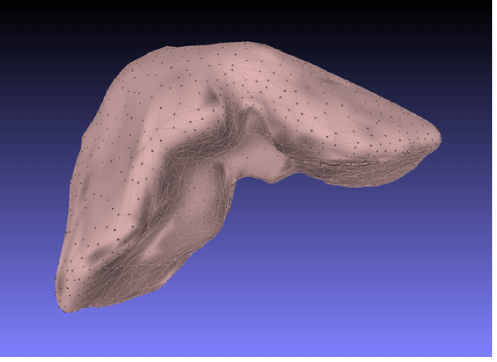}\caption{Liver}
	\end{subfigure}
	\\
	\begin{subfigure}{0.225\textwidth}
		\includegraphics[width=1\textwidth]{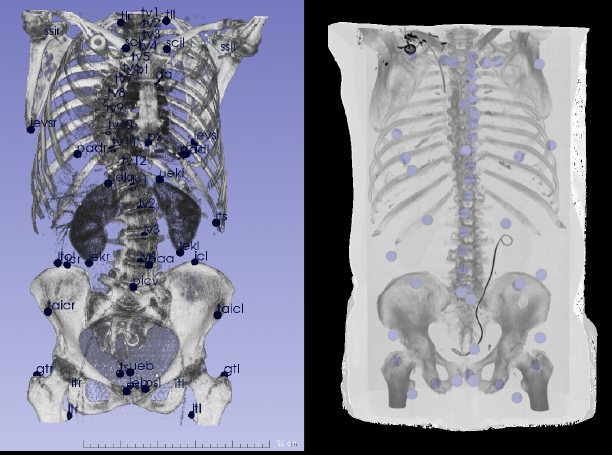}\caption{LiTS}
	\end{subfigure}
	\quad
	\begin{subfigure}{0.40\textwidth}
		\includegraphics[width=1\textwidth]{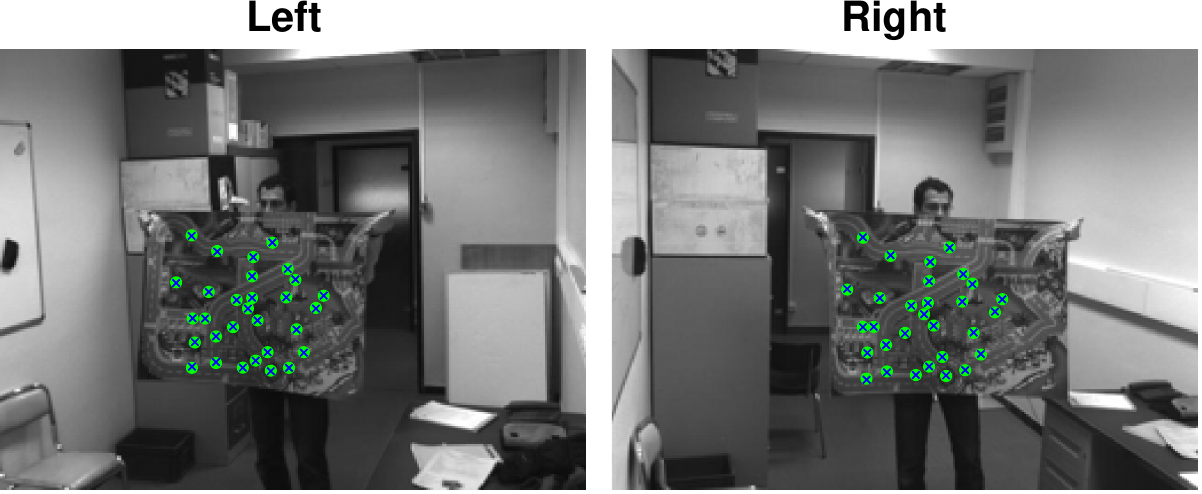}\caption{ToyRug}
	\end{subfigure}
	\caption{Sample image of each dataset.}
	\label{fig. samples of each dataset, with corresponding landmarks}
\end{figure*}

We use Euclidean GPA and Affine GPA with costs defined in the datum-space as benchmark algorithms.
These two methods are denoted as $\ast$EUC\_d and $\ast$AFF\_d,
where \_d means the methods minimize the datum-space cost.
The notion $\ast$ indicates they are benchmark methods.
For $\ast$EUC\_d and $\ast$AFF\_d,
we use the MATLAB implementation by \citep{bartoli2013stratified} based on closed-form initialization and Levenberg–Marquardt refinement.

Our methods are denoted by AFF\_r,
TPS\_r($3$), TPS\_r($5$), and TPS\_r($7$) respectively, where
\_r means the methods minimize the reference-space cost.
AFF\_r stands for GPA with the affine model, and TPS\_r($\cdot$) stands for GPA with the TPS warp.
We choose the control points of the TPS warp evenly along each principal axis of the datum shape.
We examine the cases of $3$, $5$ and $7$ points along each principal axis,
which results in $9$, $25$, $49$ overall control points for 2D datasets
and $27$, $125$, $343$ overall control points for the LiTS and Liver datasets.
The ToyRug dataset is almost flat, thus we assign two layers of control points along the first two principal axes which yields $18$, $50$, $98$ overall control points.

Our methods are implemented in MATLAB,
which constitutes a fair comparison against $\ast$EUC\_d and $\ast$AFF\_d.
The experiments are carried out by an
{Intel(R) Core(TM) i7-6700K CPU @ 4.00GHz $\times$ 8 CPU},
running Ubuntu 18.04.5 LTS.
The MATLAB version is R2020b.

\subsection{Evaluation Metrics}
\label{subsection: evaluation metrics used in the experiments}

\subsubsection{Landmark Residual}

We use RMSE\_r to denote the landmark residual defined in the reference-space,
and RMSE\_d the landmark residual in the datum-space.
These two metrics are defined as:
\begin{equation*}
\mathrm{RMSE\_r} = 
\sqrt{\frac{1}{\kappa} \sum_{i=1}^n\,
	\left\|
	\left(
	\mathcal{T}_i\left(\boldsymbol{D}_i\right)  - \boldsymbol{S}^{\star}
	\right)
	\boldsymbol{\Gamma}_{i}
	\right\|_{F}^2}
\end{equation*}
\begin{equation*}
\mathrm{RMSE\_d} = 
\sqrt{\frac{1}{\kappa} \sum_{i=1}^n\,
	\left\|
	\left(
	\boldsymbol{D}_i - \mathcal{T}_i^{-1}\left(\boldsymbol{S}^{\star}\right)
	\right)
	\boldsymbol{\Gamma}_{i}
	\right\|_{F}^2},
\end{equation*}
where $\kappa = \sum_{i=1}^{n} \mathbf{nnz} \left( \boldsymbol{\Gamma}_i \right)$.
If the transformation model is invertible,
it is easy to derive RMSE\_r and RMSE\_d from one another.
However, this is typically not the case for LBWs,
\textit{e.g.}~the TPS warp is not invertible.
We propose to use control points and their images as samples to fit an inverse TPS warp.
In specific,
let $\mathcal{T}\left(\cdot\right)$ be a TPS warp,
and $\boldsymbol{c}_i$ $\left( i \in \left[ 1:l \right] \right)$ be its $l$ control points.
Let the images of these $l$ control points under $\mathcal{T}\left(\cdot\right)$ be
$\boldsymbol{c}'_i = \mathcal{T} \left(\boldsymbol{c}_i\right)$
$\left( i \in \left[ 1:l \right] \right)$.
Then in essence $\mathcal{T}\left(\cdot\right)$ is a regression model obtained by fitting the datum pairs
$\left(\boldsymbol{c}_i, \boldsymbol{c}'_i\right)$
$\left( i \in \left[ 1:l \right] \right)$
with $\boldsymbol{c}_i$ being the input and $\boldsymbol{c}'_i$ being the output.
Therefore the inverse of $\mathcal{T}\left(\cdot\right)$, denoted by $\mathcal{T}^{-1}\left(\cdot\right)$, can be defined by fitting the pairs
$\left(\boldsymbol{c}'_i, \boldsymbol{c}_i\right)$
$\left( i \in \left[ 1:l \right] \right)$,
with $\boldsymbol{c}'_i$ being the input and $\boldsymbol{c}_i$ being the output.
For the TPS warp, this is realized by letting $\boldsymbol{c}'_i$ be the set of control points of $\mathcal{T}^{-1}\left(\cdot\right)$,
and computing the warp parameters using the relation $\mathcal{T}^{-1}\left(\boldsymbol{c}'_i\right) = \boldsymbol{c}_i$.

\subsubsection{Cross-Validation Error}

The warp models can overfit the data by using a small enough TPS smoothing parameter.
We quantify this behavior by the Cross-Validation Error (CVE) defined as:
\begin{equation*}
\mathrm{CVE} = 
\sqrt{\frac{1}{\kappa} \sum_{i=1}^n\,
	\left\|
	\left(
	\boldsymbol{\hat{S}}_{i} - \boldsymbol{S}^{\star}
	\right)
	\boldsymbol{\Gamma}_{i}
	\right\|_{F}^2}
,
\end{equation*}
where the \textit{predicted reference shape} $\boldsymbol{\hat{S}}_{i}$ is computed by the $G$-fold cross-validation as follow.
We group all $m$ points indexed as $1,2,\dots,m$ as $G$ mutually exclusive subsets as:
\begin{equation*}
\underbracket{\boldsymbol{p}_1,\cdots, \boldsymbol{p}_N}_{g_1},\,
\underbracket{\boldsymbol{p}_{N+1},\cdots,
	\boldsymbol{p}_{2N}}_{g_2}
\,\cdots\,
\underbracket{\boldsymbol{p}_{(G-1)N + 1}, \cdots, \boldsymbol{p}_{m}}_{g_G}.
\end{equation*}
Each subset has $N$ points except the subset $g_G$ which contains the remaining left.
We index the points $g_k$ in shape $\boldsymbol{D}_i$ as $\boldsymbol{D}_i \left( g_k \right)$,
and the points in shape $\boldsymbol{\hat{S}}_{i}$ as $\boldsymbol{\hat{S}}_{i} \left( g_k \right)$.
Now for each subset $g_k$ $\left(k = \left[ 1:G \right]\right)$,
we solve the GPA with all the points except those in $g_k$.
This solution is denoted as $\boldsymbol{S}_{g_k}^{\star}$ for the reference shape and $\mathcal{T}_{g_k : i}$ $\left(i \in \left[ 1: n\right]\right)$ for the transformations.
Then for each datum shape, we predict the positions of points in $\boldsymbol{D}_i \left( g_k \right)$ by the transformation $\mathcal{T}_{g_k : i} \left( \boldsymbol{D}_i \left( g_k \right) \right)$.
We correct the gauge of $\mathcal{T}_{g_k : i} \left( \boldsymbol{D}_i \left( g_k \right) \right)$ with the Euclidean Procrustes between $\boldsymbol{S}_{g_k}^{\star}$ and $\boldsymbol{S}^{\star}$ (by eliminating points $g_k$ in $\boldsymbol{S}^{\star}$). Denote such a solution to be $\boldsymbol{R}_{g_k}$ and $\boldsymbol{t}_{g_k}$.
Finally we set
$\boldsymbol{\hat{S}}_{i} \left( g_k \right)
=
\boldsymbol{R}_{g_k}
\mathcal{T}_{g_k : i} \left( \boldsymbol{D}_i \left( g_k \right) \right)
+
\boldsymbol{t}_{g_k} \boldsymbol{1}^{\top}
$.
Repeating the above procedure for all $G$ subsets, we obtain the predicted reference shape $\boldsymbol{\hat{S}}_{i}$.

The CVE resembles the RMSE, except that it handles overfitting.
For the Face, Bag, Pillow, LiTS and ToyRug dataset,
we set $N=1$.
For the Liver dataset, we set $N=40$ to cope with the larger dataset size and dimension.

\subsection{Thin-Plate Spline Smoothing Parameters}

\begin{figure*}[t]
	\centering
	\begin{subfigure}{.47\textwidth}
		\centering
		\includegraphics[width=1\textwidth]{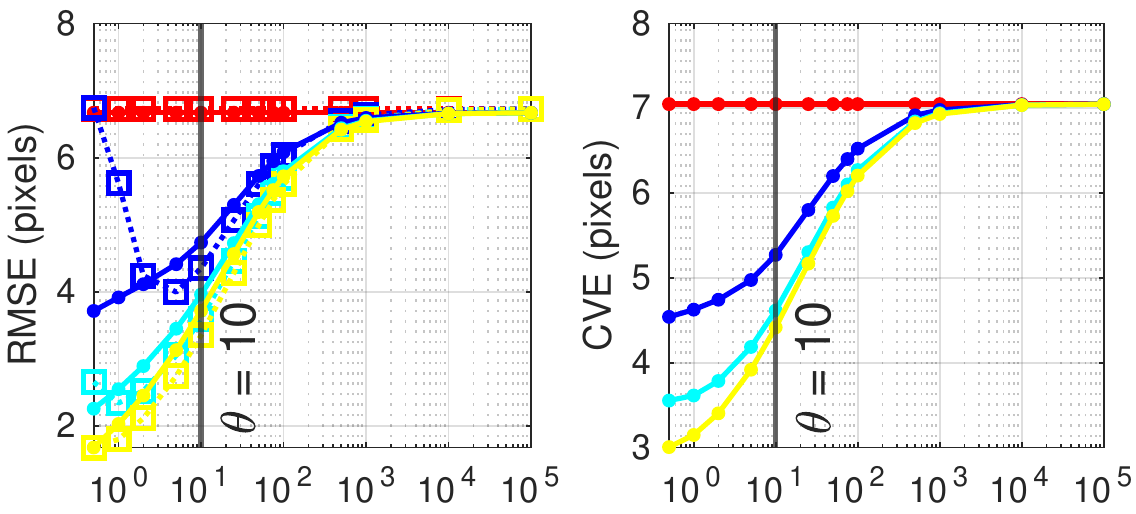}
		\caption{Face}
	\end{subfigure}
	\hfil
	\begin{subfigure}{.47\textwidth}
		\centering
		\includegraphics[width=1\textwidth]{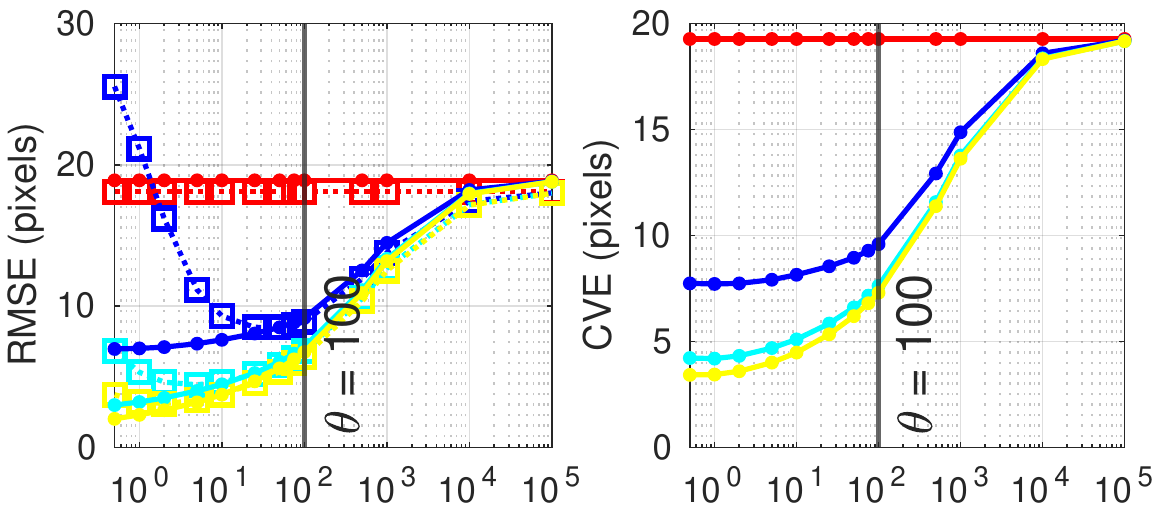}
		\caption{Bag}
	\end{subfigure}
	\begin{subfigure}{.47\textwidth}
		\centering
		\includegraphics[width=1\textwidth]{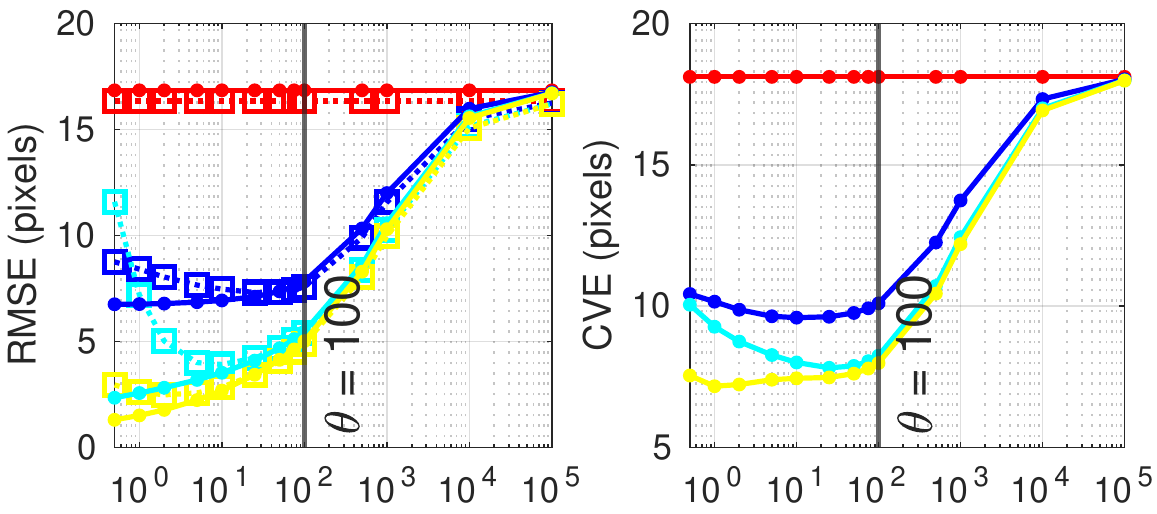}
		\caption{Pillow}
	\end{subfigure}
	\hfil
	\begin{subfigure}{.47\textwidth}
		\centering
		\includegraphics[width=1\textwidth]{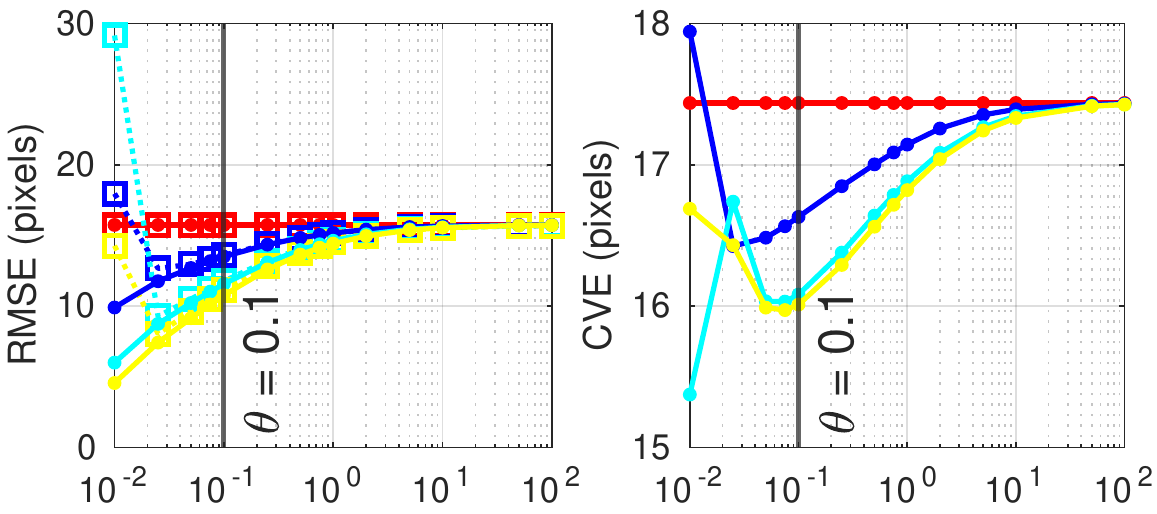}
		\caption{LiTS}
	\end{subfigure}
	\begin{subfigure}{.47\textwidth}
		\centering
		\includegraphics[width=1\textwidth]{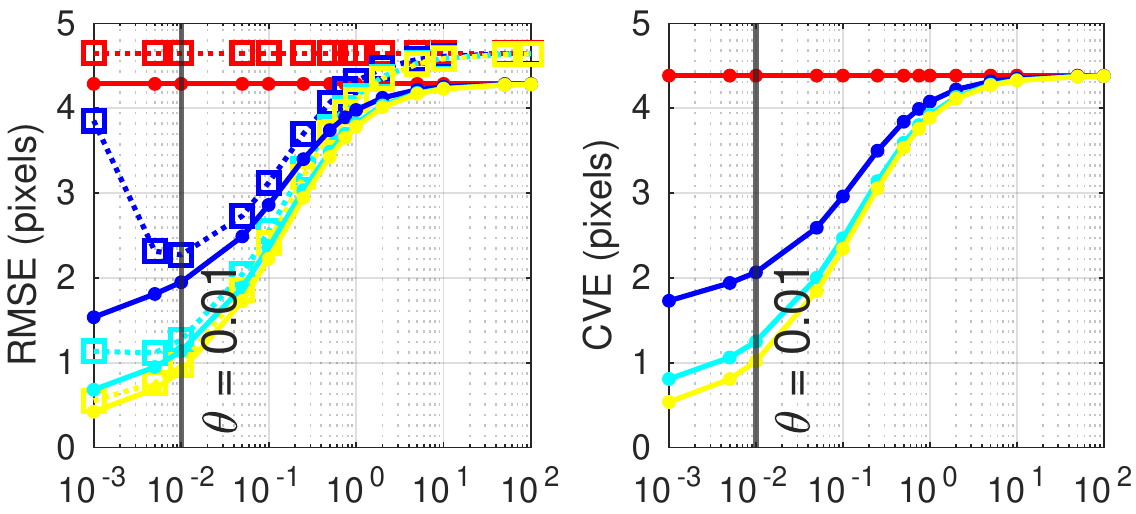}
		\caption{Liver}	
	\end{subfigure}
	\hfil
	\begin{subfigure}{.47\textwidth}
		\centering
		\includegraphics[width=1\textwidth]{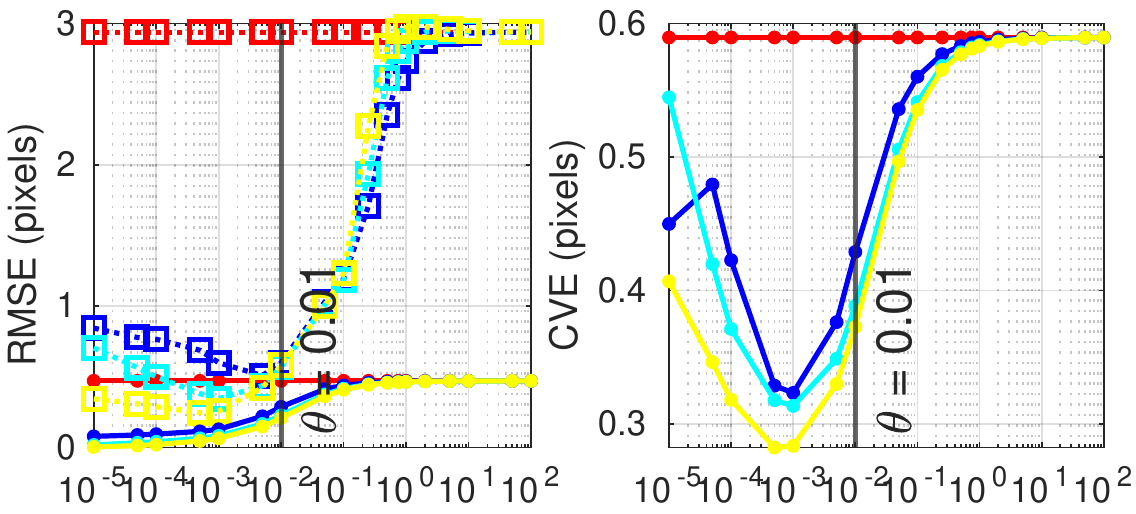}
		\caption{ToyRug}	
	\end{subfigure}
	\begin{subfigure}{.95\textwidth}
		\centering
		\includegraphics[width=1\textwidth]{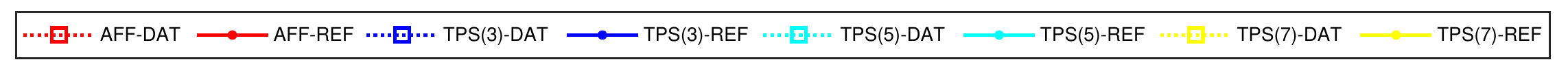}
	\end{subfigure}
	\caption{The choice of TPS smoothing parameters. The Root-Mean-Squared-Error (RMSE), Rigidity-Score (RS), and Cross-Validation-Error (CVE) are plotted with respect to different $\theta$ values (in horizontal axes).}
	\label{fig. The RMSE and RS metric with respect to different TPS smoothing parameters.}
\end{figure*}

We set the TPS smoothing parameter
$\mu_i$ as $\mu_i = \mathbf{nnz} \left( \boldsymbol{\Gamma}_i\right) \theta$,
$\left(i \in \left[ 1:n \right]\right)$,
and adjust $\theta$ within the range
$\left[1e-5, 1e+5\right]$.
We report the RMSE\_r, RMSE\_d, and CVE with respect to different $\theta$ in Figure~\ref{fig. The RMSE and RS metric with respect to different TPS smoothing parameters.}.
The $\theta$ used to generate our results are marked as vertical lines, annotated with the chosen value.

In Figure~\ref{fig. The RMSE and RS metric with respect to different TPS smoothing parameters.}, the RMSE\_r decreases monotonically as $\theta$ decreases, as a smaller $\theta$ implies more flexibility.
However, a sufficiently small $\theta$ may cause overfitting. The overfitting in DefGPA is mainly twofold: the bias to the data, and the bias to the optimization criterion (\textit{i.e.,} the cost function). When overfitting happens, the optimization tends to fit the training data to the cost function as much as possible, thus actually creating misfitting to new data and other criteria.
This suggests that we can use the fitness to new data and other criteria to detect overfitting.
The idea of using new data is realized by cross-validation given as the CVE metric,
and that of using other criteria by the datum-space cost (denoted by RMSE\_d).
In Figure~\ref{fig. The RMSE and RS metric with respect to different TPS smoothing parameters.},
the overfitting is reflected as the divergence of the RMSE\_r and RMSE\_d metric,
as well as the slope change of the CVE curve.
In general, CVE has a positive slope if RMSE\_r and RMSE\_d agree with one another, and a negative slope if RMSE\_r and RMSE\_d diverge.

For 2D datasets, we choose
$\theta = 10$ for the Face, and
$\theta = 100$ for the Bag and Pillow datasets.
For 3D datasets, we choose $\theta = 0.1$ for the LiTS,
and $\theta = 0.01$ for the Liver and ToyRug datasets.
Although the bending energy term has different constants for 2D and 3D cases, the choice of $\theta$ is roughly stable within each category.

In terms of overfitting, we focus on the agreement of the trend of the RMSE\_r and RMSE\_d curves: whether they both increase or decrease.
There is always a gap between the RMSE\_r and RMSE\_d metrics, which stands for the asymmetry between different cost functions.
We will examine this asymmetry in Section \ref{exp_section: asymmetry in cost functions}.

\subsection{Results Based on Landmarks}

\subsubsection{Statistics}

We report in Table \ref{The residual error and computational time of each method on benchmark datasets} the RMSE\_d, RMSE\_r, CVE and the computational time per GPA problem.
The RMSE\_d, RMSE\_r and CVE are evaluated in pixels, and the time in seconds.
\begin{table*}[h]
	\centering
	\begin{tabular*}{0.99\textwidth}{@{} @{\extracolsep{\fill}} ll cccccc @{}}
		\toprule\midrule
		Dataset &   &  $\ast$EUC\_d & $\ast$AFF\_d & AFF\_r & TPS\_r($3$) & TPS\_r($5$) & TPS\_r($7$)  \\ 
		\midrule 
		Face  & RMSE\_d & 7.11 & 6.52 & 6.73 & 4.35 & 3.60 & 3.35 \\ 
		& RMSE\_r & 7.11 & 6.64 & 6.67 & 4.74 & 3.96 & 3.72 \\ 
		& CVE & 7.26 & 7.07 & 7.05 & 5.27 & 4.62 & 4.42 \\ 
		& Time & 0.6326 & 0.2113 & 0.0731 & 0.0759 & 0.0149 & 0.0191 \\ 
		Bag  & RMSE\_d & 30.23 & 17.73 & 18.08 & 8.91 & 6.80 & 6.42 \\ 
		& RMSE\_r & 30.23 & 19.78 & 18.90 & 9.16 & 7.12 & 6.75 \\ 
		& CVE & 30.48 & 20.19 & 19.27 & 9.59 & 7.62 & 7.30 \\ 
		& Time & 0.3785 & 0.0535 & 0.0328 & 0.0303 & 0.0218 & 0.0275 \\ 
		Pillow  & RMSE\_d & 23.44 & 16.23 & 16.34 & 7.58 & 5.39 & 4.89 \\ 
		& RMSE\_r & 23.44 & 17.96 & 16.84 & 7.82 & 5.47 & 5.02 \\ 
		& CVE & 24.01 & 19.26 & 18.11 & 10.10 & 8.25 & 7.99 \\ 
		& Time & 0.1862 & 0.0339 & 0.0253 & 0.0211 & 0.0140 & 0.0176 \\ 
		LiTS  & RMSE\_d & 19.85 & 15.43 & 15.71 & 13.60 & 11.76 & 11.13 \\ 
		& RMSE\_r & 19.85 & 15.84 & 15.74 & 13.50 & 11.57 & 10.77 \\ 
		& CVE & 20.50 & 17.19 & 17.44 & 16.63 & 16.08 & 16.01 \\ 
		& Time & 0.2144 & 0.0267 & 0.0245 & 0.0234 & 0.0387 & 0.1284 \\ 
		Liver  & RMSE\_d & 4.94 & 4.60 & 4.65 & 2.27 & 1.27 & 0.97 \\ 
		& RMSE\_r & 4.94 & 4.68 & 4.29 & 1.95 & 1.14 & 0.90 \\ 
		& CVE & 5.01 & 4.78 & 4.38 & 2.06 & 1.25 & 1.02 \\ 
		& Time & 249.3995 & 41.8242 & 9.7071 & 10.1804 & 11.0086 & 13.3759 \\ 
		ToyRug  & RMSE\_d & 0.88 & 0.42 & 2.94 & 0.60 & 0.57 & 0.58 \\ 
		& RMSE\_r & 0.88 & 2.41 & 0.47 & 0.29 & 0.23 & 0.21 \\ 
		& CVE & 0.98 & 12.01 & 0.59 & 0.43 & 0.39 & 0.37 \\ 
		& Time & 2.2158 & 0.2969 & 0.0792 & 0.1393 & 0.2122 & 0.4351 \\ 
		\midrule\bottomrule 
	\end{tabular*}
	\caption{The statistics of landmark registration.}
	\label{The residual error and computational time of each method on benchmark datasets}	
\end{table*}

\subsubsection{Visualization}

For each case, we visualize the optimal reference shape $\boldsymbol{S}^{\star}$,
along with the set of predicted reference shapes $\boldsymbol{\hat{S}}_{i}$ computed by the leave-$N$-out cross-validation.
The details of how to compute each $\boldsymbol{\hat{S}}_{i}$
have been given in Section \ref{subsection: evaluation metrics used in the experiments}.
Recall that $\boldsymbol{\hat{S}}_{i}$ is sensitive to overfitting, thus is better to benchmark fitness than using the transformed datum shapes $\mathcal{T}_i(\boldsymbol{D}_i)$ directly.

For the 2D datasets, we visualize $\boldsymbol{\hat{S}}_{i}$ and $\boldsymbol{S}^{\star}$ directly in Figure~\ref{fig. landmark visualization of 2D datasets. The fitness of each registration method.}.
The predicted reference shapes $\boldsymbol{\hat{S}}_{i}$ are plotted in blue,
and the optimal reference shape $\boldsymbol{S}^{\star}$ on top in green.
For the 3D datasets, we visualize the points in $\boldsymbol{S}^{\star}$, and encode the CVE with the marker size and color in Fig. \ref{fig. landmark visualization of 3D cases. The fitness of each registration method.}.
The bigger the marker and the redder the color, the larger the associated CVE value.

The result in
Figure~\ref{fig. landmark visualization of 2D datasets. The fitness of each registration method.}
shows that the landmark fitness is consistently improved by using the affine-GPA ($\ast$AFF\_d and AFF\_r) or the TPS warp based GPA compared with the Euclidean-GPA ($\ast$EUC\_d).
Both $\ast$AFF\_d and AFF\_r give similar results.
There is a clear improvement of TPS\_r($3$) over the affine-GPA.
However, the improvement of TPS\_r($7$) over TPS\_r($5$) is marginal.

The result in
Fig. \ref{fig. landmark visualization of 3D cases. The fitness of each registration method.}
agrees with the result in Figure~\ref{fig. landmark visualization of 2D datasets. The fitness of each registration method.},
which improves sequentially by using the $\ast$EUC\_d, AFF\_r, TPS\_r($3$), TPS\_r($5$) and TPS\_r($7$) methods.
Both the datum-space method $\ast$AFF\_d and the reference-space method AFF\_r produce similar results on the LiTS and Liver datasets.
The $\ast$AFF\_d method gives rather poor fitness on the ToyRug dataset.
This is due to the asymmetry in the cost function which be examined in detail in Section \ref{exp_section: asymmetry in cost functions}.

\begin{figure*}[t]
	\centering
	\begin{tabular*}{0.99\textwidth}{@{} @{\extracolsep{\fill}} cccccc @{}}
		\toprule\midrule
		\textbf{\scriptsize \textsf{$\ast$EUC\_d}} & \textbf{\scriptsize \textsf{$\ast$AFF\_d}} & \textbf{\scriptsize \textsf{AFF\_r}} & \textbf{\scriptsize \textsf{TPS\_r($3$)}} & \textbf{\scriptsize \textsf{TPS\_r($5$)}} & \textbf{\scriptsize \textsf{TPS\_r($7$)}}  \\
		\midrule
		\includegraphics[width=0.14\textwidth]{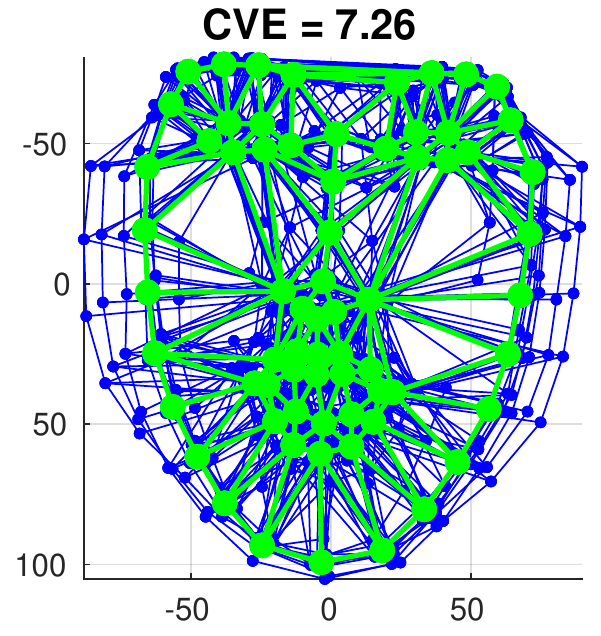} &
		\includegraphics[width=0.148\textwidth]{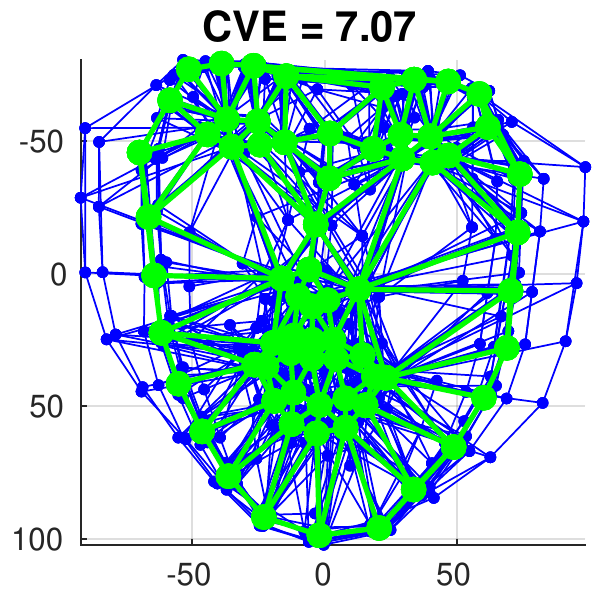} &
		\includegraphics[width=0.14\textwidth]{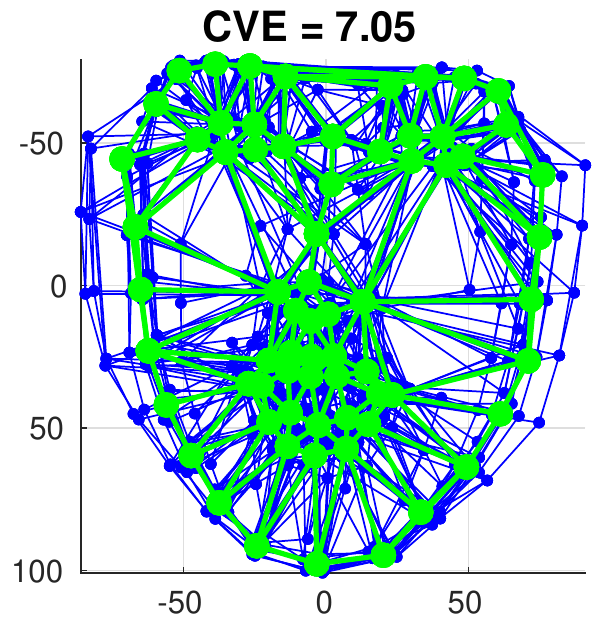} &
		\includegraphics[width=0.134\textwidth]{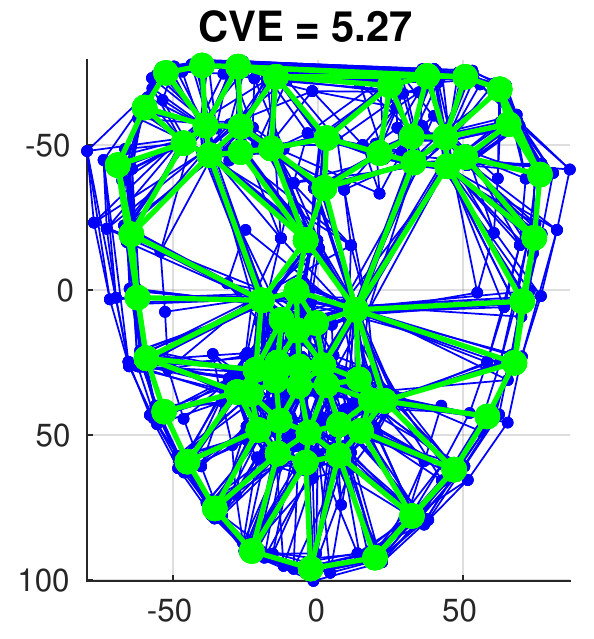} & 
		\includegraphics[width=0.134\textwidth]{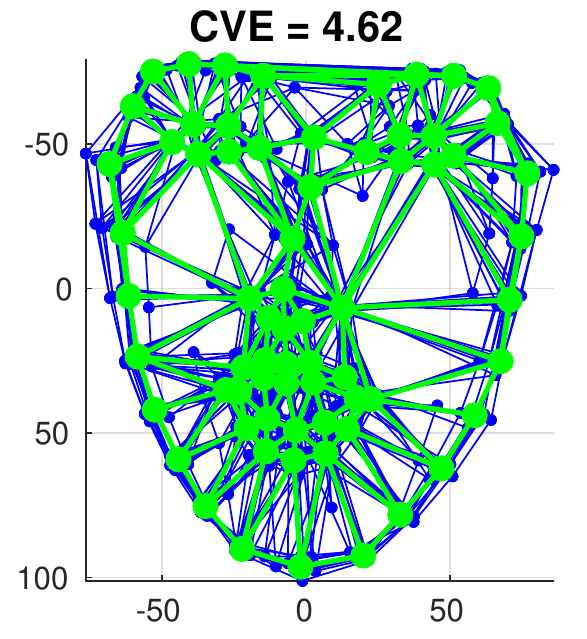} &
		\includegraphics[width=0.134\textwidth]{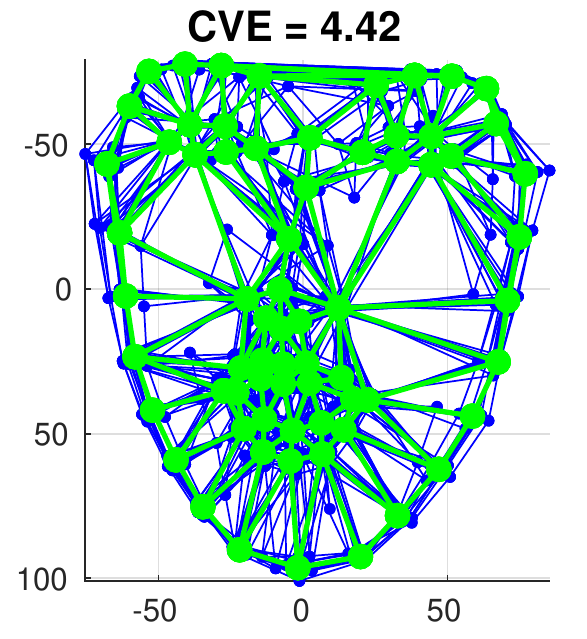} \\
		\includegraphics[width=0.14\textwidth]{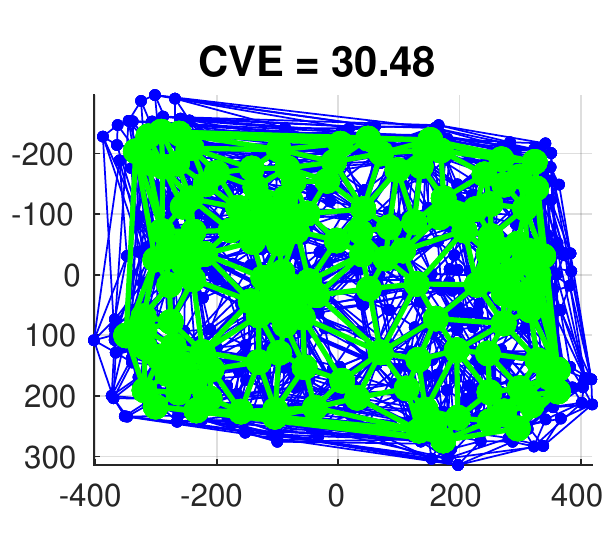}  &
		\includegraphics[width=0.134\textwidth]{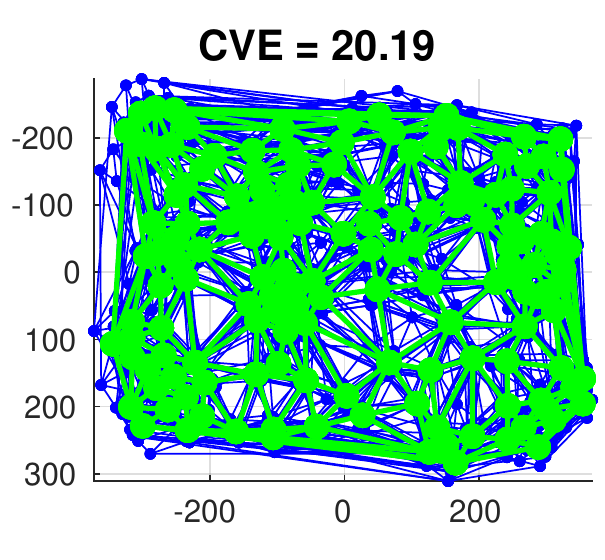} &
		\includegraphics[width=0.14\textwidth]{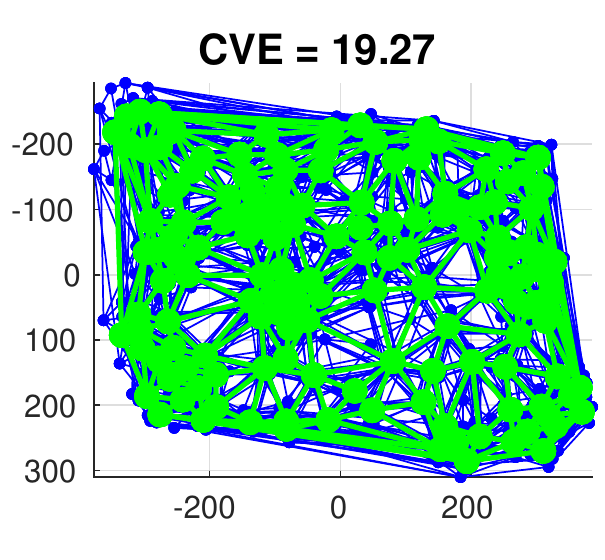}  &
		\includegraphics[width=0.14\textwidth]{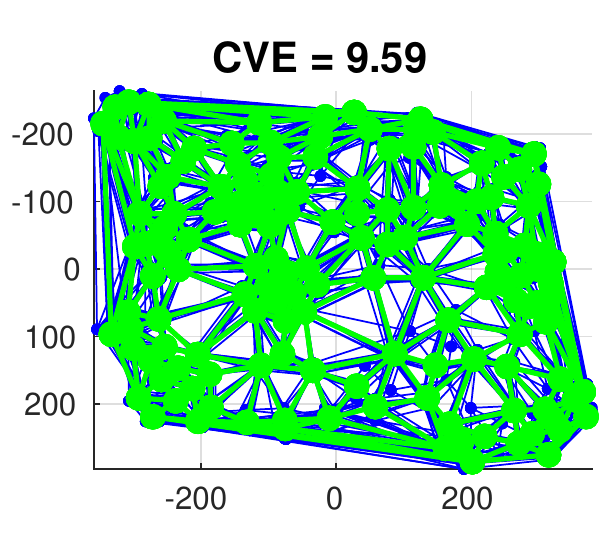} & 
		\includegraphics[width=0.14\textwidth]{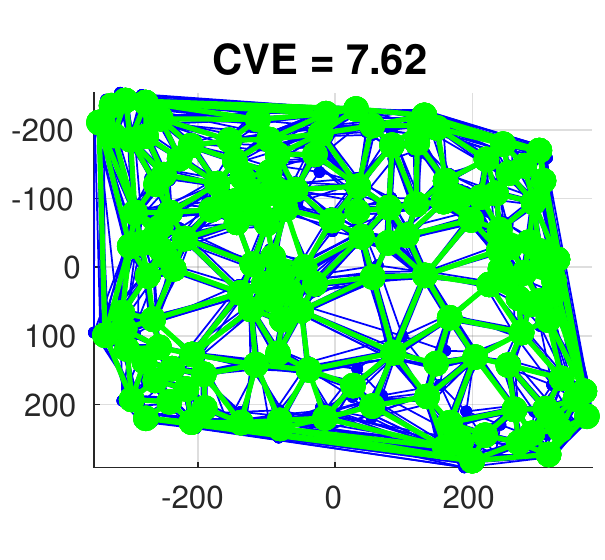}  &
		\includegraphics[width=0.14\textwidth]{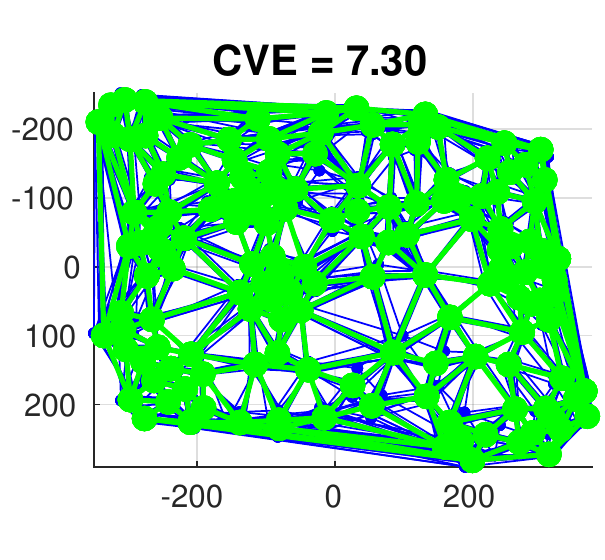}  \\
		\includegraphics[width=0.14\textwidth]{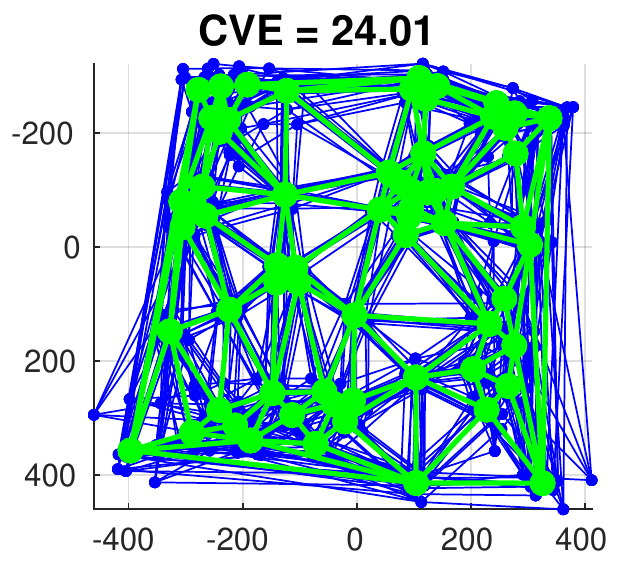}  &
		\includegraphics[width=0.143\textwidth]{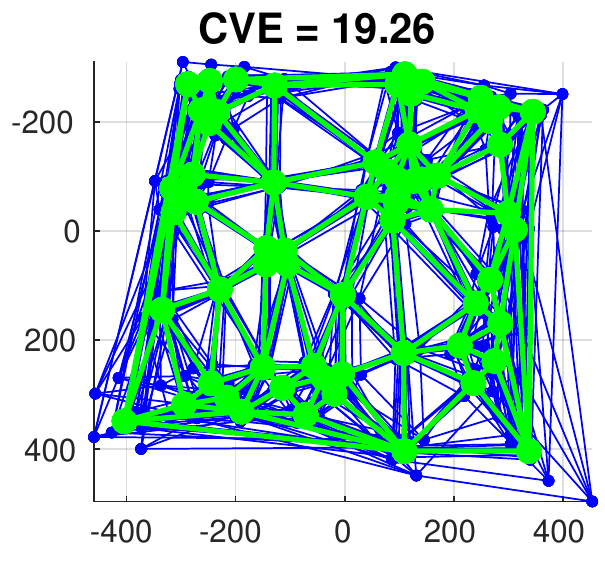} &
		\includegraphics[width=0.14\textwidth]{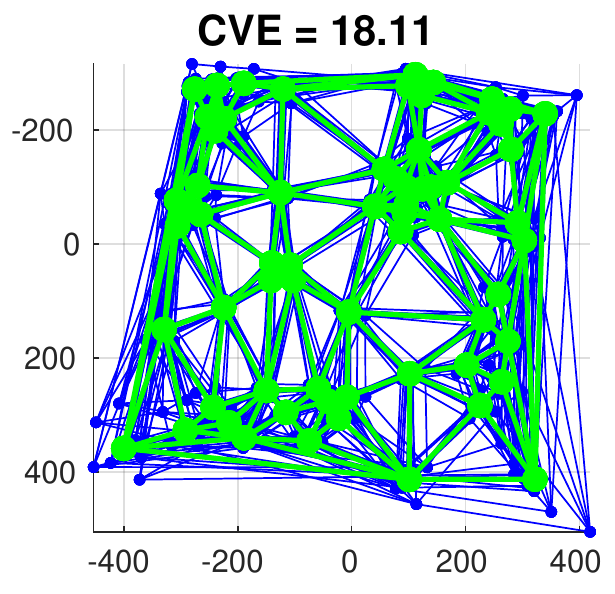}  &
		\includegraphics[width=0.134\textwidth]{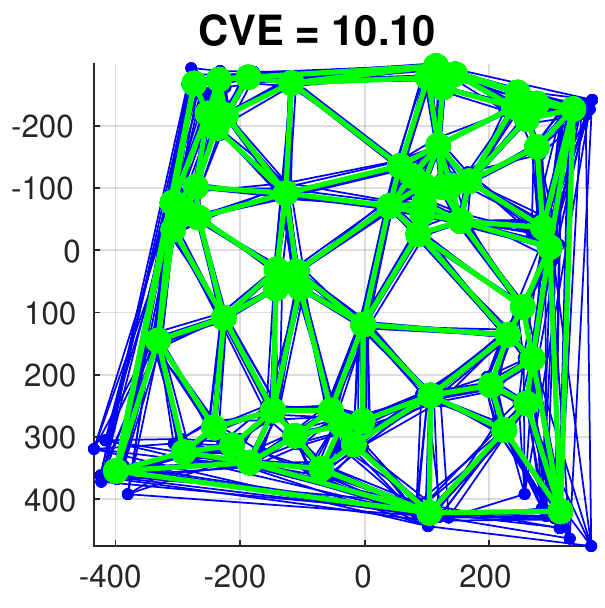} & 
		\includegraphics[width=0.13\textwidth]{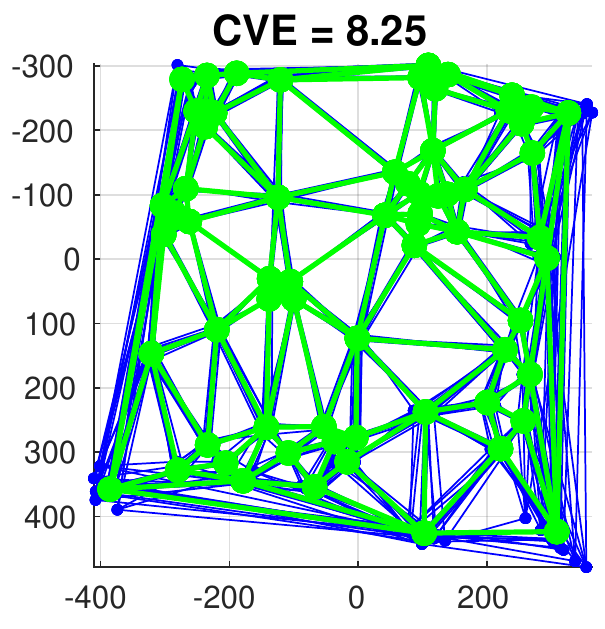}  &
		\includegraphics[width=0.13\textwidth]{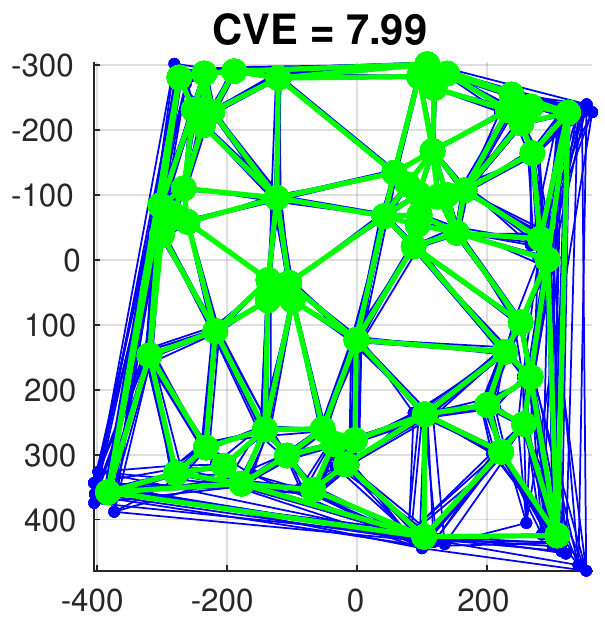}  \\
		\midrule\bottomrule
	\end{tabular*}
	\caption{The fitness of each methods visualized by the cross-validation on 2D datasets. The reference shape $\boldsymbol{S}^{\star}$ is plotted in green and the predicted reference shapes $\boldsymbol{\hat{S}}_{i}$ in blue.}
	\label{fig. landmark visualization of 2D datasets. The fitness of each registration method.}
\end{figure*}

\begin{figure*}[t]
	\centering
	\begin{tabular*}{0.99\textwidth}{@{} @{\extracolsep{\fill}} cccccc @{}}
		\toprule\midrule
		\textbf{\scriptsize \textsf{$\ast$EUC\_d}} & \textbf{\scriptsize \textsf{$\ast$AFF\_d}} & \textbf{\scriptsize \textsf{AFF\_r}} & \textbf{\scriptsize \textsf{TPS\_r($3$)}} & \textbf{\scriptsize \textsf{TPS\_r($5$)}} & \textbf{\scriptsize \textsf{TPS\_r($7$)}}  \\
		\midrule
		\includegraphics[width=0.145\textwidth]{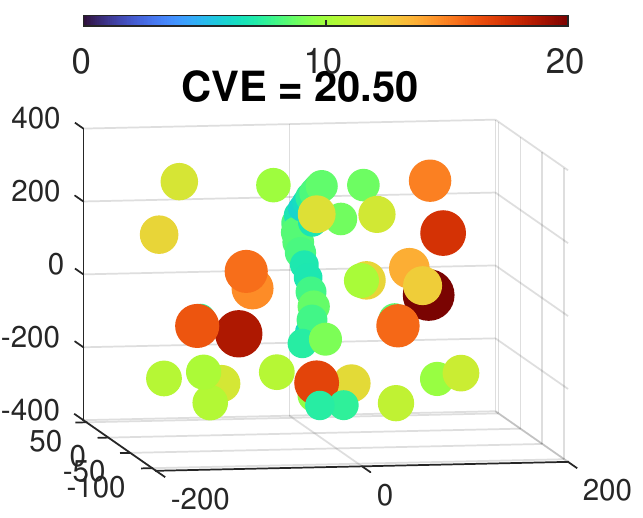}  &
		\includegraphics[width=0.145\textwidth]{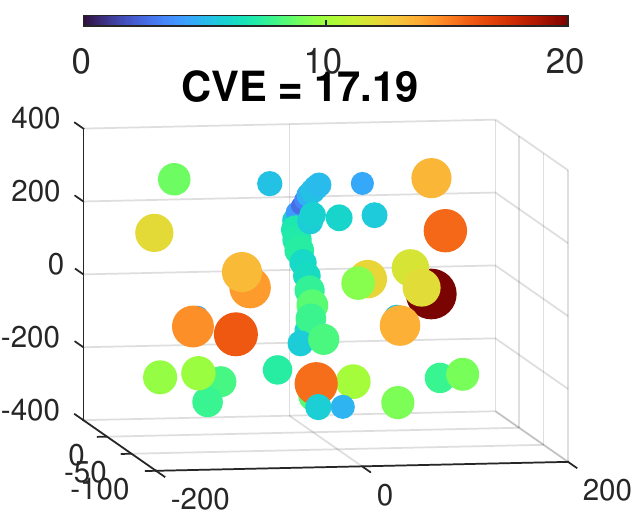} &
		\includegraphics[width=0.145\textwidth]{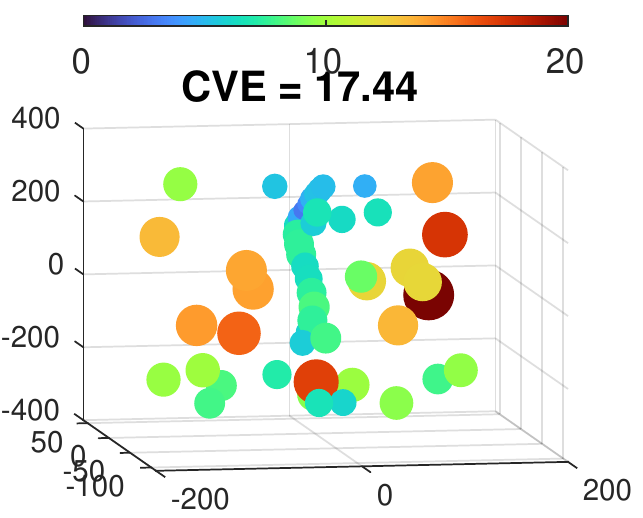}  &
		\includegraphics[width=0.145\textwidth]{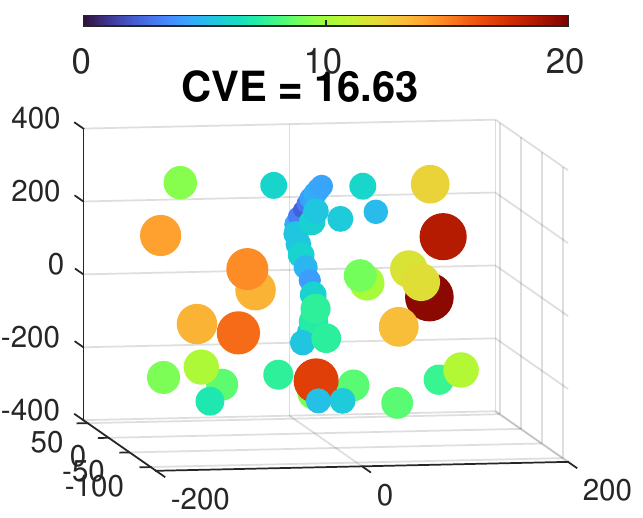} & 
		\includegraphics[width=0.145\textwidth]{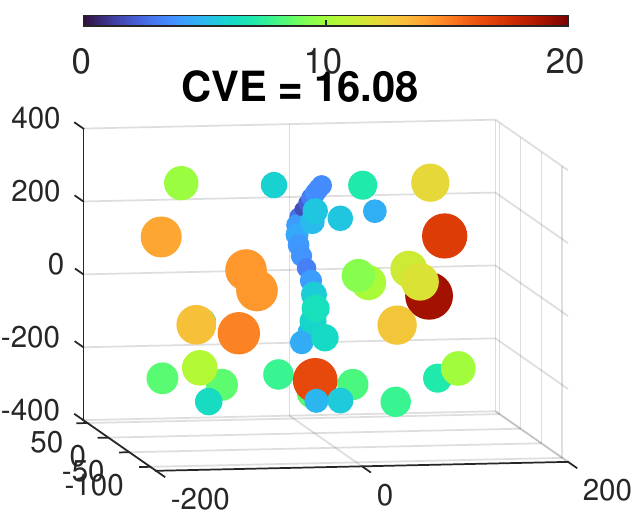}  &
		\includegraphics[width=0.145\textwidth]{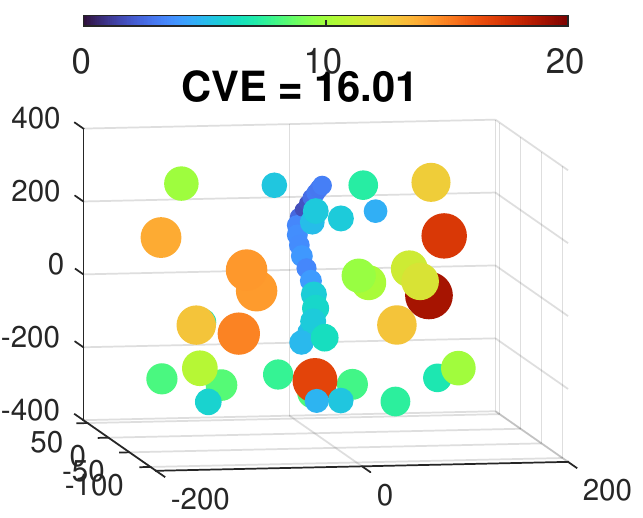}  \\
		\includegraphics[width=0.14\textwidth]{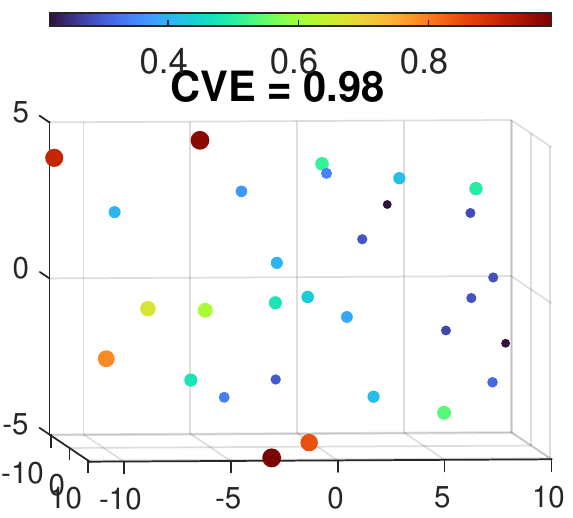}  &
		\includegraphics[width=0.14\textwidth]{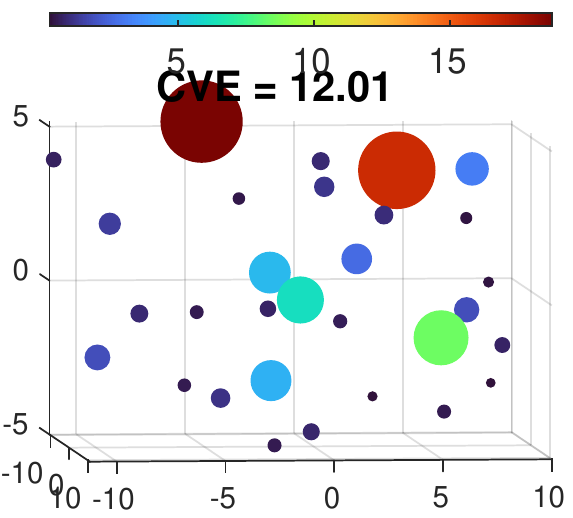} &
		\includegraphics[width=0.14\textwidth]{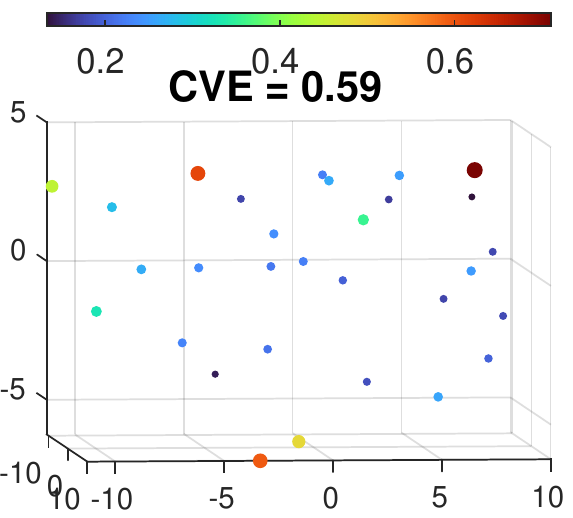}  &
		\includegraphics[width=0.14\textwidth]{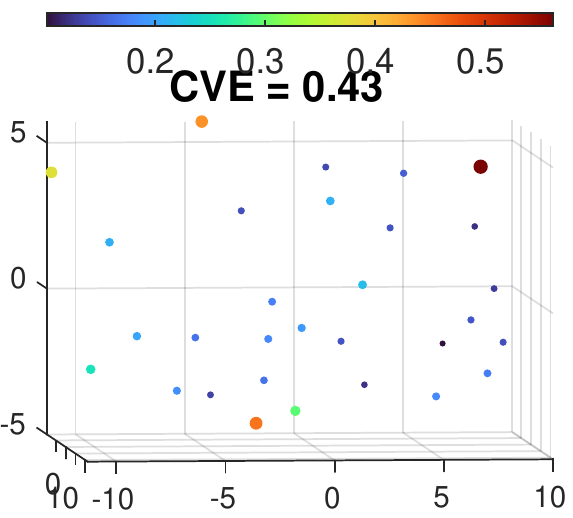} & 
		\includegraphics[width=0.14\textwidth]{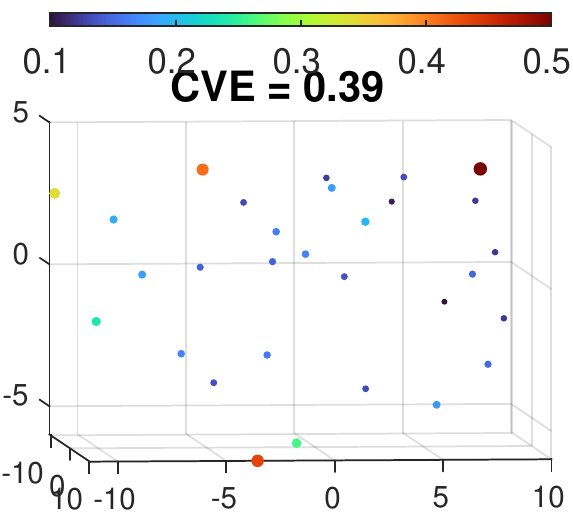}  &
		\includegraphics[width=0.14\textwidth]{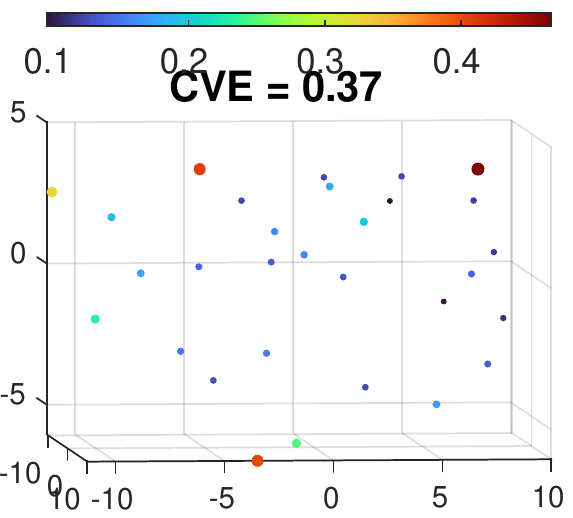}  \\
		\midrule\bottomrule
		\hfill
	\end{tabular*}
	\begin{tabular}{@{}  cccc @{}}
		\centering
		\begin{subfigure}{0.25\textwidth}
			\centering
			\includegraphics[width=1\textwidth]{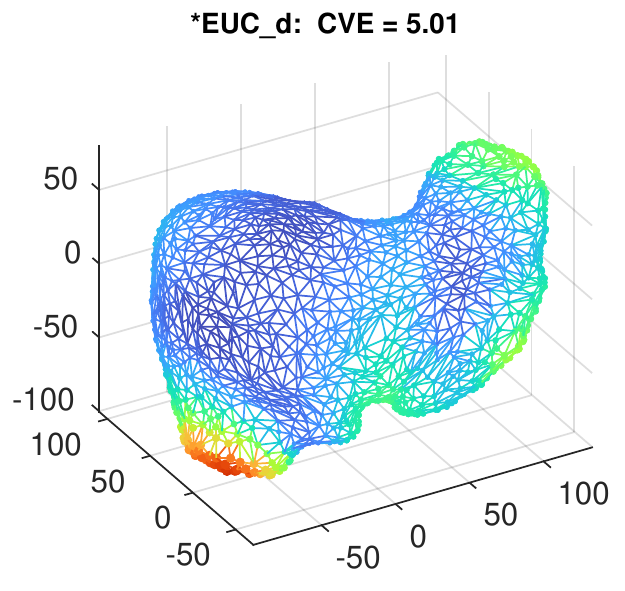}
		\end{subfigure}
		&
		\begin{subfigure}{0.25\textwidth}
			\centering
			\includegraphics[width=1\textwidth]{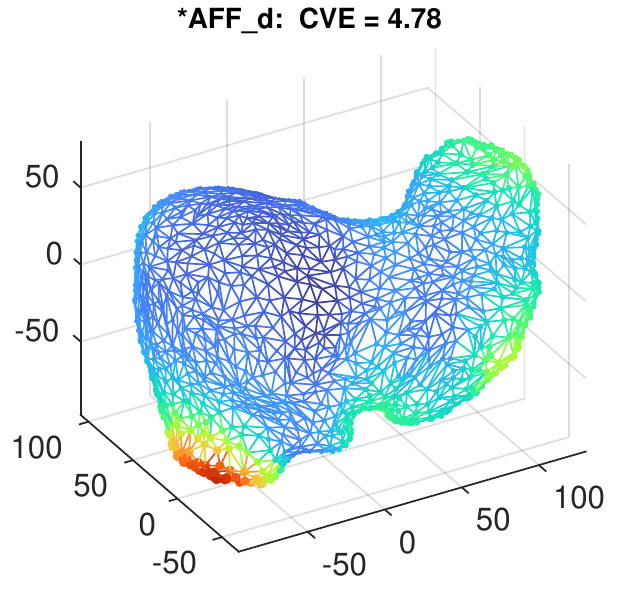}
		\end{subfigure}
		&
		\begin{subfigure}{0.25\textwidth}
			\centering
			\includegraphics[width=1\textwidth]{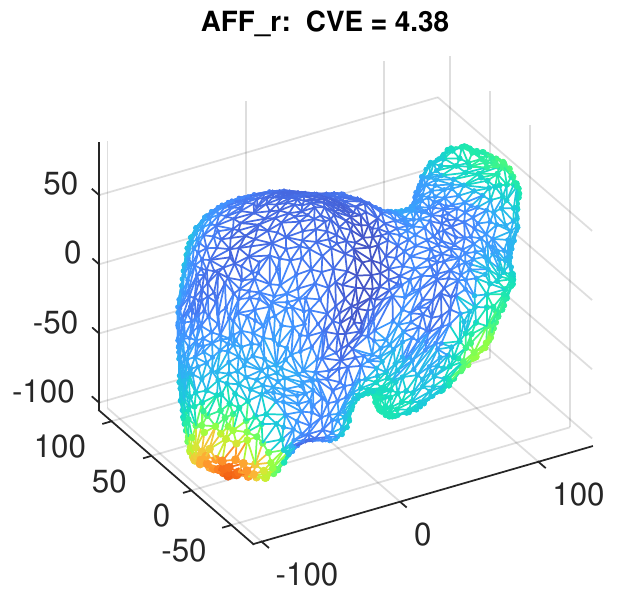}
		\end{subfigure}
		& 
		\multirow{2}{*}[40pt]{
			\includegraphics[width=0.035\textwidth]{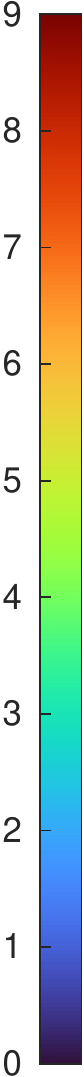}
		}
		\\
		\begin{subfigure}{0.25\textwidth}
			\centering
			\includegraphics[width=1\textwidth]{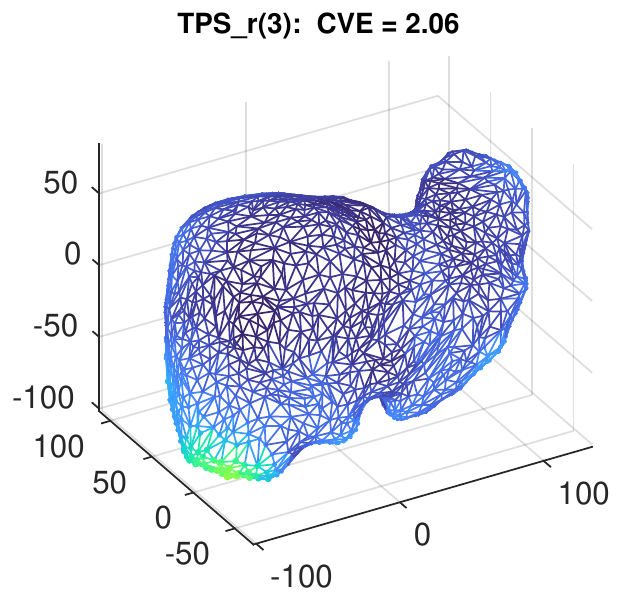}
		\end{subfigure}
		&
		\begin{subfigure}{0.25\textwidth}
			\centering
			\includegraphics[width=1\textwidth]{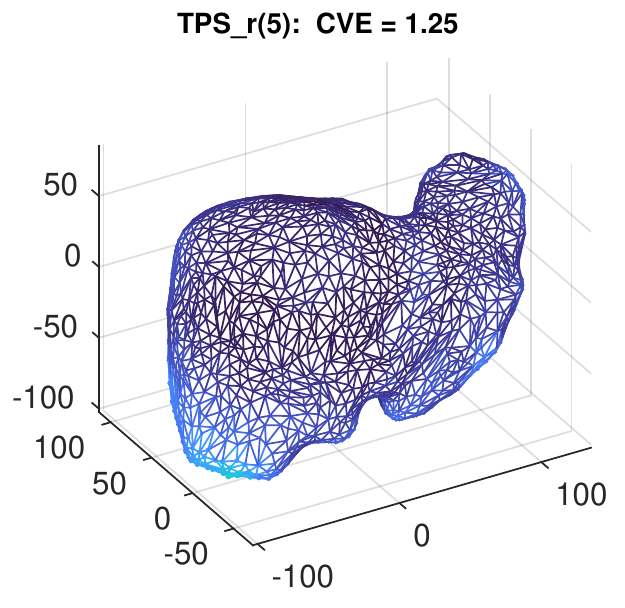}
		\end{subfigure}
		&
		\begin{subfigure}{0.25\textwidth}
			\centering
			\includegraphics[width=1\textwidth]{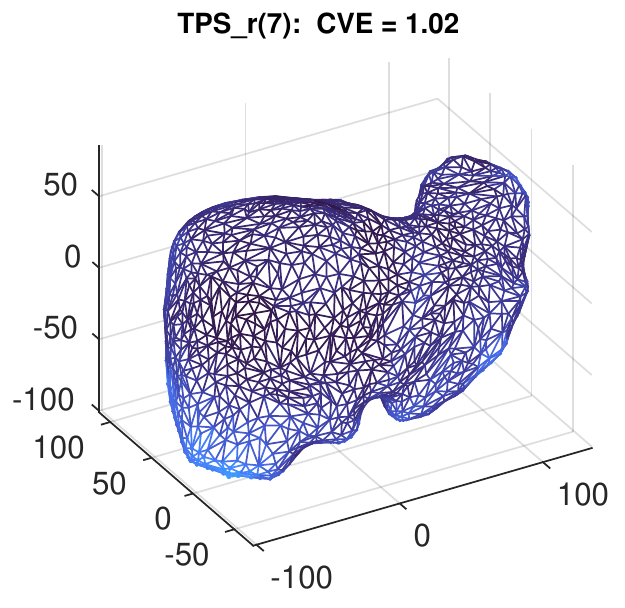}
		\end{subfigure}	
		&
		\\
	\end{tabular}
	\caption{The fitness of each methods visualized by the cross-validation on 3D datasets. We present the reference shape $\boldsymbol{S}^{\star}$ and encode the cross-validation error for each point with both the marker sizes and colors. The smaller the marker, the lower the error.}
	\label{fig. landmark visualization of 3D cases. The fitness of each registration method.}
\end{figure*}

\subsection{Results Based on Image Intensities}

For the 2D datasets, we visualize the image intensities under the computed warps.
The image pixel coordinates can be regarded as testing points (in addition to the landmarks) which are immune to the issue of overfitting.
We thus evaluate the fitness using the intensities in the transformed pixel coordinates.

Let $I_i$ be a sample image. For each pixel coordinates $[u,v]$, we apply the transformation $\mathcal{T}_i$ to obtain its target position $[u',v'] = \mathcal{T}_i \left( [u,v] \right)$. The \textit{transformed image} (or warped image) $I_i'$ is defined as
$I_i' (u',v') = I_i (u,v)$.
The warped image $I_i'$ resides in the coordinate frame of the reference shape, thus can be compared with other warped images.
Having computed $n$ warped sample images $I'_i$ $\left( i \in \left[1:n\right]\right)$,
we define the \textit{image intensity deviation} as an image
$\Delta I'(u,v) = \mathbf{std}\left\{I'_1(u,v), I'_2(u,v),\dots,I'_n(u,v)\right\}$.
In the case of RGB images, we compute the intensity deviation for all the three channels to obtain an RGB intensity deviation image.

We visualize the image intensity deviation for the Bag and Pillow datasets in Figure~\ref{fig. image intensity deviation, 2d images}.
The result for $\ast$AFF\_d is similar to AFF\_r, thus is not shown.
In Figure~\ref{fig. image intensity deviation, 2d images},
from left to right,
the methods $\ast$EUC\_d, AFF\_r, TPS\_r($3$), TPS\_r($5$) and TPS\_r($7$) consistently increase the fitness due to the increasing modeling capabilities (from rigid to affine and deformable models), shown as the shrinking of bright regions.

The image intensity deviation of the Face dataset is noisy due to perspective changes, so we visualize a typical warped sample image (with missing correspondences and perspective changes) and the corresponding transformed landmarks in Figure~\ref{fig. transformed images on Face dataset, 2d images}.
This result helps intuitively understand how each transformation model works: the rigid model preserves the distance; the affine model shears the image; the deformation models can deform the image nonlinearly (\textit{e.g.,}~the area around the nose and mouth in Figure~\ref{fig. transformed images on Face dataset, 2d images}).
The improved fitness using deformation models is obvious without saying.

\begin{figure*}[th]
	\centering
	\begin{subfigure}{.99\textwidth}
		\centering
		\includegraphics[width=0.18\textwidth]{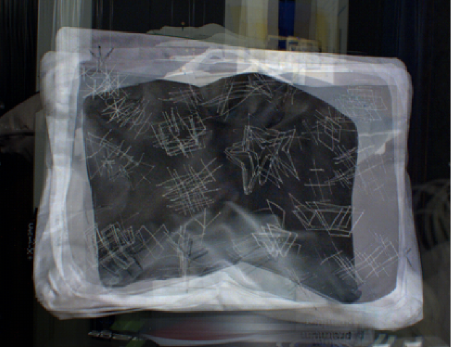}
		\includegraphics[width=0.18\textwidth]{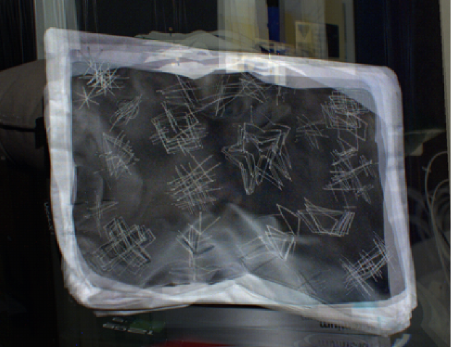}
		\includegraphics[width=0.18\textwidth]{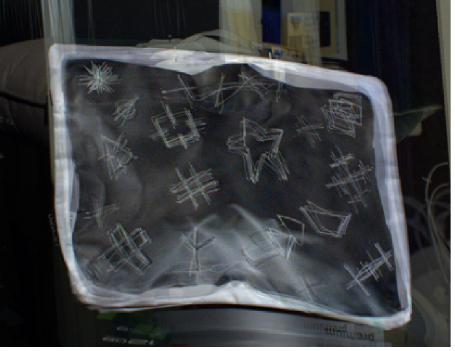}
		\includegraphics[width=0.18\textwidth]{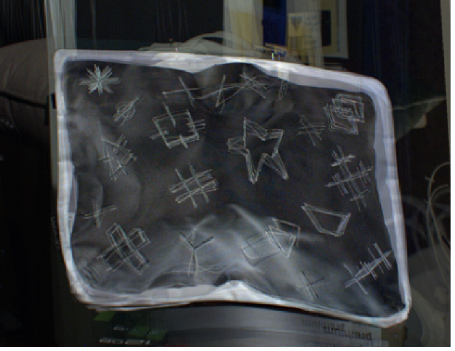}
		\includegraphics[width=0.18\textwidth]{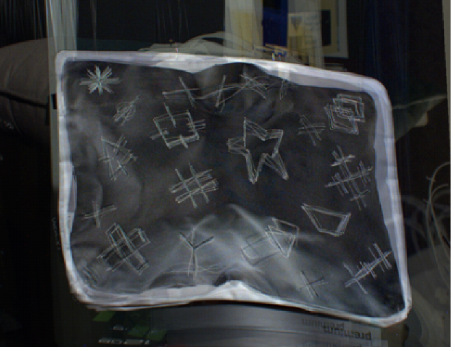}
		\caption{Bag}
	\end{subfigure}
	\begin{subfigure}{.99\textwidth}
		\centering
		\includegraphics[width=0.18\textwidth]{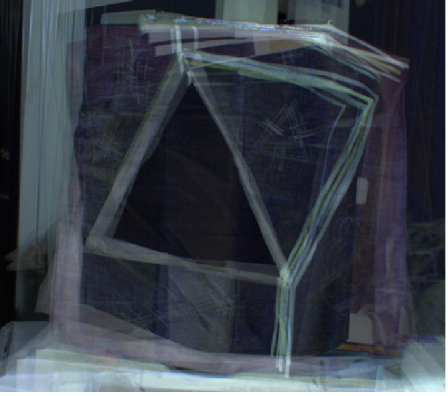}
		\includegraphics[width=0.18\textwidth]{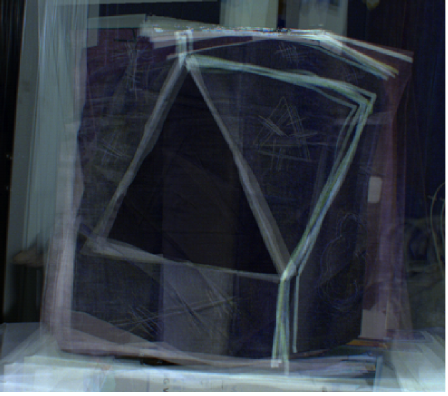}
		\includegraphics[width=0.18\textwidth]{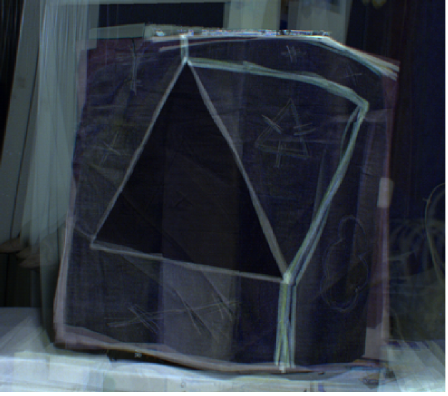}
		\includegraphics[width=0.18\textwidth]{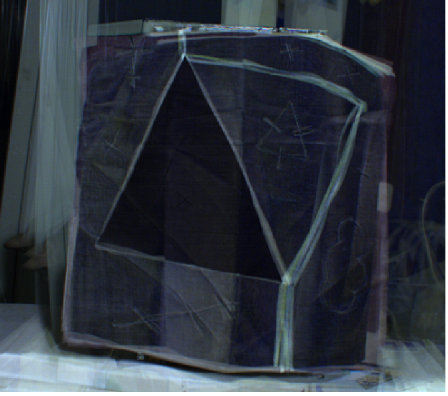}
		\includegraphics[width=0.18\textwidth]{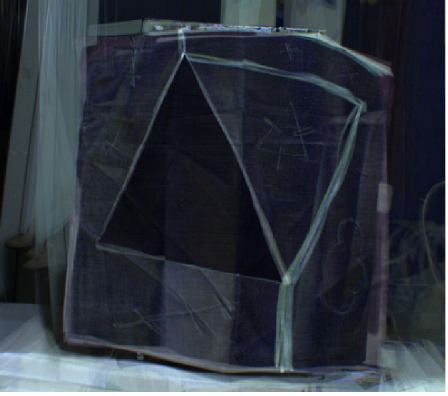}
		\caption{Pillow}
	\end{subfigure}
	\caption{The intensity deviation of the transformed datum images. From left to right are respectively the intensity deviation for $\ast$EUC\_d, AFF\_r, TPS\_r($3$), TPS\_r($5$) and TPS\_r($7$). The bright regions correspond to large intensity deviations, and the dark regions to small intensity deviations.}
	\label{fig. image intensity deviation, 2d images}
\end{figure*}

\begin{figure*}[th]
	\centering
	\begin{subfigure}{.99\textwidth}
		\centering
		\includegraphics[width=0.15\textwidth]{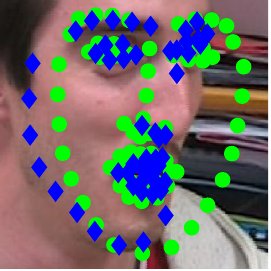}
		\includegraphics[width=0.15\textwidth]{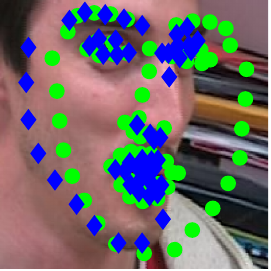}
		\includegraphics[width=0.15\textwidth]{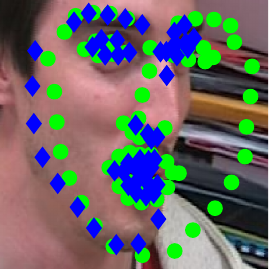}
		\includegraphics[width=0.15\textwidth]{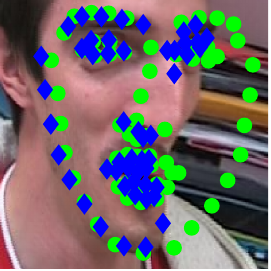}
		\includegraphics[width=0.15\textwidth]{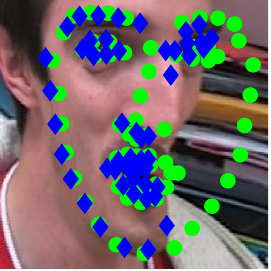}
		\includegraphics[width=0.15\textwidth]{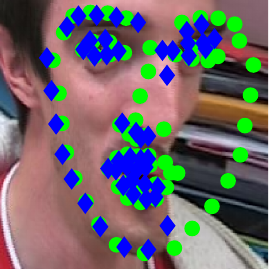}
	\end{subfigure}
	\caption{The transformed images using the computed warps. From left to right are respectively the warped images by $\ast$EUC\_d, $\ast$AFF\_d, AFF\_r, TPS\_r($3$), TPS\_r($5$) and TPS\_r($7$). The transformed landmarks are marked in blue, and the landmarks in the reference shape are marked in green.}
	\label{fig. transformed images on Face dataset, 2d images}
\end{figure*}

\subsection{The Asymmetry in Cost Functions}
\label{exp_section: asymmetry in cost functions}

\begin{figure*}[t]
	\centering
	\begin{subfigure}{.45\textwidth}
		\centering
		\includegraphics[width=0.99\textwidth]{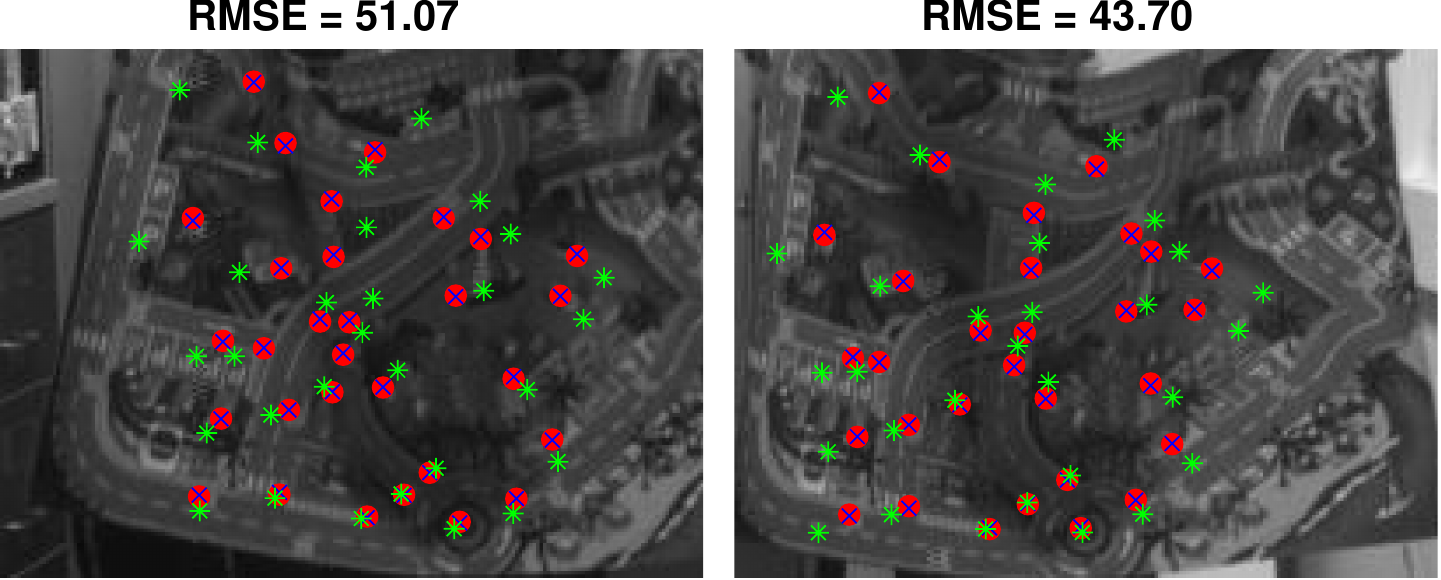}
		\caption{$\ast$EUC\_d}
	\end{subfigure}
	\hfil
	\begin{subfigure}{.45\textwidth}
		\centering
		\includegraphics[width=0.99\textwidth]{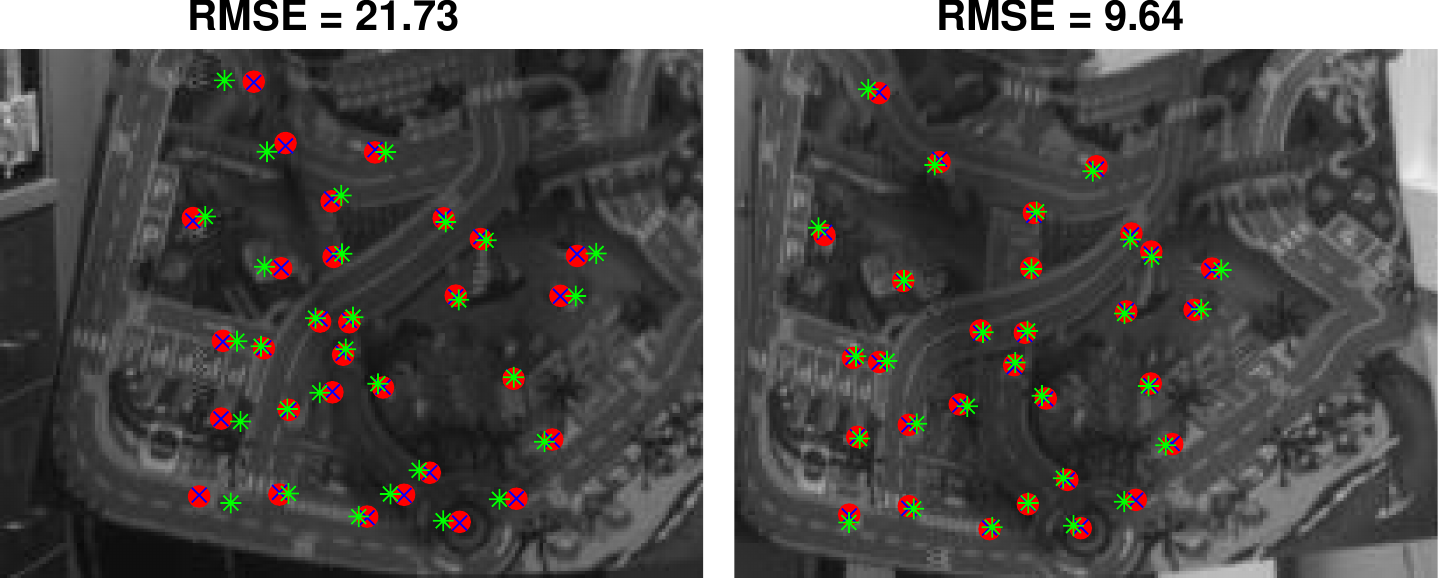}
		\caption{$\ast$AFF\_d}
	\end{subfigure}
	\begin{subfigure}{.45\textwidth}
		\centering
		\includegraphics[width=0.99\textwidth]{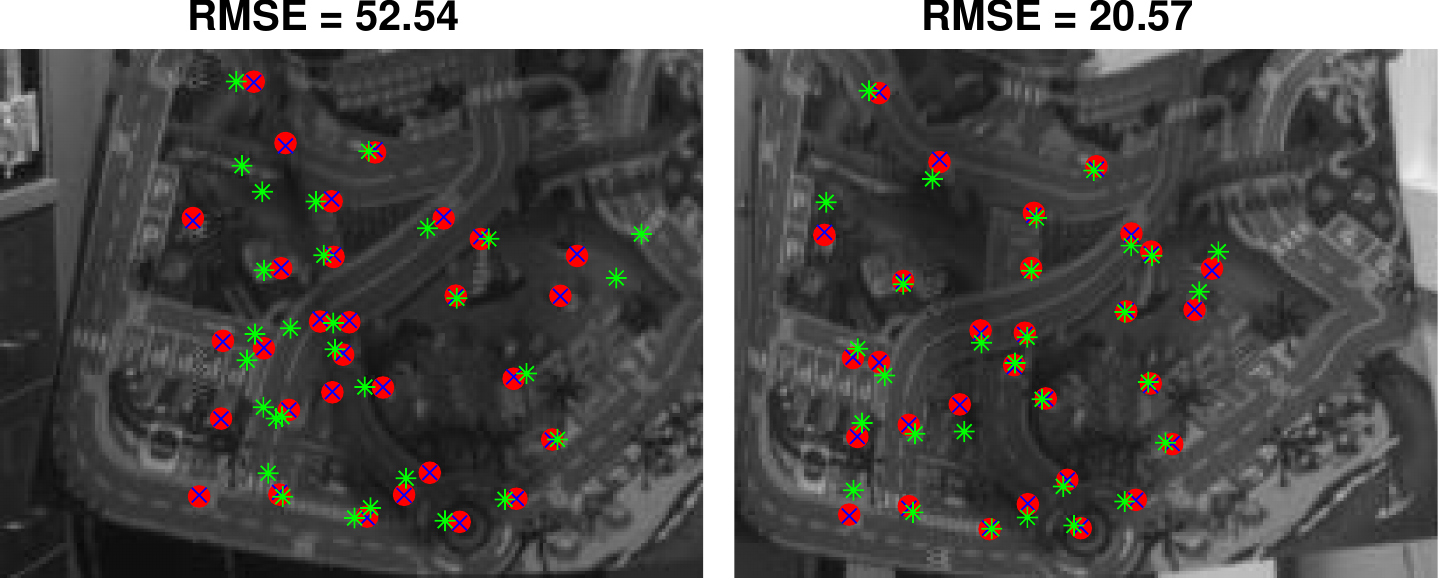}
		\caption{AFF\_r}
	\end{subfigure}
	\hfil
	\begin{subfigure}{.45\textwidth}
		\centering
		\includegraphics[width=0.99\textwidth]{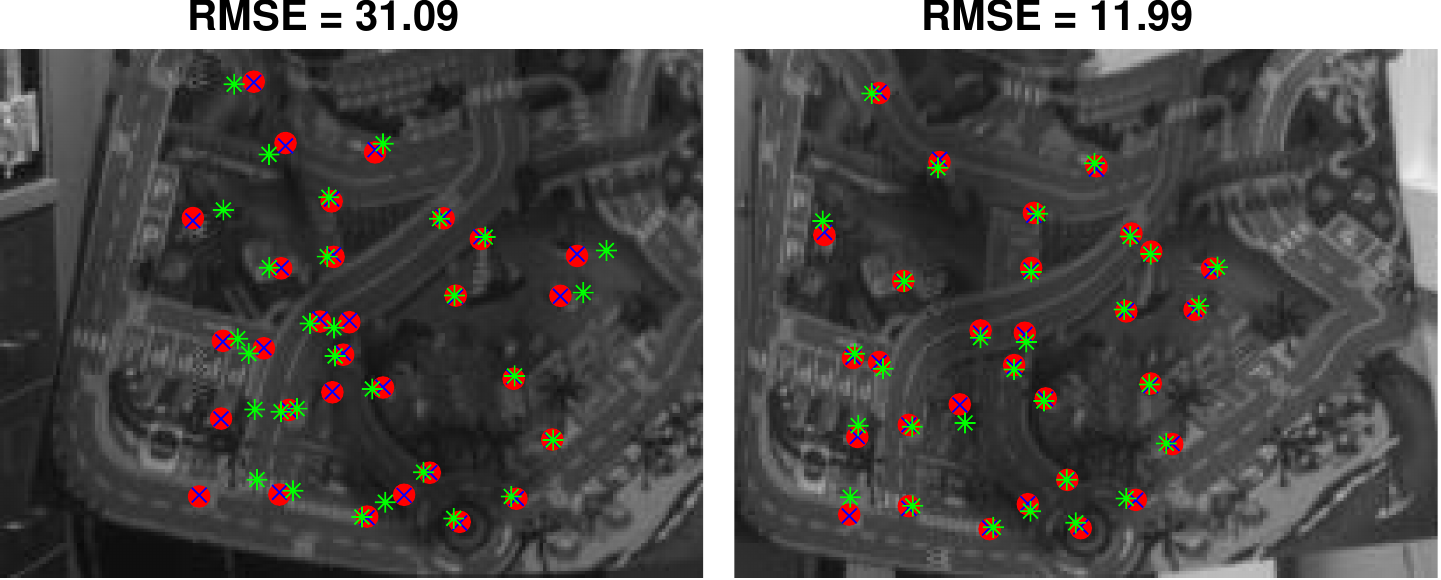}
		\caption{TPS\_r($3$)}
	\end{subfigure}
	\begin{subfigure}{.45\textwidth}
		\centering
		\includegraphics[width=0.99\textwidth]{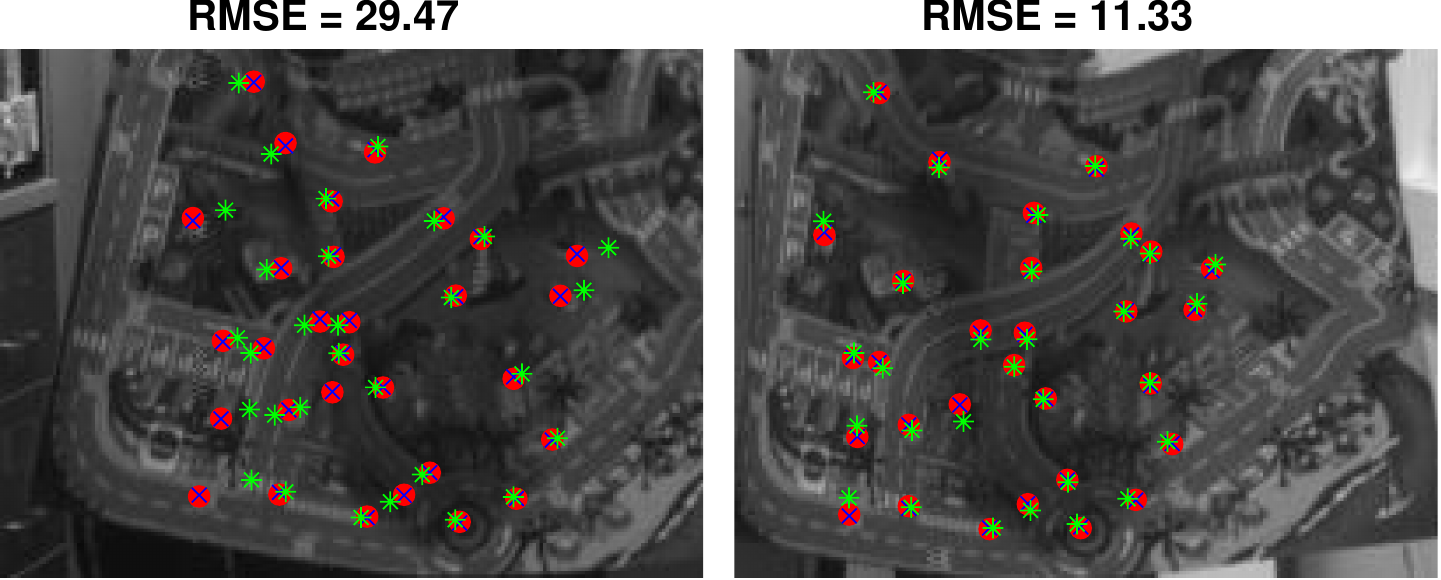}
		\caption{TPS\_r($5$)}
	\end{subfigure}
	\hfil
	\begin{subfigure}{.45\textwidth}
		\centering
		\includegraphics[width=0.99\textwidth]{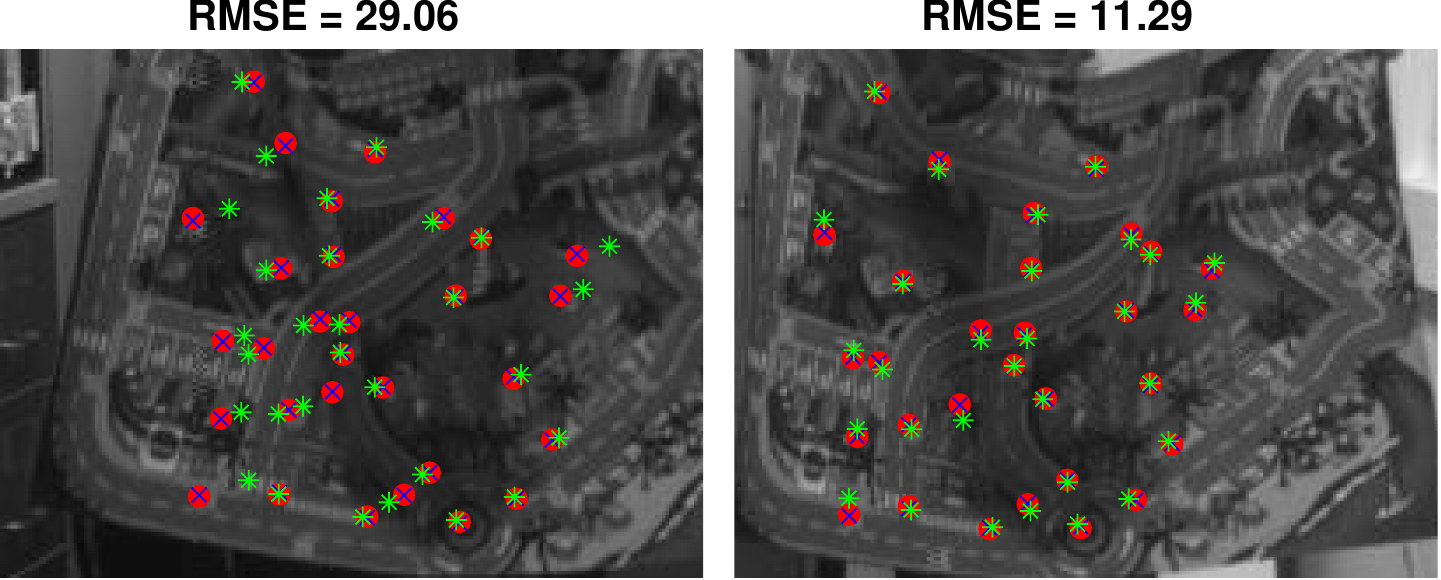}
		\caption{TPS\_r($7$)}
	\end{subfigure}
	\caption{The discrepancy between the generated image features and original image features on the ToyRug dataset.
		The generated image features are plotted in the green color, and the original image features in the red color.
		The blue x-markers are reprojected 2D image features from the triangulated 3D datum shapes.
		The RMSE statistics are measured in pixels.}
	\label{fig: the generated image features of each method on the ToyRug dataset}
\end{figure*}

For the ToyRug dataset, the RMSE\_r and RMSE\_d statistics in Table \ref{The residual error and computational time of each method on benchmark datasets} have shown a strong asymmetry between the reference-space cost and datum-space cost, in particular for the $\ast$AFF\_d and AFF\_r methods.
This asymmetry is further evidenced in
Figure~\ref{fig. landmark visualization of 3D cases. The fitness of each registration method.} as the $\ast$AFF\_d method yields one magnitude larger CVE error than the AFF\_r method.
Recall that the datum-space cost aims to find the generative model that best explains all the data; in contrast, the reference-space cost aims to find the reference shape that produces the best fitness to the data.
The above results suggest that while being closely related, minimizing the datum-space cost does not necessarily result in a good fitness to the reference shape.
Now we report on the other hand that minimizing the reference-space does not necessarily produce a good generative model.

The 3D datum shapes of the ToyRug dataset are obtained by triangulating the image features of a stereo camera rig with known intrinsic and extrinsic parameters.
To evaluate GPA methods as generative models,
we use the 3D reference shape estimate and transformation estimates to generate 3D datum shapes, and reproject the generated 3D datum shapes to 2D image features.
We report the discrepancies between the generated image features and the original image futures in Figure~\ref{fig: the generated image features of each method on the ToyRug dataset}, by visualizing the result of the datum shape that produces the largest error.

Contrary to the worst reference-space fitness,
the $\ast$AFF\_d method generates the best image features that closely match the original ones as shown in Figure~\ref{fig: the generated image features of each method on the ToyRug dataset},
because it minimizes the datum-space cost directly while being more flexible than the $\ast$EUC\_d method.
Due to nonrigid deformations,
the datum-space method $\ast$EUC\_d does not provide good result.
In parallel, because of the metric asymmetry,
poor performance is observed for the AFF\_r method which minimizes the reference-space cost.
However,
although not optimizing the datum-space cost directly,
the reference-space methods TPS\_r($3$), TPS\_r($5$) and TPS\_r($7$) are still capable to provide satisfactory results.
The quality of the generated features are consistently improved over the TPS\_r($3$), TPS\_r($5$) and TPS\_r($7$) methods.
In particular, the features in the bottom right corner generated by the method TPS\_r($7$) are better than those generated by the $\ast$AFF\_d method.

This brings us to the final conclusion. The datum-space cost and the reference-cost are two distinct but closely related metrics, while a good result based on one cost does not necessarily guarantee a good result on the other cost due to the metric asymmetry.
However it does not mean that the cost functions make a big difference for every dataset: as can be seen in  Table \ref{The residual error and computational time of each method on benchmark datasets}, the RMSE\_r and RMSE\_d metrics agree most of the time.
In adverse cases like the ToyRug dataset,
the deformation models like the TPS warp can achieve the best result on both sides,
which coins the importance of using DefGPA over classical GPA methods.

\section{Conclusion}
\label{section: conclusion}

To summarize, we have introduced the problem of GPA with deformations.
We have proposed a general problem statement applicable to LBWs, subsuming previous statements made for specific models.
We have derived a closed-form solution, applicable to the general case of partial shapes and with regularization.
We have extensively validated the proposed solution on various datasets with great care taken to select the regularization weights.
Our future work will include studying the consistency of estimates under certain statistical models, and extending our theory to the Nonrigid Structure-from-Motion (NRSfM) problem.

\begin{appendices}

\section{The Thin-Plate Spline}
\label{Appendix: Thin-Plate Spline}

\subsection{General Mapping Definition}

The Thin-Plate Spline (TPS) is a nonlinear mapping driven by a set of chosen control centers.
Let $\boldsymbol{c}_1, \boldsymbol{c}_2, \dots, \boldsymbol{c}_l$, with $\boldsymbol{c}_i \in \mathbb{R}^{d}$ $(d=2,3)$,
be the $l$ control centers of TPS.
We present the TPS model in 3D; the 2D case can be analogously derived by removing the $z$ coordinate and letting the corresponding TPS kernel function be $\phi \left(r\right) = r^2 \log\left(r^2\right)$ \citep{bookstein1997morphometric}.
Strictly speaking, the TPS only occurs in 2D, but is part of a general family of functions called the harmonic splines. For simplicity, we call the function TPS independently of the dimension.

We denote a point in 3D by $\boldsymbol{p}^{\top} = \left[x, y,z\right]$.
The TPS function $\tau (\boldsymbol{p})$ is a $\mathbb{R}^3 \rightarrowtail \mathbb{R}$ mapping:
\begin{equation*}
\tau (\boldsymbol{p}) = 
\sum_{k=1}^l w_k\,  \phi \left(\lVert \boldsymbol{p} - \boldsymbol{c}_k \rVert\right) 
+ a_1 x + a_2 y + a_3 z + a_4
\end{equation*}
where $\phi\left(\cdot\right)$ is the TPS kernel function.
In 3D, we have $\phi \left(r\right) = - \vert r \vert $.

We collect the parameters in a single vector $\boldsymbol{\eta}^{\top} = \left[\boldsymbol{w}^{\top}, \boldsymbol{a}^{\top}\right]$,
with the coefficient in the nonlinear part as
$\boldsymbol{w}^{\top} = \left[w_1, w_2, \cdots, w_l\right]$, and that in the linear part as $\boldsymbol{a}^{\top} = \left[a_1, a_2, a_3, a_4\right]$.
Let $\boldsymbol{\tilde{p}}^{\top} = \left[x, y, z, 1\right]$.
Collect the nonlinear components in a vector:
\begin{equation*}
\boldsymbol{\phi}_{\boldsymbol{p}}^{\top} = 
\left[
\phi \left(\lVert \boldsymbol{p} - \boldsymbol{c}_1 \rVert\right),\, \phi \left(\lVert \boldsymbol{p} - \boldsymbol{c}_2 \rVert\right),\, \cdots,\, 
\phi \left(\lVert \boldsymbol{p} - \boldsymbol{c}_l \rVert\right)
\right]
.
\end{equation*}
Then the TPS function compacts in the vector form as:
\begin{equation*}
\tau (\boldsymbol{p}) = 
\boldsymbol{\phi}_{\boldsymbol{p}}^{\top} \boldsymbol{w} + \boldsymbol{\tilde{p}}^{\top} \boldsymbol{a}
.
\end{equation*}
We define $\boldsymbol{\tilde{c}}_i^{\top} = \left[\boldsymbol{c}_i^{\top}, 1\right]$ $(i \in \left[ 1:l \right])$,
and
$\boldsymbol{\tilde{C}} = \left[
\boldsymbol{\tilde{c}}_1,   \boldsymbol{\tilde{c}}_2,
\dots,
\boldsymbol{\tilde{c}}_l\right]$.
The transformation parameters in the nonlinear part of the TPS model are constrained with respect to the control centers,
\textit{i.e.},
$
\boldsymbol{\tilde{C}} \boldsymbol{w} = \boldsymbol{0}
$.
To solve the parameters of TPS, 
the $l$ control centers
$\boldsymbol{c}_1, \boldsymbol{c}_2, \dots, \boldsymbol{c}_l$,
are mapped to $l$ scalar outputs $y_1, y_2, \dots, y_l$,
by
$\tau (\boldsymbol{c}_k) = y_k$ $(k \in \left[ 1:l \right])$.
We collect the outputs in a vector
$\boldsymbol{y}^{\top} = \left[y_1, y_2, \dots, y_l\right]$.

\subsection{Standard Form as a Linear Regression Model}
\label{appendix: TPS Standard Form as a Linear Regression Model}

While there are $l+4$ parameters in $\boldsymbol{\eta}$,
we will show in the following that the $4$ constraints in $
\boldsymbol{\tilde{C}} \boldsymbol{w} = \boldsymbol{0}
$ can be absorbed in a reparameterization of the TPS model, by using $l$ parameters only.

To proceed, we define a matrix:
\begin{equation*}
\boldsymbol{K}_{\lambda} = 
\begin{bmatrix}
\lambda & \phi_{12} & \dots & \phi_{1l}  \\
\phi_{21} & \lambda & \dots & \phi_{2l}  \\
\vdots & \vdots &    & \vdots \\
\phi_{l1}   & \phi_{l2}   & \dots & \lambda
\end{bmatrix}
,
\end{equation*}
where we use the shorthand
$
\phi_{ij}
=
\phi \left(\lVert \boldsymbol{c}_i - \boldsymbol{c}_j \rVert\right)
$.
The original diagonal elements of $\boldsymbol{K}_{\lambda}$ are all zeros, \textit{i.e.,}~$\phi_{11} = \phi_{22} = \cdots = \phi_{ll} = 0$.
$\lambda$ is an adjustable parameter acting as an internal smoothing parameter to improve the conditioning of the matrix.

The $l$ TPS mappings of control centers $\tau (\boldsymbol{c}_k) = y_k$ $(k \in \left[ 1:l \right])$, and the $4$ parameter constraints $
\boldsymbol{\tilde{C}} \boldsymbol{w} = \boldsymbol{0}
$
can be collected in the following matrix form:
\begin{equation*}
\begin{bmatrix}
\boldsymbol{K}_{\lambda}  & \boldsymbol{\tilde{C}}^{\top} \\
\boldsymbol{\tilde{C}} & \boldsymbol{O}
\end{bmatrix}
\begin{bmatrix}
\boldsymbol{w} \\ \boldsymbol{a}
\end{bmatrix}
=
\begin{bmatrix}
\boldsymbol{y} \\ \boldsymbol{0}
\end{bmatrix}
,
\end{equation*}
which admits a closed-form solution:
\begin{equation*}
\begin{bmatrix}
\boldsymbol{w} \\ \boldsymbol{a}
\end{bmatrix}
=
\boldsymbol{\mathcal{E}}_{\lambda} \boldsymbol{y}
,
\end{equation*}
with the matrix $\boldsymbol{\mathcal{E}}_{\lambda} \in \mathbb{R}^{\left(l+4\right) \times l}$ defined as:
\begin{equation*}
\boldsymbol{\mathcal{E}}_{\lambda} = 
\begin{bmatrix}
\boldsymbol{K}_{\lambda}^{-1} - \boldsymbol{K}_{\lambda}^{-1} \boldsymbol{\tilde{C}}^{\top} \left(\boldsymbol{\tilde{C}} \boldsymbol{K}_{\lambda}^{-1} \boldsymbol{\tilde{C}}^{\top}\right)^{-1}
\boldsymbol{\tilde{C}}
\boldsymbol{K}_{\lambda}^{-1} \\
\left(\boldsymbol{\tilde{C}} \boldsymbol{K}_{\lambda}^{-1} \boldsymbol{\tilde{C}}^{\top}\right)^{-1}
\boldsymbol{\tilde{C}}
\boldsymbol{K}_{\lambda}^{-1}
\end{bmatrix}
.
\end{equation*}

Therefore the TPS function writes:
\begin{align*}
\tau (\boldsymbol{p})  =
\begin{bmatrix}
\boldsymbol{\phi}_{\boldsymbol{p}} \\ \boldsymbol{\tilde{p}}
\end{bmatrix}^{\top}
\begin{bmatrix}
\boldsymbol{w} \\ \boldsymbol{a}
\end{bmatrix}
& =
\begin{bmatrix}
\boldsymbol{\phi}_{\boldsymbol{p}} \\ \boldsymbol{\tilde{p}}
\end{bmatrix}^{\top}
\boldsymbol{\mathcal{E}}_{\lambda} \boldsymbol{y}
\\[5pt] & =
\boldsymbol{y}^{\top}
\boldsymbol{\mathcal{E}}_{\lambda}^{\top}
\begin{bmatrix}
\boldsymbol{\phi}_{\boldsymbol{p}} \\ \boldsymbol{\tilde{p}}
\end{bmatrix}
.
\end{align*}
Let $\boldsymbol{\beta} (\boldsymbol{p}) = 	\boldsymbol{\mathcal{E}}_{\lambda}^{\top} \begin{bmatrix}
\boldsymbol{\phi}_{\boldsymbol{p}} \\ \boldsymbol{\tilde{p}}
\end{bmatrix}$.
Then
$
\tau (\boldsymbol{p}) =  \boldsymbol{y}^{\top}  \boldsymbol{\beta} (\boldsymbol{p})
$.

The 3D TPS warp model $\mathcal{W}\left(\boldsymbol{p}\right)\,:\, \mathbb{R}^3 \rightarrowtail \mathbb{R}^3$ is obtained by stacking $3$ TPS functions together, one for each dimension:
\begin{equation}
\label{eq: TPS warp model, W function definition}
	\begin{aligned}
	\mathcal{W} (\boldsymbol{p}) & = 
	\begin{bmatrix}
		\tau_{x}(\boldsymbol{p}) \\
		\tau_{y}(\boldsymbol{p}) \\
		\tau_{z}(\boldsymbol{p})
	\end{bmatrix}
	=
	\begin{bmatrix}
		\boldsymbol{y}_{x}^{\top}\, \boldsymbol{\beta} (\boldsymbol{p}) \\
		\boldsymbol{y}_{y}^{\top}\,  \boldsymbol{\beta} (\boldsymbol{p}) \\
		\boldsymbol{y}_{z}^{\top}\,  \boldsymbol{\beta} (\boldsymbol{p})
	\end{bmatrix}
	\\[5pt] & =
	\left[\boldsymbol{y}_{x}, \boldsymbol{y}_{y}, \boldsymbol{y}_{z} \right]^{\top}
	\boldsymbol{\beta} (\boldsymbol{p})
	.
	\end{aligned}
\end{equation}
This is the standard form of the general warp model defined in equation (\ref{eq - warp model - operating on point}),
with $\boldsymbol{W} = \left[\boldsymbol{y}_{x}, \boldsymbol{y}_{y}, \boldsymbol{y}_{z} \right] \in \mathbb{R}^{l \times 3}$ being a set of new transformation parameters that are free of constraints.

\subsection{Regularization by Bending Energy Matrix}
\label{appendix, TPS Regularization by Bending Energy Matrix}

Let 
$\boldsymbol{\bar{\mathcal{E}}}_{\lambda}
=	\boldsymbol{K}_{\lambda}^{-1} - \boldsymbol{K}_{\lambda}^{-1} \boldsymbol{\tilde{C}}^{\top} \left(\boldsymbol{\tilde{C}} \boldsymbol{K}_{\lambda}^{-1} \boldsymbol{\tilde{C}}^{\top}\right)^{-1}
\boldsymbol{\tilde{C}}
\boldsymbol{K}_{\lambda}^{-1}$.
The $l \times l$ matrix $\boldsymbol{\bar{\mathcal{E}}}_{\lambda}$ is the bending energy matrix of the TPS warp \citep{bookstein1989principal},
which satisfies
$ \boldsymbol{\bar{\mathcal{E}}}_{\lambda} \boldsymbol{K}_{\lambda} \boldsymbol{\bar{\mathcal{E}}}_{\lambda}
=
\boldsymbol{\bar{\mathcal{E}}}_{\lambda}
$.
The bending energy matrix $\boldsymbol{\bar{\mathcal{E}}}_{\lambda}$ is positive semidefinite,
with rank $l - 4$.
The bending energy of a single TPS function is proportional to $\boldsymbol{w}^{\top} \boldsymbol{K}_{\lambda} \boldsymbol{w}$,
with $\boldsymbol{w} = \boldsymbol{\bar{\mathcal{E}}}_{\lambda}
\boldsymbol{y}$,
which is given by:
\begin{equation*}
\int_{\mathbb{R}^d} 
\lVert
\frac{\partial^2}{\partial \boldsymbol{p}^2}
\tau (\boldsymbol{p})
\rVert_F^2 \, d \boldsymbol{p}
\propto
\boldsymbol{w}^{\top} \boldsymbol{K}_{\lambda} \boldsymbol{w}
=
\boldsymbol{y}^{\top}
\boldsymbol{\bar{\mathcal{E}}}_{\lambda}
\boldsymbol{y}
.
\end{equation*}

For the 3D TPS warp model, the overall bending energy is:
\begin{align*}
&\, \int_{\mathbb{R}^d} 
\left\|
\frac{\partial^2}{\partial \boldsymbol{p}^2}
\mathcal{W} (\boldsymbol{p}, \boldsymbol{W})
\right\|_F^2 \, d \boldsymbol{p}
\\[5pt] = &\,
\int_{\mathbb{R}^d} 
\bigg(
\left\|
\frac{\partial^2}{\partial \boldsymbol{p}^2}
\tau_x (\boldsymbol{p})
\right\|_F^2
+ 
\left\|
\frac{\partial^2}{\partial \boldsymbol{p}^2}
\tau_y (\boldsymbol{p})
\right\|_F^2
 \\ & + 
\left\|
\frac{\partial^2}{\partial \boldsymbol{p}^2}
\tau_z (\boldsymbol{p})
\right\|_F^2
\bigg) 
d \boldsymbol{p}
\\[5pt] = &\,
\boldsymbol{y}_x^{\top}
\boldsymbol{\bar{\mathcal{E}}}_{\lambda}
\boldsymbol{y}_x
+
\boldsymbol{y}_y^{\top}
\boldsymbol{\bar{\mathcal{E}}}_{\lambda}
\boldsymbol{y}_y
+
\boldsymbol{y}_z^{\top}
\boldsymbol{\bar{\mathcal{E}}}_{\lambda}
\boldsymbol{y}_z
\\[5pt] = &\,
\mathbf{tr}\left( \boldsymbol{W}^{\top}  \boldsymbol{\bar{\mathcal{E}}}_{\lambda} \boldsymbol{W} \right)
=
\lVert \sqrt{ \boldsymbol{\bar{\mathcal{E}}}_{\lambda}}
\boldsymbol{W}  \rVert_F^2
\end{align*}
where $\sqrt{ \boldsymbol{\bar{\mathcal{E}}}_{\lambda} }$ is the matrix square root of $\boldsymbol{\bar{\mathcal{E}}}_{\lambda}$.
Then the bending energy of the TPS warp writes into the standard form
$
\mathcal{R} (\boldsymbol{W}) = \lVert \sqrt{ \boldsymbol{\bar{\mathcal{E}}}_{\lambda} }
 \boldsymbol{W} \rVert_F^2
$,
given in equation~(\ref{eq: standard form regularization term}).

\subsection{The Thin-Plate Spline Satisfies $\boldsymbol{\mathcal{B}}\left(\cdot\right)^{\top} \boldsymbol{x} = \boldsymbol{1}$ and $\boldsymbol{Z} \boldsymbol{x} = \boldsymbol{0}$}
\label{appendix: theorem equivalent conditions are satisified for the TPS warp}

Let an arbitrary point cloud $\boldsymbol{D}$ be
$
\boldsymbol{D}
=
\left[\boldsymbol{p}_1,\,
\boldsymbol{p}_2,
\dots,
\boldsymbol{p}_m
\right]
$.
Then by equation~(\ref{eq: TPS warp model, W function definition})
and the point-cloud operation defined in Section \ref{subsection: generalized warp models, operating on point cloud},
we write:
\begin{align*}
 \boldsymbol{\mathcal{B}}(\boldsymbol{D})
 & = 
\left[\boldsymbol{\beta}(\boldsymbol{p}_1),\,
\boldsymbol{\beta}(\boldsymbol{p}_2),
\dots,
\boldsymbol{\beta}(\boldsymbol{p}_m)
\right]
\\[5pt] & =
\boldsymbol{\mathcal{E}}_{\lambda}^{\top} \begin{bmatrix}
\boldsymbol{\phi}_{\boldsymbol{p}_1}
& \boldsymbol{\phi}_{\boldsymbol{p}_2}
& \dots &
\boldsymbol{\phi}_{\boldsymbol{p}_m} \\
\boldsymbol{\tilde{p}}_1 &
\boldsymbol{\tilde{p}}_2 &
\dots &
\boldsymbol{\tilde{p}}_m \\
\end{bmatrix}
= 
\boldsymbol{\mathcal{E}}_{\lambda}^{\top} \begin{bmatrix}
\boldsymbol{M}_{\boldsymbol{D}} \\
\boldsymbol{1}^{\top}
\end{bmatrix}
.
\end{align*}
Hence $\boldsymbol{\mathcal{B}}(\boldsymbol{D})^{\top} $ has structure:
\begin{equation*}
	\boldsymbol{\mathcal{B}}(\boldsymbol{D})^{\top} 
	=
	\begin{bmatrix}
	\boldsymbol{M}_{\boldsymbol{D}}^{\top} &
	\boldsymbol{1}
	\end{bmatrix}
	\boldsymbol{\mathcal{E}}_{\lambda}
	.
\end{equation*}

\subsubsection{$\boldsymbol{\mathcal{B}}(\cdot)^{\top} \boldsymbol{x}
	= \boldsymbol{1}$}

It suffices to show such an $\boldsymbol{x}$ satisfies
$
\boldsymbol{\mathcal{E}}_{\lambda}
\boldsymbol{x}
=
\begin{bmatrix}
\boldsymbol{0} \\ 1
\end{bmatrix}
$.
Because:
\begin{equation*}
\boldsymbol{\mathcal{E}}_{\lambda}
\boldsymbol{\tilde{C}}^{\top}
=
\begin{bmatrix}
\boldsymbol{O} \\
\boldsymbol{I}
\end{bmatrix}
\Longrightarrow
\boldsymbol{\mathcal{E}}_{\lambda}
\boldsymbol{\tilde{C}}^{\top}
\begin{bmatrix}
\boldsymbol{0} \\ 1
\end{bmatrix}
=
\begin{bmatrix}
\boldsymbol{O} \\
\boldsymbol{I}
\end{bmatrix}
\begin{bmatrix}
\boldsymbol{0} \\ 1
\end{bmatrix}
=
\begin{bmatrix}
\boldsymbol{0} \\ 1
\end{bmatrix}
,
\end{equation*}
such an $\boldsymbol{x}$ indeed exists,
which is $\boldsymbol{x} = \boldsymbol{\tilde{C}}^{\top}
\begin{bmatrix}
\boldsymbol{0} \\ 1
\end{bmatrix}
$.

\subsubsection{$\boldsymbol{Z} \boldsymbol{x} = \boldsymbol{0}$}

	We need to show:
	\begin{equation*}
	\boldsymbol{Z} \boldsymbol{x} = \boldsymbol{0}
	\iff
	\lVert  \boldsymbol{Z} \boldsymbol{x} \rVert^2
	=
	\boldsymbol{x}^{\top} \boldsymbol{Z}^{\top}
	\boldsymbol{Z} \boldsymbol{x} = 0
	.
	\end{equation*}
The proof is immediate by noting that
$\boldsymbol{Z}^{\top}
\boldsymbol{Z}
=
\boldsymbol{\bar{\mathcal{E}}}_{\lambda}
$,
and
$
\boldsymbol{\bar{\mathcal{E}}}_{\lambda}
\boldsymbol{x}
=
\boldsymbol{\bar{\mathcal{E}}}_{\lambda}
\boldsymbol{\tilde{C}}^{\top}
\begin{bmatrix}
\boldsymbol{0} \\ 1
\end{bmatrix}
=
\boldsymbol{O}
\begin{bmatrix}
\boldsymbol{0} \\ 1
\end{bmatrix}
=
\boldsymbol{0}
$.

\subsection{The Thin-Plate Spline is Invariant to the Coordinate Transformation}
\label{appendix. apply coordinate transformations to the TPS warp}

Given a datum shape $\boldsymbol{D}$,
we apply the same rigid transformation $(\boldsymbol{R},\, \boldsymbol{t})$ to each datum point $\boldsymbol{p}_i$, and denote the transformed datum point as $\boldsymbol{p}_i' = \boldsymbol{R} \boldsymbol{p}_i + \boldsymbol{t}$.
Each control point $\boldsymbol{c}_i$ is transformed to
$
\boldsymbol{c}_i' = 
\boldsymbol{R} \boldsymbol{c}_i + \boldsymbol{t}
$.
The matrix $\boldsymbol{K}_{\lambda}$ is invariant by replacing $\boldsymbol{c}_i$ with $\boldsymbol{c}_i'$ as its elements are radial basis functions.
By replacing $\boldsymbol{\tilde{C}}$ by $\boldsymbol{\tilde{C}}' = \left[
\boldsymbol{\tilde{c}}_1',
\boldsymbol{\tilde{c}}_2',
\dots,
\boldsymbol{\tilde{c}}_l'\right]$,
the matrix $\boldsymbol{\mathcal{E}}_{\lambda}$ becomes:
\begin{equation*}
\boldsymbol{\mathcal{E}}_{\lambda}' =
\begin{bmatrix}
	\boldsymbol{I} & \boldsymbol{O} \\
	\boldsymbol{O} & (\boldsymbol{T}^{\top})^{-1}
\end{bmatrix}
\boldsymbol{\mathcal{E}}_{\lambda} 
, \  \mathrm{with} \ 
\boldsymbol{T} = \begin{bmatrix}
\boldsymbol{R} & \boldsymbol{t} \\
\boldsymbol{0}^{\top} & 1
\end{bmatrix}
.
\end{equation*}
$\boldsymbol{\phi}_{\boldsymbol{p}_i}$ is also invariant by simultaneously replacing $\boldsymbol{c}_i$ with $\boldsymbol{c}_i'$ and $\boldsymbol{p}_i$ with $\boldsymbol{p}_i'$.
Notice that
$\boldsymbol{\tilde{p}}_i' = \boldsymbol{T} \boldsymbol{\tilde{p}}_i$,
thus:
\begin{equation*}
\boldsymbol{\beta} \left(\boldsymbol{p}_i'\right) = 	{\boldsymbol{\mathcal{E}}_{\lambda}'}^{\top} \begin{bmatrix}
	\boldsymbol{\phi}_{\boldsymbol{p}_i'} \\ \boldsymbol{\tilde{p}}_i'
\end{bmatrix}
=
\boldsymbol{\mathcal{E}}_{\lambda}^{\top} \begin{bmatrix}
\boldsymbol{\phi}_{\boldsymbol{p}_i} \\ \boldsymbol{\tilde{p}}_i
\end{bmatrix}
= \boldsymbol{\beta} \left(\boldsymbol{p}_i\right)
.
\end{equation*}
This proves $\boldsymbol{\mathcal{B}}\left(\boldsymbol{D}'\right)
=
\boldsymbol{\mathcal{B}}\left(\boldsymbol{D}\right)$
where $\boldsymbol{D}' = \boldsymbol{R} \boldsymbol{D} + \boldsymbol{t} \boldsymbol{1}^{\top}$.
Denote $\boldsymbol{\bar{\mathcal{E}}}_{\lambda}'$ to be the top $l \times l$ sub-block of $\boldsymbol{\mathcal{E}}_{\lambda}'$.
It is obvious that
$\boldsymbol{\bar{\mathcal{E}}}_{\lambda}' = \boldsymbol{\bar{\mathcal{E}}}_{\lambda}$ thus the regularization matrix $\boldsymbol{Z}' = \sqrt{ \boldsymbol{\bar{\mathcal{E}}}_{\lambda}' }
= \sqrt{ \boldsymbol{\bar{\mathcal{E}}}_{\lambda} }
= \boldsymbol{Z}
$
remains unchanged.

\section{Lemmas and Proof of Theorem \ref{label: theorem: three equivalent conditions, p, q , cost}}
\label{appendix: proof of Theorem 1 - four equivalent statements}

\subsection{Lemmas on Positive Semidefinite Matrices}

\begin{lemma}
\label{lemma: summation of psd matrices is psd matrix}
For any $\boldsymbol{M} = \sum_{i=1}^{n} \boldsymbol{M}_i$,
if $\boldsymbol{M}_i \succeq \boldsymbol{O}$ for each $i \in \left[ 1:n \right]$,
then $\boldsymbol{M} \succeq \boldsymbol{O}$.
\end{lemma}

\begin{lemma}
	\label{lemma: a useful lemma regarding positive semidefinite matrix}
	For any $\boldsymbol{M} \succeq \boldsymbol{O}$,
	we have
	$\boldsymbol{x}^{\top} \boldsymbol{M} \boldsymbol{x} = 0
	\iff 
	\boldsymbol{M} \boldsymbol{x} = \boldsymbol{0}$.
\end{lemma}

\begin{proof}
	Given $\boldsymbol{M} \succeq \boldsymbol{O}$, there exist a matrix $\sqrt{\boldsymbol{M}}$,
	such that $\boldsymbol{M} = \sqrt{\boldsymbol{M}} \sqrt{\boldsymbol{M}}$.
	Therefore:
	\begin{multline*}
	\boldsymbol{x}^{\top} \boldsymbol{M} \boldsymbol{x} 
	=
	\lVert \sqrt{\boldsymbol{M}} \boldsymbol{x} \rVert_2^2
	= 0 \\
	\iff
	\sqrt{\boldsymbol{M}} \boldsymbol{x} = \boldsymbol{0}
	\Longrightarrow
	\boldsymbol{M} \boldsymbol{x}
	=
	\sqrt{\boldsymbol{M}}
	\sqrt{\boldsymbol{M}} \boldsymbol{x} = \boldsymbol{0}
	.
	\end{multline*}
	This completes the proof of the sufficiency.
	The necessity is obvious.
\end{proof}

\begin{lemma}
\label{lemma: summation of PSD matrices, M x = 0 means Mi x = 0}
For any $\boldsymbol{M} = \sum_{i=1}^{n} \boldsymbol{M}_i \succeq \boldsymbol{O}$,
with each $\boldsymbol{M}_i \succeq \boldsymbol{O}$,
we have
$\boldsymbol{M} \boldsymbol{x} = \boldsymbol{0}
\iff
\boldsymbol{M}_i \boldsymbol{x} = \boldsymbol{0}
$
for each $i \in \left[ 1:n \right]$.
\end{lemma}
\begin{proof}
Because $\boldsymbol{M} \succeq \boldsymbol{O}$ and $\boldsymbol{M}_i \succeq \boldsymbol{O}$,
by Lemma \ref{lemma: a useful lemma regarding positive semidefinite matrix},
it suffices to show
$\boldsymbol{x}^{\top} \boldsymbol{M} \boldsymbol{x} = 0
\iff 
\boldsymbol{x}^{\top} \boldsymbol{M}_i\, \boldsymbol{x} = 0$,
which is true because
$
\boldsymbol{x}^{\top} \boldsymbol{M} \boldsymbol{x}
=
\sum_{i=1}^{n}
\boldsymbol{x}^{\top} \boldsymbol{M}_i\, \boldsymbol{x}
$
with each summand $\boldsymbol{x}^{\top} \boldsymbol{M}_i\, \boldsymbol{x} \ge 0$.
Therefore
$\boldsymbol{x}^{\top} \boldsymbol{M} \boldsymbol{x} = 0
\iff 
\boldsymbol{x}^{\top} \boldsymbol{M}_i\, \boldsymbol{x} = 0$
for each $i \in \left[ 1:n \right]$.
\end{proof}

\subsection{Proof of Theorem \ref{label: theorem: three equivalent conditions, p, q , cost}}

We assume $\mu_i$ being a tuning parameter $\mu_i \ge 0$.
Then the following chain of generalized equality holds:
\begin{align*}
\boldsymbol{I}
& \succeq
\boldsymbol{\mathcal{B}}_i^{\top}
\left(\boldsymbol{\mathcal{B}}_i
\boldsymbol{\mathcal{B}}_i^{\top}
\right)^{-1}
\boldsymbol{\mathcal{B}}_i
\\ & \succeq
\boldsymbol{\mathcal{B}}_i^{\top}
\left(\boldsymbol{\mathcal{B}}_i
\boldsymbol{\mathcal{B}}_i^{\top}
+
\mu_i \boldsymbol{Z}_i^{\top}  \boldsymbol{Z}_i\right)^{-1}
\boldsymbol{\mathcal{B}}_i
  \succeq
\boldsymbol{O}
.
\end{align*}
Therefore
$\boldsymbol{I} - \boldsymbol{Q}_i \succeq \boldsymbol{O}$,
and thus
$
\boldsymbol{\mathcal{P}}_{\mathrm{\Rmnum{2}}}
=
\sum_{i=1}^{n} \left(\boldsymbol{I} - \boldsymbol{Q}_i\right)
\succeq \boldsymbol{O}
$.

\begin{proof}
	We shall prove the result by showing:
	$(a)\iff(b)$,
	$(a)\iff(c)$,
	$(a)\iff(d)$,
	and
	$(c)\iff(e)$.

	{$(a)\iff(b)$.}
	Because
	$\boldsymbol{\mathcal{P}}_{\mathrm{\Rmnum{2}}}
	=
	n \boldsymbol{I}
	-
	\boldsymbol{\mathcal{Q}}_{\mathrm{\Rmnum{2}}}
	$,
	thus
	$\boldsymbol{\mathcal{P}}_{\mathrm{\Rmnum{2}}}
	\boldsymbol{1}
	=
	n \boldsymbol{1}
	-
	\boldsymbol{\mathcal{Q}}_{\mathrm{\Rmnum{2}}}
	\boldsymbol{1}
	$,
	which means
	$\boldsymbol{\mathcal{P}}_{\mathrm{\Rmnum{2}}} \boldsymbol{1} = \boldsymbol{0} \iff \boldsymbol{\mathcal{Q}}_{\mathrm{\Rmnum{2}}} \boldsymbol{1} = n \boldsymbol{1}$.

	$(a)\iff(c)$.
	Note that
	$\boldsymbol{\mathcal{P}}_{\mathrm{\Rmnum{2}}}
	=
	\sum_{i=1}^{n} \left(\boldsymbol{I} - \boldsymbol{Q}_i\right)
	\succeq \boldsymbol{O}$
	with
	$\boldsymbol{I} - \boldsymbol{Q}_i \succeq \boldsymbol{O}$.
	Therefore by Lemma \ref{lemma: summation of PSD matrices, M x = 0 means Mi x = 0}, we have
	$
	\boldsymbol{\mathcal{P}}_{\mathrm{\Rmnum{2}}}
	\boldsymbol{1} = \boldsymbol{0}
	\iff
	\left(\boldsymbol{I} - \boldsymbol{Q}_i\right) \boldsymbol{1} = \boldsymbol{0}
	\iff
	\boldsymbol{Q}_i \boldsymbol{1} = \boldsymbol{1}
	$.

	{$(a)\iff(d)$.}
	Applying a translation $\boldsymbol{t}$ to $\boldsymbol{S}$, we have:
	\begin{multline}
	\label{eq: the appendix proof: equivalent statements of P1=0}
	\mathbf{tr}\left(\boldsymbol{S} \boldsymbol{\mathcal{P}}_{\mathrm{\Rmnum{2}}} \boldsymbol{S}^{\top}\right) =
	\mathbf{tr}\left( \left(\boldsymbol{S} + \boldsymbol{t} \boldsymbol{1}^{\top}\right) \boldsymbol{\mathcal{P}}_{\mathrm{\Rmnum{2}}} \left(\boldsymbol{S} + \boldsymbol{t} \boldsymbol{1}^{\top}\right)^{\top} \right)
	\\
	\iff
	- 2 \mathbf{tr}\left(\boldsymbol{S} \boldsymbol{\mathcal{P}}_{\mathrm{\Rmnum{2}}} \boldsymbol{1} \boldsymbol{t}^{\top} \right)
	=
	\mathbf{tr}\left( \boldsymbol{t} \boldsymbol{1}^{\top} \boldsymbol{\mathcal{P}}_{\mathrm{\Rmnum{2}}} \boldsymbol{1} \boldsymbol{t}^{\top}\right)
	.
	\end{multline}
	
	Sufficiency:
	if $\boldsymbol{\mathcal{P}}_{\mathrm{\Rmnum{2}}} \boldsymbol{1} = \boldsymbol{0}$,
	then
	$- 2 \mathbf{tr}\left(\boldsymbol{S} \boldsymbol{\mathcal{P}}_{\mathrm{\Rmnum{2}}} \boldsymbol{1} \boldsymbol{t}^{\top} \right)
	=
	\mathbf{tr}\left( \boldsymbol{t} \boldsymbol{1}^{\top} \boldsymbol{\mathcal{P}}_{\mathrm{\Rmnum{2}}} \boldsymbol{1} \boldsymbol{t}^{\top}\right) = 0
	$.
	
	Necessity:
	note that $\boldsymbol{\mathcal{P}}_{\mathrm{\Rmnum{2}}}$ is positive semidefinite, so $\mathbf{tr}\left( \boldsymbol{t} \boldsymbol{1}^{\top} \boldsymbol{\mathcal{P}}_{\mathrm{\Rmnum{2}}} \boldsymbol{1} \boldsymbol{t}^{\top}\right) \ge 0$,
	which means $\mathbf{tr}\left(\boldsymbol{S} \boldsymbol{\mathcal{P}}_{\mathrm{\Rmnum{2}}} \boldsymbol{1} \boldsymbol{t}^{\top}\right) \le 0$.
	Since equation (\ref{eq: the appendix proof: equivalent statements of P1=0}) holds for any $\boldsymbol{t}$, flip the sign of $\boldsymbol{t}^{\top}$, we obtain $\mathbf{tr}\left(\boldsymbol{S} \boldsymbol{\mathcal{P}}_{\mathrm{\Rmnum{2}}} \boldsymbol{1} \boldsymbol{t}^{\top}\right) \ge 0$.
	Therefore:
	\begin{multline*}
	0 \le \mathbf{tr}\left(\boldsymbol{S} \boldsymbol{\mathcal{P}}_{\mathrm{\Rmnum{2}}} \boldsymbol{1} \boldsymbol{t}^{\top}\right) \le 0
	\\ \iff
	\mathbf{tr}\left(\boldsymbol{S} \boldsymbol{\mathcal{P}}_{\mathrm{\Rmnum{2}}} \boldsymbol{1} \boldsymbol{t}^{\top}\right) = 0
	\iff
	\mathbf{tr}\left( \boldsymbol{t} \boldsymbol{1}^{\top} \boldsymbol{\mathcal{P}}_{\mathrm{\Rmnum{2}}} \boldsymbol{1} \boldsymbol{t}^{\top}\right)  = 0
	.
	\end{multline*}
	Note that
	$\boldsymbol{1}^{\top} \boldsymbol{\mathcal{P}}_{\mathrm{\Rmnum{2}}} \boldsymbol{1}$ is a scalar, hence:
	\begin{align*}
	\mathbf{tr}\left( \boldsymbol{t} \boldsymbol{1}^{\top} \boldsymbol{\mathcal{P}}_{\mathrm{\Rmnum{2}}} \boldsymbol{1} \boldsymbol{t}^{\top}\right) 
	& = 
	\boldsymbol{1}^{\top} \boldsymbol{\mathcal{P}}_{\mathrm{\Rmnum{2}}} \boldsymbol{1}
	\,
	\mathbf{tr}\left( \boldsymbol{t}  \boldsymbol{t}^{\top}\right)
	\\ & =
	\boldsymbol{1}^{\top} \boldsymbol{\mathcal{P}}_{\mathrm{\Rmnum{2}}} \boldsymbol{1}
	\lVert \boldsymbol{t} \rVert_2^2
	=
	0,
	\end{align*}
	for any $\boldsymbol{t}$, if and only if
	$
	\boldsymbol{1}^{\top} \boldsymbol{\mathcal{P}}_{\mathrm{\Rmnum{2}}} \boldsymbol{1} = 0
	\iff
	\boldsymbol{\mathcal{P}}_{\mathrm{\Rmnum{2}}} \boldsymbol{1}
	=\boldsymbol{0}
	$ by Lemma \ref{lemma: a useful lemma regarding positive semidefinite matrix}.

	$(c)\iff(e)$.
	Let
	$
	\boldsymbol{\bar{Q}}_i
	= 
	\boldsymbol{\mathcal{B}}_i^{\top}
	\left(\boldsymbol{\mathcal{B}}_i
	\boldsymbol{\mathcal{B}}_i^{\top}
	\right)^{-1}
	\boldsymbol{\mathcal{B}}_i
	$,
	and $\left(\boldsymbol{\mathcal{B}}_i^{\top}\right)^{\dagger} = \left(\boldsymbol{\mathcal{B}}_i
	\boldsymbol{\mathcal{B}}_i^{\top}
	\right)^{-1}
	\boldsymbol{\mathcal{B}}_i
	$.
	Applying the Woodbury matrix identity,
	the matrix $\boldsymbol{Q}_i$ can be expanded as:
	\begin{equation*}
	\begin{aligned}
	& 
	\boldsymbol{Q}_i
	= 
	\boldsymbol{\mathcal{B}}_i^{\top}
	\left(\boldsymbol{\mathcal{B}}_i
	\boldsymbol{\mathcal{B}}_i^{\top}
	+
	\mu_i \boldsymbol{Z}_i^{\top}  \boldsymbol{Z}_i\right)^{-1}
	\boldsymbol{\mathcal{B}}_i
	 = 
	\boldsymbol{\bar{Q}}_i  \ - 
	\\ & 
	\mu_i
\scalebox{0.88}{$
	\underbrace{
		\left(\boldsymbol{Z}_i\left(\boldsymbol{\mathcal{B}}_i^{\top}\right)^{\dagger}\right)^{\top}
		\left(\boldsymbol{I} + \mu_i \boldsymbol{Z}_i \left(\boldsymbol{\mathcal{B}}_i
		\boldsymbol{\mathcal{B}}_i^{\top}\right)^{-1} \boldsymbol{Z}_i^{\top}\right)^{-1}
		\boldsymbol{Z}_i \left(\boldsymbol{\mathcal{B}}_i^{\top}\right)^{\dagger}
	}_{\boldsymbol{{\Delta}}_i}	
$}
.
	\end{aligned}
	\end{equation*}
	Then we write:
	\begin{equation*}
	\boldsymbol{I} - \boldsymbol{Q}_i
	=
	\left( \boldsymbol{I} - \boldsymbol{\bar{Q}}_i \right)
	+
	\mu_i  {\boldsymbol{{\Delta}}_i} \succeq \boldsymbol{O},
	\end{equation*}
	with $\boldsymbol{I} - \boldsymbol{\bar{Q}}_i \succeq \boldsymbol{O}$ and ${\boldsymbol{{\Delta}}_i} \succeq \boldsymbol{O}$.
	Noting that $\mu_i \ge 0$,
	by Lemma \ref{lemma: summation of PSD matrices, M x = 0 means Mi x = 0},
	we have
	$\left(\boldsymbol{I} - \boldsymbol{Q}_i\right) \boldsymbol{1} = \boldsymbol{0}$
	if and only if
	$\left(\boldsymbol{I} - \boldsymbol{\bar{Q}}_i\right)
	\boldsymbol{1}
	= \boldsymbol{0}$
	and
	$
	{\boldsymbol{{\Delta}}_i}
	\boldsymbol{1}
	= \boldsymbol{0}
	$.
	In other words,
$\boldsymbol{Q}_i \boldsymbol{1} = \boldsymbol{1}$
if and only if
$\boldsymbol{\bar{Q}}_i \boldsymbol{1} = \boldsymbol{1}$
and
${\boldsymbol{{\Delta}}_i} \boldsymbol{1} = \boldsymbol{0}$.

Note that $\boldsymbol{\bar{Q}}_i = \boldsymbol{\mathcal{B}}_i^{\top}
\left(\boldsymbol{\mathcal{B}}_i
\boldsymbol{\mathcal{B}}_i^{\top}
\right)^{-1}
\boldsymbol{\mathcal{B}}_i$ is the orthogonal projection into the range space of $\boldsymbol{\mathcal{B}}_i^{\top}$,
which states
$\boldsymbol{\bar{Q}}_i
\boldsymbol{1} = \boldsymbol{1}
\iff
\boldsymbol{1} \in \mathbf{Range}\left( \boldsymbol{\mathcal{B}}_i^{\top} \right)
\iff
\exists \boldsymbol{x}$
such that
$\boldsymbol{\mathcal{B}}_i^{\top} \boldsymbol{x} = \boldsymbol{1}$.
Therefore such an $\boldsymbol{x}$ is
$\boldsymbol{x} = \left(\boldsymbol{\mathcal{B}}_i^{\top}\right)^{\dagger}
\boldsymbol{1}$.
	
Since
$\boldsymbol{Z}_i \left(\boldsymbol{\mathcal{B}}_i
\boldsymbol{\mathcal{B}}_i^{\top}\right)^{-1} \boldsymbol{Z}_i^{\top} \succeq \boldsymbol{O}$, we know $\boldsymbol{I} + \mu_i \boldsymbol{Z}_i \left(\boldsymbol{\mathcal{B}}_i
\boldsymbol{\mathcal{B}}_i^{\top}\right)^{-1} \boldsymbol{Z}_i^{\top}$ is exactly positive definite.
Therefore
${\boldsymbol{{\Delta}}_i} \boldsymbol{1} = \boldsymbol{0} \iff
\boldsymbol{1}^{\top}
{\boldsymbol{{\Delta}}_i}
\boldsymbol{1}
= 0$
happens if and only if
$\boldsymbol{Z}_i \left(\boldsymbol{\mathcal{B}}_i^{\top}\right)^{\dagger}
\boldsymbol{1}
= \boldsymbol{0}$
which is
$\boldsymbol{Z}_i \boldsymbol{x} = \boldsymbol{0}$.
\end{proof}

\section{Proof of Theorem \ref{label: theorem: three equivalent conditions, p, q , cost, in case of partial shapes}}
\label{appendix: proofs of theorem 3 for parital shapes}

\begin{proof}

	$(b) \iff (c)$.
	We define:	
	$$\boldsymbol{\bar{Q}}_i = \boldsymbol{\Gamma}_i
	\boldsymbol{\mathcal{B}}_i^{\top}
	\left(
	\boldsymbol{\mathcal{B}}_i
	\boldsymbol{\Gamma}_i
	\boldsymbol{\Gamma}_i
	\boldsymbol{\mathcal{B}}_i^{\top}
	\right)^{-1}
	\boldsymbol{\mathcal{B}}_i
	\boldsymbol{\Gamma}_i .$$
	By the Woodbury matrix identity, we have:
	\begin{equation*}
	\begin{aligned}
	& \boldsymbol{\Gamma}_i
	\boldsymbol{\mathcal{B}}_i^{\top}
	\left(
	\boldsymbol{\mathcal{B}}_i
	\boldsymbol{\Gamma}_i
	\boldsymbol{\Gamma}_i
	\boldsymbol{\mathcal{B}}_i^{\top}
	+
	\mu_i \boldsymbol{Z}_i^{\top}  \boldsymbol{Z}_i
	\right)^{-1}
	\boldsymbol{\mathcal{B}}_i
	\boldsymbol{\Gamma}_i
	= 
	\boldsymbol{\bar{Q}}_i
	\  -  \\ & 
	\mu_i	
\scalebox{0.78}{$
	\underbrace{
		\left(\boldsymbol{Z}_i\left(\boldsymbol{\Gamma}_i\boldsymbol{\mathcal{B}}_i^{\top}\right)^{\dagger}\right)^{\top}
		\left(\boldsymbol{I} + \mu_i \boldsymbol{Z}_i \left(\boldsymbol{\mathcal{B}}_i \boldsymbol{\Gamma}_i \boldsymbol{\Gamma}_i
		\boldsymbol{\mathcal{B}}_i^{\top}\right)^{-1} \boldsymbol{Z}_i^{\top}\right)^{-1}
		\boldsymbol{Z}_i \left(\boldsymbol{\Gamma}_i\boldsymbol{\mathcal{B}}_i^{\top}\right)^{\dagger}
	}_{\boldsymbol{{\Delta}}_i}
$}
	.
	\end{aligned}
	\end{equation*}
	Then
	$\boldsymbol{P}_i$
	can be written as:
	$
	\boldsymbol{P}_i
	=
	\left( \boldsymbol{\Gamma}_i - \boldsymbol{\bar{Q}}_i \right)
	+
	\mu_i {\boldsymbol{{\Delta}}_i}
	$
	with $\boldsymbol{\Gamma}_i - \boldsymbol{\bar{Q}}_i \succeq \boldsymbol{O}$ and $\boldsymbol{{\Delta}}_i \succeq \boldsymbol{O}$.
	Thus $\boldsymbol{P}_i \succeq \boldsymbol{O}$.
	By Lemma \ref{lemma: summation of PSD matrices, M x = 0 means Mi x = 0},
	$\boldsymbol{P}_i \boldsymbol{1} = \boldsymbol{0}$ if and only if $\left( \boldsymbol{\Gamma}_i - \boldsymbol{\bar{Q}}_i \right) \boldsymbol{1} = \boldsymbol{0}$ and $\boldsymbol{{\Delta}}_i \boldsymbol{1} = \boldsymbol{0}$
	
	Note that $\boldsymbol{\bar{Q}}_i$ is the orthogonal projection to $\mathbf{Range}\left( \boldsymbol{\Gamma}_i \boldsymbol{\mathcal{B}}_i^{\top} \right)$. Thus we have:
	\begin{align*}
	\left( \boldsymbol{\Gamma}_i - \boldsymbol{\bar{Q}}_i \right) \boldsymbol{1} = \boldsymbol{0}
	& \iff
	\boldsymbol{\bar{Q}}_i \boldsymbol{\Gamma}_i \boldsymbol{1} = \boldsymbol{\Gamma}_i \boldsymbol{1}
	\\ & \iff
	\boldsymbol{\Gamma}_i \boldsymbol{1} \in \mathbf{Range}\left( \boldsymbol{\Gamma}_i \boldsymbol{\mathcal{B}}_i^{\top} \right)
	,
	\end{align*}
	if and only if there exists
	$\boldsymbol{x}$
	such that
	$\boldsymbol{\Gamma}_i \boldsymbol{\mathcal{B}}_i^{\top} \boldsymbol{x} =
	\boldsymbol{\Gamma}_i \boldsymbol{1}$.
	Such an $\boldsymbol{x}$ is
	$\boldsymbol{x} = \left(\boldsymbol{\Gamma}_i \boldsymbol{\mathcal{B}}_i^{\top}\right)^{\dagger}
	\boldsymbol{1}$.
	
	Note that $\boldsymbol{{\Delta}}_i \succeq \boldsymbol{O}$ and
	$\boldsymbol{I} + \mu_i \boldsymbol{Z}_i \left(\boldsymbol{\mathcal{B}}_i \boldsymbol{\Gamma}_i \boldsymbol{\Gamma}_i
	\boldsymbol{\mathcal{B}}_i^{\top}\right)^{-1} \boldsymbol{Z}_i^{\top} \succ \boldsymbol{O}$,
	thus we have:
	\begin{align*}
	\boldsymbol{{\Delta}}_i \boldsymbol{1} = \boldsymbol{0}
	& \iff
	\boldsymbol{1}^{\top}
	\boldsymbol{{\Delta}}_i \boldsymbol{1} = 0
	\\ & \iff
	\boldsymbol{Z}_i \boldsymbol{x} = 
	\boldsymbol{Z}_i \left(\boldsymbol{\Gamma}_i\boldsymbol{\mathcal{B}}_i^{\top}\right)^{\dagger} \boldsymbol{1}
	=
	\boldsymbol{0}
	.
	\end{align*}

	$(a) \iff (b)$. Since $\boldsymbol{P}_i \succeq \boldsymbol{O}$, the proof is immediate by Lemma \ref{lemma: summation of PSD matrices, M x = 0 means Mi x = 0}.

	This completes the proof.
\end{proof}

\section{Rigid Procrustes Analysis with Two Shapes}
\label{appendix: Procrustes problem on sod}

The pairwise Procrustes problem aims to solve for a similarity transformation
(a scale factor $s>0$, a rotation/reflection $\boldsymbol{R} \in \mathrm{O}(3)$,
and a translation $\boldsymbol{t} \in \mathbb{R}^3$),
that minimizes the registration error between the two point-clouds $\boldsymbol{D}_1$ and $\boldsymbol{D}_2$ with known correspondences:
\begin{equation*}
\argmin_{s > 0,\, \boldsymbol{R} \in \mathrm{O(3)}}\ 	\lVert s \boldsymbol{R} \boldsymbol{D}_1
+
\boldsymbol{t} \boldsymbol{1}^{\top}
-
\boldsymbol{D}_2 \rVert_F^2
.
\end{equation*}
This problem admits a closed-form solution
\citep{horn1987closed,horn1988closed, kanatani1994analysis}.

Let
$
\boldsymbol{\bar{D}}_1 = \boldsymbol{D}_1 -
\frac{1}{m}
\boldsymbol{D}_1  \boldsymbol{1} \boldsymbol{1}^{\top}
$,
and
$
\boldsymbol{\bar{D}}_2 = \boldsymbol{D}_2 -
\frac{1}{m}
\boldsymbol{D}_2  \boldsymbol{1} \boldsymbol{1}^{\top}
$.
Let the SVD of $\boldsymbol{\bar{D}}_1 \boldsymbol{\bar{D}}_2^{\top}$ be
$
\boldsymbol{\bar{D}}_1 \boldsymbol{\bar{D}}_2^{\top}
=
\boldsymbol{\bar{U}} \boldsymbol{\bar{\Sigma}} \boldsymbol{\bar{V}}^{\top}
$.
Then the optimal rotation is $\boldsymbol{R}^{\star} = \boldsymbol{\bar{V}} \boldsymbol{\bar{U}}^{\top}$
if $\boldsymbol{R}^{\star} \in \mathrm{O}(3)$.
If we require $\boldsymbol{R}^{\star} \in \mathrm{SO}(3)$, the optimal rotation is:
\begin{equation*}
\boldsymbol{R}^{\star}
=
\boldsymbol{\bar{V}}\,
\mathrm{diag}
\left(
1, 1, \det\left(\boldsymbol{\bar{V}} \boldsymbol{\bar{U}}^{\top}\right)
\right)
 \boldsymbol{\bar{U}}^{\top}
.
\end{equation*}
The optimal scale factor and the optimal translation are then given by:
\begin{align*}
s^{\star} & = \frac{\mathbf{tr}\left(
	\boldsymbol{R}^{\star} \boldsymbol{\bar{D}}_1  \boldsymbol{\bar{D}}_2^{\top}
	\right)}{\mathbf{tr}\left(
	\boldsymbol{\bar{D}}_1  \boldsymbol{\bar{D}}_1^{\top}
	\right)}
\\
\boldsymbol{t}^{\star}
 & =
\frac{1}{m}
\boldsymbol{D}_2  \boldsymbol{1}
-
\frac{1}{m} s^{\star}
\boldsymbol{R}^{\star}
\boldsymbol{D}_1  \boldsymbol{1}
.
\end{align*}

\end{appendices}

\section*{Acknowledgments}
This work was supported by ANR via the TOPACS project. We thank the authors of the public datasets which we could use in our experiments. We appreciate the valuable comments of the anonymous reviewers that have helped improving the quality of the manuscript.


\end{document}